\documentclass{article}

\usepackage[toc,page]{appendix}

\usepackage[hidelinks]{hyperref}
\usepackage{microtype}
\usepackage{graphicx} \usepackage{subfigure}
\usepackage{booktabs} 
\usepackage{float}
\usepackage{algcompatible}

\usepackage[accepted]{icml2020}

\usepackage{xcolor}

\newcommand{\stepsize}{t}
\newcommand{\ngrad}{\mbox{$\hat{g}$}}
\newcommand{\vngrad}{\mbox{$\myvecsym{\ngrad}$}}
\newcommand{\gradop}[2]{\partial_{#1} #2}
\newcommand{\hessop}[2]{\partial_{#1}^2 #2}
\newcommand{\gradrdop}[2]{\partial_{#1}^3 #2}
\newcommand{\crossop}[3]{\gradop{#1}{\gradop{#2}{#3}}}
\newcommand{\crossrdop}[4]{\gradop{#1}{\gradop{#2}{\gradop{#3}{#4}}}}
\newcommand{\anyvar}{\lat}
\newcommand{\vanyvar}{\vlat}
\newcommand{\mixCompA}{B}

\newcommand{\vlat}{\vz}
\newcommand{\lat}{z}

\newcommand{\varpar}{\lambda}
\newcommand{\vvarpar}{\vlambda}

\newcommand{\vmix}{\vw}

\newcommand{\mix}{w}

\newcommand{\varmean}{m}
\newcommand{\vvarmean}{\vm}

\newcommand{\fim}{F}
\newcommand{\vfim}{\vF}

\newcommand{\dkls}[3]{\mathbb{D}_{KL}^{#1}[#2 \, \|\, #3]}

\newcommand\cut[1]{}

\newcommand{\elbofinal}{\mathcal{L}}

\newcommand{\squishlist}{
   \begin{list}{$\bullet$}
    { \setlength{\itemsep}{0pt}      \setlength{\parsep}{3pt}
      \setlength{\topsep}{3pt}       \setlength{\partopsep}{0pt}
      \setlength{\leftmargin}{1.5em} \setlength{\labelwidth}{1em}
      \setlength{\labelsep}{0.5em} } }

\newcommand{\squishlisttwo}{
   \begin{list}{$\bullet$}
    { \setlength{\itemsep}{0pt}    \setlength{\parsep}{0pt}
      \setlength{\topsep}{0pt}     \setlength{\partopsep}{0pt}
      \setlength{\leftmargin}{2em} \setlength{\labelwidth}{1.5em}
      \setlength{\labelsep}{0.5em} } }

\newcommand{\squishend}{
    \end{list}  }

{}
\newtheorem{thm}{Theorem}{}
{}
{}
\newtheorem{defn}{Definition}{}

\newcommand{\half}{\mbox{$\frac{1}{2}$}}

\newcommand{\real}{\mbox{$\mathbb{R}$}}

\newcommand{\rnd}[1]{\left(#1\right)}
\newcommand{\sqr}[1]{\left[#1\right]}
\newcommand{\crl}[1]{\left\{#1\right\}}
\newcommand{\myang}[1]{\langle#1\rangle}
\newcommand{\myexpect}{\mathbb{E}}
\newcommand{\Unmyexpect}[1]{\mathbb{E}_{\scaleto{#1\mathstrut}{6pt}}}

\newcommand{\lapdist}{\mbox{Lap}}

\newcommand{\gauss}{\mbox{${\cal N}$}}

\newcommand{\Student}{\mbox{${\cal T}$}}

\newcommand{\myvec}[1]{\mbox{$\mathbf{#1}$}}
\newcommand{\myvecsym}[1]{\mbox{$\boldsymbol{#1}$}}

\newcommand{\valpha}{\mbox{$\myvecsym{\alpha}$}}

\newcommand{\vepsilon}{\mbox{$\myvecsym{\epsilon}$}}
\newcommand{\veta}{\mbox{$\myvecsym{\eta}$}}

\newcommand{\vmu}{\mbox{$\myvecsym{\mu}$}}
\newcommand{\vlambda}{\mbox{$\myvecsym{\lambda}$}}

\newcommand{\vphi}{\mbox{$\myvecsym{\phi}$}}

\newcommand{\vpi}{\mbox{$\myvecsym{\pi}$}}

\newcommand{\vsigma}{\mbox{$\myvecsym{\sigma}$}}
\newcommand{\vSigma}{\mbox{$\myvecsym{\Sigma}$}}

\newcommand{\vb}{\mbox{$\myvec{b}$}}

\newcommand{\ve}{\mbox{$\myvec{e}$}}

\newcommand{\vg}{\mbox{$\myvec{g}$}}
\newcommand{\vh}{\mbox{$\myvec{h}$}}

\newcommand{\vm}{\mbox{$\myvec{m}$}}

\newcommand{\vs}{\mbox{$\myvec{s}$}}

\newcommand{\vu}{\mbox{$\myvec{u}$}}

\newcommand{\vw}{\mbox{$\myvec{w}$}}

\newcommand{\vz}{\mbox{$\myvec{z}$}}
\newcommand{\vA}{\mbox{$\myvec{A}$}}

\newcommand{\vF}{\mbox{$\myvec{F}$}}
\newcommand{\vG}{\mbox{$\myvec{G}$}}

\newcommand{\vI}{\mbox{$\myvec{I}$}}

\newcommand{\vL}{\mbox{$\myvec{L}$}}

\newcommand{\vQ}{\mbox{$\myvec{Q}$}}
\newcommand{\vR}{\mbox{$\myvec{R}$}}
\newcommand{\vS}{\mbox{$\myvec{S}$}}

\newcommand{\vU}{\mbox{$\myvec{U}$}}
\newcommand{\vV}{\mbox{$\myvec{V}$}}

\newcommand{\vX}{\mbox{$\myvec{X}$}}

\newcommand{\calD}{\mbox{${\cal D}$}}

\newcommand{\data}{\calD}

\newcommand{\be}{\begin{equation}}
\newcommand{\ee}{\end{equation}}
\newcommand{\bea}{\begin{eqnarray}}
\newcommand{\eea}{\end{eqnarray}}
\newcommand{\beaa}{\begin{eqnarray*}}
\newcommand{\eeaa}{\end{eqnarray*}}

\usepackage[colorinlistoftodos,prependcaption,textsize=tiny]{todonotes}
\usepackage{scalerel}
\usepackage{graphicx,color}
\usepackage{amsmath}
\usepackage{amsfonts}
\usepackage{verbatim}
\usepackage{url}
\usepackage{footnote}

\usepackage{mathtools}  
\usepackage{amssymb}
\usepackage{tabulary}
\usepackage{booktabs}

\makesavenoteenv{tabular}
\makesavenoteenv{table}

\newtheorem{lemma}{Lemma}
\newenvironment{proof}{\paragraph{Proof:}}{\hfill$\square$}

\icmltitlerunning{
Handling the Positive-Definite Constraint in the Bayesian Learning Rule
}

\begin{document}

\twocolumn[
\icmltitle{
Handling the Positive-Definite Constraint in the Bayesian Learning Rule
}




\begin{icmlauthorlist}
\icmlauthor{Wu Lin}{ubc}
\icmlauthor{Mark Schmidt}{ubc,amii}
\icmlauthor{Mohammad Emtiyaz Khan}{riken}
\end{icmlauthorlist}

\icmlaffiliation{ubc}{University of British Columbia, Vancouver, Canada.}
\icmlaffiliation{riken}{RIKEN Center for Advanced Intelligence Project, Tokyo, Japan}
\icmlaffiliation{amii}{CIFAR AI Chair, Alberta Machine Intelligence Institute, Canada.}

\icmlcorrespondingauthor{Wu Lin}{wlin2018@cs.ubc.ca}

\icmlkeywords{Machine Learning, ICML}

\vskip 0.25in
]



\printAffiliationsAndNotice{} 

\begin{abstract}
The Bayesian learning rule is a natural-gradient variational inference method, which not only contains many existing learning algorithms as special cases but also enables the design of new algorithms.
Unfortunately, when variational parameters lie in an open constraint set, the rule may not satisfy the constraint and requires line-searches which could slow down the algorithm.
In this work, we address this issue for  positive-definite constraints by proposing an improved rule that naturally handles the constraints.
Our modification is obtained by using Riemannian gradient methods, and is valid when the approximation attains a \emph{block-coordinate natural parameterization} (e.g., Gaussian distributions and their mixtures).
We propose a  principled way to derive Riemannian gradients and retractions from scratch. 
Our method outperforms existing methods without any significant increase in computation. 
Our work makes it easier to apply the rule in the presence of positive-definite constraints in parameter spaces.

\end{abstract}

\section{Introduction}
\label{sec:intro}
The Bayesian learning rule, a recently proposed method, enables derivation of learning algorithms from Bayesian principles \cite{emti2020bayesprinciple}.
It is a natural-gradient variational inference method \cite{khan2017conjugate} where, by carefully choosing a posterior approximation, we can derive many  algorithms in fields such as probabilistic graphical models, continuous optimization, and deep learning. 
\citet{khan2017conjugate} derive approximate inference methods, such as stochastic variational inference and variational message passing; \citet{khan18a} derive connections to deep-learning algorithms; and \citet{emti2020bayesprinciple} derive many classical algorithms such as least-squares, gradient descent, Newton's method, and the forward-backward algorithm. 
We can also design new algorithms using this rule such as uncertainty estimation in deep learning \cite{osawa2019practical} and 
the ensemble of Newton methods \cite{lin2019fast}.

An issue with the rule is that when  parameters of a posterior approximation lie in an open constraint set, the update may not always satisfy the constraints. For Gaussian approximations, the posterior covariance needs to be positive definite but the rule may violate this; see Appendix D.1~in \citet{khan18a} for detail.
A straightforward solution is to use a backtracking line-search to keep the updates within the constraint set \cite{khan2017conjugate}, but this can lead to slow convergence.
In some cases, we can find an approximate update which always satisfies the constraints, e.g., for Gaussian approximations \cite{khan18a}.
However, in general, it is difficult to come up with such approximations that are both fast and reasonably accurate.
Our goal in this paper is to modify the Bayesian learning rule so that it can naturally handle such constraints. 

We propose an improved Bayesian learning rule to handle the positive-definite constraints. This is obtained by using a generalization of natural-gradient methods called  Riemannian-gradient methods.
We show that, for many useful approximations with a specific block-diagonal structure on the Fisher information matrix, the constraints are satisfied after an additional term is added to the rule.
Such a structure is possible when the parameters of the approximation are partitioned in what we call the \emph{block-coordinate natural (BCN) parameterizations}.
Fortunately, for many approximations with such parameterizations, the improved rule requires almost the same computation as the original rule.
An example is shown in Figure \ref{fig:ourAdam} where our improved rule fixes an implementation issue with an algorithm proposed by \citet{osawa2019practical} for deep learning. 
We present examples where the improved rule converges faster than the original rule and many existing baseline methods.

\begin{figure*}[!t]
\center

	\fbox{
			\begin{minipage}{.47\textwidth}
            \textbf{Variational Online Gauss-Newton (VOGN) Algorithm}
				\begin{algorithmic}[1]
               \STATE $\vlat \leftarrow \vmu {\, +\,  \left(N\hat{\vs}\right)^{-1/2} \odot \vepsilon}$, where  $\vepsilon \sim \gauss(\mathbf{0},\vI)$
               \STATE Randomly sample a minibatch $\mathcal{M}$ of size $M$ 
               \STATE {\color{blue} Compute and store individual gradients $\vg_i, \forall i \in \mathcal{M}$} 
               \STATE $\vg_\mu \leftarrow \frac{\lambda}{N} \vmu + \frac{1}{M} \sum_{i=1}^{M}\vg_i$
					\STATE $\vm \leftarrow r_1 \, \vm + (1-r_1) \,\vg_\mu, \quad \bar{\vm} \leftarrow \vm/(1-r_1^k)$
               \STATE $\vg_s \leftarrow \frac{\lambda}{N} - \hat{\vs} + {\color{blue}\frac{1}{M}  \sum_{i=1}^{M} (\vg_i \odot \vg_i) }   $ 
               
               \STATE
                $\hat{\vs} \leftarrow  \hat{\vs} + (1-r_2) \, \vg_s $ 
               \STATE $\vmu \leftarrow \vmu - \stepsize \,\, \bar{\vm} / \bar{\vs}$,\,\, where $\bar{\vs} \leftarrow \hat{\vs} /(1-r_2^k)$
				\end{algorithmic}
	\end{minipage}
	}
   \fbox{
			\begin{minipage}{.47\textwidth}
				\textbf{Our Adam-like Optimizer}
				\begin{algorithmic}[1]
               \STATE $\vlat \leftarrow \vmu {\, +\,  \left(N\hat{\vs}\right)^{-1/2} \odot \vepsilon}$, where  $\vepsilon \sim \gauss(\mathbf{0},\vI)$
               \STATE Randomly sample a minibatch  $\mathcal{M}$ of size $M$ 
               \STATE {\color{red} Compute a mini-batch gradient $\bar{\vg} \leftarrow  \frac{1}{M}\sum_{i=1}^{M}\vg_i$} 
               \STATE $\vg_\mu \leftarrow \frac{\lambda}{N} \vmu + \bar{\vg}$
					\STATE $\vm \leftarrow r_1 \, \vm + (1-r_1) \,\vg_\mu, \quad \bar{\vm} \leftarrow \vm/(1-r_1^k)$
               \STATE $\vg_s \leftarrow  \frac{\lambda}{N} - \hat{\vs} + {\color{red}\sqr{ \left(N\hat{\vs}\right) \odot \left( \vlat - \vmu \right)} \odot \bar{\vg}}  $ 
               
               \STATE $\vmu \leftarrow \vmu - \stepsize \,\, \bar{\vm} / \bar{\vs}$, \,\, where $\bar{\vs} \leftarrow \hat{\vs} /(1-r_2^k)$
               \STATE $\hat{\vs} \leftarrow  \hat{\vs} + (1-r_2) \, \vg_s \, {\color{red}+ \half (1-r_2)^2\vg_s \odot \hat{\vs}^{-1} \odot \vg_s} $

				\end{algorithmic}
	\end{minipage}
	}
   \label{fig:ourAdam}
   \caption{Our improved Bayesian learning rule solves an implementation issue with an existing algorithm known as VOGN \cite{khan18a} (shown in the left). VOGN is an Adam-like optimizer which gives state-of-the-art results on large deep learning problems \citep{osawa2019practical}. However, it requires us to store individual gradients in a minibatch which makes the algorithm slow (shown with blue in line 3 and 6).
   This is necessary for the scaling vector $\hat{\vs}$ to obtain a good estimate of uncertainty.
   Our work in this paper fixes this issue using the improved Bayesian learning rule. Our Adam-like optimizer (shown in the right) only requires average over the minibatch (see line 3). Line 6 is simply changed to use the re-parametrization trick with the averaged gradient.
   The additional terms added to the Bayesian learning rule is shown  in red in line 8. These changes do not increase the computation cost significantly while fixing the implementation issue of VOGN. Due to our modification, the scaling vector $\hat{\vs}$ always remains positive.
   A small difference is that the mean $\vmu$ is updated before in our optimizer (see line 7 and 8), while in VOGN it is the opposite.
   The difference shows that NGD depends on parameterization.
}
   \vspace{-0.3cm}
\end{figure*}

\section{Bayesian Learning Rule}

Given a dataset $\data$, it is common to estimate unknown variables $\vlat$ of a statistical model by minimizing\footnote{We assume  $\nabla_\lat \bar{\ell}(\vlat)$ and $\nabla_\lat^2 \bar{\ell}(\vlat)$ exist almost surely whenever they are needed. $\nabla$ denotes the standard  derivative in this paper. } $ \bar{\ell}(\vlat) \equiv \ell(\data,\vlat) + R(\vlat)$ where $\ell(\data,\vlat)$ is a loss function and $R(\vlat)$ is a regularizer. Many estimation strategies can be used, giving rise to various learning algorithms. E.g., maximum-likelihood approaches use gradient-based methods such as gradient descent and Newton's method,
while Bayesian approaches use inference algorithms such as message passing.

\citet{emti2020bayesprinciple} showed that many learning algorithms can be obtained from Bayesian principles. The key idea is to use the following Bayesian formulation where, instead of minimizing over $\vlat$, we minimize over a distribution $q(\vlat)$: 
\begin{align}
   \min_{q(\lat) \in \mathcal{Q} } \Unmyexpect{q(\lat)}{ \sqr{ \ell(\data, \vlat) } }  + \dkls{}{q(\vlat)}{p(\vlat)} \equiv \elbofinal(q).
   \label{eq:bayes}
\end{align} 
Here, $q(\vlat)$ is an approximation of the posterior of $\vlat$ given $\data$, $\mathcal{Q}$ is the set of approximation distributions, $p(\vlat) \propto \exp(-R(\vlat))$ is the prior, and $\mathbb{D}_{KL}$ denotes the Kullback-Leibler divergence.
To obtain existing learning algorithms from the above formulation, we need to carefully choose the approximation family $\mathcal{Q}$. \citet{emti2020bayesprinciple} consider the following \emph{minimal} exponential family (EF) distribution:
\begin{align*}
  q(\vlat|\vvarpar) := h(\vlat)\exp\sqr{ \myang{\vphi(\vlat), \vvarpar} - A(\vvarpar)} 
\end{align*}
where $\vphi(\vlat)$ is a vector containing sufficient statistics, 
$h(\vlat)$ is the base measure, $\vvarpar \in \Omega$ is the natural parameter,
$\Omega$ is the set of valid natural-parameters so that the log-partition function $A(\vvarpar)$ is finite,
and $\myang{ \cdot , \cdot}$ denotes an inner product.

\citet{emti2020bayesprinciple} present the Bayesian learning rule to optimize \eqref{eq:bayes}, which is a natural-gradient descent (NGD) update originally proposed by \citet{khan2017conjugate} for variational inference. The update takes the following form:
\begin{align}
 \text{NGD}: \,\,\,  \vvarpar\leftarrow \vvarpar - \stepsize \vngrad, \textrm{ with } \vngrad := \vF^{-1} \partial_{\varpar} \, \elbofinal(\vvarpar)
   \label{eq:ngvi}
\end{align}
where $\stepsize>0$ is a scalar step-size and $\vngrad$ is the natural gradient defined using the Fisher information matrix (FIM) $\vF := -\myexpect_q [ \partial_{\varpar}^2 \log q(\vlat|\vvarpar)]$ of $q$ and $\elbofinal(\vvarpar)$ which is equal to $\elbofinal(q)$ but defined in terms of $\vvarpar$.
\citet{emti2020bayesprinciple} proposed further simplifications,
e.g., for approximations with base measure $h(\vlat) \equiv 1$,
we can write \eqref{eq:ngvi} as
\begin{align}
   \vvarpar & \leftarrow 
    (1-\stepsize) \vvarpar -   \stepsize \partial_{\varmean} \, \Unmyexpect{q}{ \sqr{ \bar{\ell}(\vlat) }  }
   \label{eq:blr}
\end{align}
where $\vvarmean:=\Unmyexpect{q(\lat)}{ \sqr{ \vphi(\vlat) } }$ denotes the expectation parameter.

Existing learning algorithms can be derived as special cases by choosing an approximate form for $q(\vlat)$.
For example, when $q(\vlat) := \gauss(\vlat|\vmu,\vS^{-1})$ is a multivariate Gaussian approximation with the mean $\vmu$ and the precision matrix $\vS$, the learning rule \eqref{eq:blr} can be expressed as follows:
\begin{align}
   \vS & \leftarrow (1-\stepsize) \vS  + \stepsize \Unmyexpect{q}{\sqr{ \nabla_{\lat}^2 \bar{\ell}(\vlat) } }  \label{eq:von_1} \\
   \vmu & \leftarrow \vmu - \stepsize \vS^{-1}\Unmyexpect{q}{\sqr{ \nabla_\lat \bar{\ell}(\vlat) } } \label{eq:von_2} 
\end{align}
This algorithm uses the Hessian to update $\vS$ which is then used to scale the update for $\vmu$, in a similar fashion as Newton's method. The main difference here is that the gradient and Hessian are obtained at samples from $q(\vlat)$ instead of the current iterate $\vmu$. \citet{emti2020bayesprinciple} approximate the expectation at $\vmu$ to obtain an online Newton method. This algorithm is closely related to deep-learning optimizers, such as, RMSprop and Adam \cite{khan18a,zhang2018noisy}. 
A simplified version of this algorithm obtains  state-of-the-art results on large deep-learning problems for uncertainty estimation as shown by \citet{osawa2019practical}.

Many other examples are discussed in \citet{emti2020bayesprinciple}, including algorithms such as  stochastic gradient descent. The relationship to message passing algorithms and stochastic variational inference is shown in \citet{khan2017conjugate}. In summary, the Bayesian learning rule is a generic learning rule that can be used not only to derive existing algorithms, but also to improve them and design new ones. 

\subsection{Positive-Definite Constraints}
An issue  with updates \eqref{eq:ngvi} and \eqref{eq:blr} is that  the constraint $\vvarpar\in\Omega$ is not taken into account, where $\Omega$ is the set of valid  parameters. The update is valid when $\Omega$ is unconstrained (e.g., a Euclidean  space), but otherwise it may violate the constraint. An  example is the multivariate Gaussian of dimension $d$ where the precision matrix $\vS\in \mathbb{S}^{d\times d}_{++}$ is required to be
real and  positive-definite,
while the mean $\vmu\in \real^d$ is unconstrained. In such cases, the update may violate the constraint. E.g., in the update \eqref{eq:von_1},  $\vS$ can be indefinite, when the loss $\bar{\ell}(\vlat)$ is nonconvex. A similar issue appears when flexible approximations are used such as Gaussian mixtures.

Another example is a gamma distribution:
   $q(\lat|\alpha,\beta) \propto \lat^{\alpha-1} e^{-\lat\beta}$ where both $\alpha,\beta>0$. We denote the positivity constraint using $\mathbb{S}^1_{++}$. The rule takes the following form:
\begin{align}
   \alpha \leftarrow  (1-\stepsize) \alpha  - \stepsize    \ngrad_{\alpha}, \quad\quad    
   \beta \leftarrow  (1-\stepsize) \beta  - \stepsize  \ngrad_{\beta}
   \label{eq:blr_gamma}
\end{align}
where $\ngrad_{\alpha}$ and $\ngrad_\beta$ are gradient of $\Unmyexpect{q(\lat)}{ \sqr{ \bar{\ell}(\lat) - \log\lat  } }$ with respect to the expectation parameters $\varmean_\alpha=\Unmyexpect{q(\lat)}{\sqr{\log\lat } }$ and 
$\varmean_\beta=\Unmyexpect{q(\lat)}{\sqr{-\lat } }$ respectively; see a detailed derivation in Appedix E.3~in \citet{khan2017conjugate}.
Here again the learning rule does not ensure that $\alpha$ and $\beta$ are always positive.

In general, a backtracking line search proposed in \citet{khan2017conjugate} can be used so that the iterates stay within the constraint set. However, this could be slow in practice. \citet{khan18a} discuss this issue for the Gaussian case; see Appendix D.1 in their paper. They found that using   line-search is computationally expensive and non-trivial to implement for deep-learning problems.
They address this issue by approximating the
Hessian in \eqref{eq:von_1} with a positive-definite matrix.
This ensures that $\vS$ is always positive-definite.
However,
such approximations are difficult to come up with for general cases. E.g., for the gamma case, there is no such straight-forward approximation in update \eqref{eq:blr_gamma} to ensure positivity of $\alpha$ and $\beta$.
It is also possible to use an unconstrained transformation (e.g., a Cholesky factor). This approach uses automatic-differentiation (Auto-Diff), which can be much slower than explicit gradient forms (see the discussion in Section \ref{sec:related}).
Handling constraints within the Bayesian learning rule is an open issue which limits its applications.

In this paper, we focus on  positive-definite constraints and show that, in many cases, such constraints can be naturally handled by adding an additional term to the Bayesian learning rule. We show that, for this to happen, the approximation needs to follow a specific parameterization. We will now describe the modification in the next section, and later give its derivation using Riemannian gradient methods.

\section{Improved Bayesian Learning Rule}
\label{sec:IBLR}
We will give a new rule to handle the positive-definite constraints. Our  idea is to partition the parameter into blocks
so that each constraint is isolated in an individual block.

{\bf Assumption 1 [Mutually-Exclusive Constraints] :} \emph{
We assume parameter $\vvarpar=\{\vvarpar^{[1]}, \dots,\vvarpar^{[m]}\}$ can be partitioned into $m$ blocks with mutually-exclusive constraints $\Omega =\Omega_1 \times \dots \times \Omega_m$,  where square bracket $[i]$ denotes the $i$-th block and each block $\vvarpar^{[i]}$ is either unconstrained or positive-definite.
}

For example, consider  multivariate Gaussian approximations with the two blocks: one block containing the mean $\vmu$ and another containing the full precision $\vS$. This satisfies the above assumption because the first block is unconstrained and the second block is positive definite.
In $d$-dimensional diagonal Gaussian cases, we consider $2d$ blocks: one block containing the mean $\mu_i$ and  one block containing the precision $s_i$ for each dimension $i$, where each $s_i$ is positive.
Other examples such as gammas and inverse Gaussians can be partitioned to two blocks, where each block is positive.


{\bf Assumption 2 [Block Coordinate Parameterization] :} \emph{
   A parameterization satisfied Assumption 1 is block coordinate (BC) if the FIM  is block-diagonal according to  the block structure of the parameterization.
}

For Gaussians, using the mean and the covariance/precision as two blocks is a BC parameterization (see Appendix \ref{app:gauss_case}), while the natural parameterization is not \citep{malago2015information}. 
For EFs, we could use the Crouzeix identity \citep{nielsen2019geodesic} to identify a BC parameterization.

{\bf Assumption 3 [Block Natural Parameterization for EF] :}
   \emph{
      For $q(\vlat|\vvarpar)$ and each block $\vvarpar^{[i]}$, there exist function $\phi_i$ and $h_i$ such that $q(\vlat|\vvarpar)$ can be re-expressed as a minimal EF  distribution  given that the rest of blocks $\vvarpar^{[-i]}$ are known.
   \begin{align*}
      q(\vlat|\vlambda) \equiv h_i\rnd{\vlat, \vvarpar^{[-i]}} \exp\sqr{ \left\langle {\vphi_i\rnd{ \vlat,\vvarpar^{[-i]} }, \vvarpar^{[i]}} \right\rangle - A(\vvarpar)} 
   \end{align*} 
   }
\citet{lin2019fast} originally use Assumption 3 to define a multilinear EF.
We illustrate this assumption on the Gaussian distribution which can be written as the following exponential form, 
where
$A(\vmu,\vS)=\half\big[ \vmu^T\vS\vmu - \log \left|\vS/(2\pi)\right| \big]$ is the log-partition function.
\begin{align*}
q(\vlat|\vmu,\vS)=
\exp\Big( -\half \vlat^T\vS\vlat + \vlat^T\vS\vmu  - A(\vmu,\vS)\Big) 
\end{align*} 

Considering two blocks with $\vmu$ and $\vS$ respectively, we can express this distribution as follows, where the first equation is for $\vmu$ while the second equation is for  $\vS$:
\begin{align*}
   q(\vlat|\vmu,\vS) &= \underbrace{ \exp(-\half \vlat^T\vS\vlat)}_{h_1(\bf{\lat},\bf{S})} \exp\Big(  \myang{ \underbrace{\vS\vlat}_{\phi_1(\bf{\lat}, \bf{S})}, \vmu}    - A(\vmu,\vS)\Big)  \\
    &= \underbrace{1}_{h_2(\bf{\lat},\boldsymbol{\mu})} \exp\Big(  \myang{ \underbrace{-\half\vlat\vlat^T + \vmu\vlat^T }_{\phi_2(\bf{\lat},\boldsymbol{\mu})}, \vS}   - A(\vmu,\vS)\Big) 
\end{align*}
\vspace{-0.4cm}

We define the \emph{block-coordinate natural (BCN) parameterization} for an  EF  distribution as the parameterization which satisfies Assumptions from 1 to 3. Therefore, Gaussian distribution with $\vmu$ and $\vS$  can be expressed in a BCN parameterization $\vvarpar=\{\vvarpar^{[1]}, \vvarpar^{[2]}\}$, where 
$\vvarpar^{[1]}=\vmu$ and 
$\vvarpar^{[2]}=\vS$. Let $\varpar^{a_i}$ denote  the $a$-th entry of the $i$-th block parameter $\vvarpar^{[i]}$, where $a_i$ is a local index for block $i$.
$\ngrad^{c_i}$ is the $c$-th entry of natural gradient $\vngrad^{[i]}$ with respect to $\vvarpar^{[i]}$.

We now present the rule (see Section \ref{sec:deriv} for a derivation).
Under a BCN parameterization  $\vvarpar$, our rule  for block $i$ takes the following form with an extra  term shown in red:
\begin{align}
\varpar^{c_i}  \leftarrow \varpar^{c_i} - \stepsize \ngrad^{c_i} {\color{red} - \frac{\stepsize^2}{2} \sum_{a_i} \sum_{b_i}\Gamma_{\ a_ib_i}^{c_i} \ngrad^{a_i}\ngrad^{b_i} },
\label{eq:our_iblr}
\end{align} 
where
each summation is to sum over all entries of the $i$-th block,
$\Gamma^{c_i}_{\ a_i b_i} :=  \half  \partial_{\varmean_{c_i}} \partial_{\varpar^{a_i}} \partial_{\varpar^{b_i}}A(\vvarpar)$,
and
$\varmean_{c_i}$ is  the $c$-th entry of the BC
expectation parameter  $\vvarmean_{[i]}:=\Unmyexpect{q}{\big[ \vphi_i\big( \vlat,\vvarpar^{[-i]} \big) \big] } = \partial_{\varpar^{[i]}} A(\vvarpar)$.

The modification involves  computation of  the third-order term of the log-partition function\footnote{ We assume $A(\vvarpar)$ is (jointly) $C^{3}$-smooth. Note that $A(\vvarpar)$ is block-wisely $C^{3}$-smooth as shown in \citet{johansen1979introduction}. Approximations considered in this paper  satisfy this assumption. } $A(\vvarpar)$.
In the following Section \ref{sec:rgvi_gauss_case} and  \ref{sec:rgvi_gamma_case},
we discuss two examples where this computation is simplified and can be carried out like the original rule  with minimal computational increase.

Table \ref{tab:examples} in Appendix \ref{app:all_example_table} lists  more examples satisfying Assumption 1-3, where our rule can be applied and simplified.

\subsection{Example: Online Newton using  Gaussian Approximation}
\label{sec:rgvi_gauss_case}
The original rule for Gaussian approximations gives the update \eqref{eq:von_1}-\eqref{eq:von_2}, where the natural parameterization of Gaussian is used. We consider the parameterization $\vmu$ and $\vS$, in which the improved  rule takes the form (a detailed simplification is in Appendix \ref{app:gauss_case}) with an extra non-zero term shown in red:
\begin{align}
   \vmu & \leftarrow \vmu - \stepsize \vS^{-1}\Unmyexpect{q}{\sqr{ \nabla_\lat \bar{\ell}(\vlat) } } { \color{red} +  \mathbf{0} } \label{eq:ivon_1} \\
   \vS & \leftarrow (1-\stepsize) \vS  + \stepsize \Unmyexpect{q}{\sqr{ \nabla_{\lat}^2 \bar{\ell}(\vlat) } } { \color{red} +  \frac{\stepsize^2}{2} \hat{\vG} \vS^{-1} \hat{\vG} } \label{eq:ivon_2} ,
\end{align} where
$\hat{\vG}:= \vS-\Unmyexpect{q}{\sqr{ \nabla_{\lat}^2 \bar{\ell}(\vlat) } }$.
 The extra term  ensures that the positive definite constraint is  satisfied due to Theorem \ref{thm:gauss_update}.
 
 \begin{thm}
\label{thm:gauss_update}
   The updated $\vS$ in \eqref{eq:ivon_2} is positive definite if the initial $\vS$ is positive-definite.
\end{thm}
The proof of Theorem \ref{thm:gauss_update} can be found in Appendix \ref{app:gauss_proof}.

Although \eqref{eq:ivon_1}-\eqref{eq:ivon_2} appear similar to \eqref{eq:von_1}-\eqref{eq:von_2}, there is one difference -- the old $\vS$ is used as a preconditioner to update $\vmu$.
Note that \eqref{eq:ivon_1}-\eqref{eq:ivon_2} becomes a natural-gradient descent (NGD) update if we ignore the additional term.
Though natural gradient\footnote{\label{fn:note1}It is a representation of an abstract (parameterization-free) tangent vector in a Riemannian manifold under a parameterization. } is invariant to parameterization, NGD update depends on parameterization as shown by the difference.
However, we expect the difference to make a small change in practice.

Like
\citet{emti2020bayesprinciple}, an online Newton method can be obtained by approximating the expectations at $\vmu$, e.g., $\Unmyexpect{q}{\sqr{ \nabla_\lat \bar{\ell}(\vlat) } } \approx \nabla_\mu \bar{\ell}(\vmu) $ and $\Unmyexpect{q}{\sqr{ \nabla_{\lat}^2 \bar{\ell}(\vlat) } } \approx \nabla_{\mu}^2 \bar{\ell}(\vmu) $.
In this case, the algorithm converges to a local minimal of the loss $\bar{\ell}(\vlat)$. A key point is that, unlike Newton's method where the preconditioner may not be positive-definite for nonconvex functions, $\vS$ is  guaranteed to be positive definite.

When applied to factorized Gaussians, 
these updates give an improved version of the Variational Online Gauss-Newton (VOGN) algorithm in \citet{osawa2019practical}. It is shown in Figure \ref{fig:ourAdam} where the differences in our algorithm are shown in red.
 Our algorithm uses the reparameterization trick to avoid computing $\nabla_\lat^2 \bar{\ell}(\vlat)$ in \eqref{eq:ivon_2}. A derivation is given in Appendix 
\ref{app:adam_diag_gauss}.
Our algorithm fixes an implementation issue in  VOGN without comprising its performance and speed, where our 
update only stores a mini-batch gradient while VOGN has to store all individual gradients in a mini-batch.

\subsection{Example: Gamma Approximation}
\label{sec:rgvi_gamma_case}
Let's consider  gamma cases.
We use a BCN parameterization $\vvarpar=\{ \varpar^{[1]}, \varpar^{[2]}\}$ (see Appendix \ref{app:gamma_case} for detail), where $\varpar^{[1]}=\alpha$ and $\varpar^{[2]}=\frac{\beta}{\alpha}$.
The constraint is $\Omega = \mathbb{S}_{++}^1 \times \mathbb{S}_{++}^1$.
Since each block contains a scalar, we use global indexes as $\varpar^{(i)}=\varpar^{[i]}$ and $\ngrad^{(i)}=\ngrad^{[i]}$.
Moreover, we use $\Gamma^{i}_{\ \ ii}$ to denote $\Gamma^{c_i}_{\ \ a_i b_i}$.
Let $\mathrm{Ga}(\cdot)$ be the gamma function.
Under this parameterization, a gamma distribution is expressed as:

\vspace{-0.4cm}
\begin{align*}
q(\lat|\vvarpar)=\frac{1}{\lat}
\exp\Big(   \varpar^{(1)}\log  \lat  - \lat\varpar^{(1)}\varpar^{(2)} - A(\vvarpar)\Big),
\end{align*}
\vspace{-0.4cm}

where
$A(\vvarpar)=\log\mathrm{Ga}(\varpar^{(1)})- \varpar^{(1)}\left( \log \varpar^{(1)}+ \log \varpar^{(2)}\right)$.

Let $\psi(\cdot)$ be the digamma function.
We can compute the third derivatives (see Appendix \ref{app:gamma_case} for a derivation) as:
\begin{align*}
\Gamma^{1}_{\ \ 11}  =  \frac{  \frac{1}{ \varpar^{(1)} \times \varpar^{(1)}} + \hessop{\varpar^{(1)}} {\psi(\varpar^{(1)})}  }{ 2 \left( - \frac{1}{\varpar^{(1)}} + \gradop{\varpar^{(1)}}{\psi(\varpar^{(1)})}  \right) }, \,\,\,
\Gamma^{2}_{\ \ 22}  =  -\frac{1}{\varpar^{(2)}}
\end{align*}

The proposed rule in this case  is
\begin{align}
\varpar^{(i)} &\leftarrow \varpar^{(i)}  - \stepsize \ngrad^{(i)} {\color{red} - \frac{\stepsize^2}{2} \left( \Gamma^{i}_{\ \ ii} \right)  \ngrad^{(i)} \times \ngrad^{(i)}}, \,\,\, i=1,2 \label{eq:gamma_rgvi}
\end{align}
where each $\ngrad^{(i)}$ is a natural gradient computed via the implicit re-parameterization trick as shown in Appendix \ref{app:gamma_ng}.

\begin{thm}
\label{thm:gamma_update}
   The updated $\varpar^{(i)}$ in \eqref{eq:gamma_rgvi} is positive  if the initial $\varpar^{(i)}$ is positive.
\end{thm}
The proof of Theorem \ref{thm:gamma_update} can be found in Appendix \ref{app:gamma_proof}.

\subsection{Extension to EF Mixtures}
Our learning rule can be extended to mixture approximations, such as finite mixture of Gaussians (MOG) (shown in Appendix \ref{app:mog_case}) and skew Gaussian approximations (given in Appendix \ref{app:skewg_case}) 
using the joint FIM -- the FIM of the joint distribution of a mixture -- as suggested by \citet{lin2019fast}.
By extending the definition of the BCN parameterization to the joint distribution, our rule can be easily applied to mixture cases (see Appendix \ref{app:cef_rgvi} for detail).
For example, our update for MOG approximation  can be found at \eqref{eq:mog_iblr} in Appendix \ref{app:mog_case}, where our rule handles the positive-definite constraints in MOG.
Our update can be viewed as an improved version of the ensemble of Newton methods proposed  by \citet{lin2019fast}.
We also discuss why it is non-trivial  to extend VOGN to MOG cases in Appendix \ref{app:mog_ng}.

\section{Related Works}
\label{sec:related}

In $d$-dim  Gaussian cases, we can use unconstrained transformations (e.g., a Cholesky factor).
However, the natural-gradient computation becomes complicated.
Eq \eqref{eq:ngvi} gives  $O(d^6)$ for direct computation. \citet{salimbeni2018natural} propose an indirect approach  via  additional  vector-Jacobian products (VJPs), which could give $O(d^4)$. For some parameterizations, their method  gives an implicit  $O(d^3)$ update, where Auto-Diff is needed  to track non-zero terms in the additional Jacobians and to simplify the VJPs. Contrarily, our method gives a simple and explicit $O(d^3)$ update and  builds a direct connection to Newton's method. In practice, our  update is more numerically stable and much faster (see Figure \ref{figure:unblr_plots} in  Appendix \ref{app:more_plots}) than theirs if both use Auto-Diff.
Our approach is also easily extended to EFs and mixtures.

Our work is closely related to the method of \citet{tran2019variational}.
They propose a method based on a retraction map in Gaussian cases, which is a special case of ours (see Appendix \ref{app:special_case_cov}).
However, their retraction map does not directly generalize to other distributions, while ours does.
They do not provide a justification or derivation of the map. We fix this gap by deriving the map from first principles, justifying its use, and obtaining an Adam-like update by choosing a proper parametrization for Gaussian cases (see Appendix \ref{app:gauss_case}).
They also do not distinguish the  difference between the Riemannian gradient for a positive-definite matrix and the natural gradient for a covariance matrix (see Footnote \ref{ft:mis} in Appendix \ref{app:special_case_cov}).
Moreover, the retraction  and Riemannian gradients used in neural network cases are not derived from the same Riemannian metric.
In our work, retractions are implicitly induced by our rule.
Our retractions and Riemannian gradients are naturally derived from the same metric.

\citet{song2018accelerating} give  a similar update in non-Bayesian contexts, but the update does not always satisfy the constraints  for univariate Gaussians (see Appendix \ref{app:case}). Their update is neither simple nor efficient for multivariate cases such as multivariate Gaussians and MOGs (see 
Section \ref{sec:irgd}).

\citet{hosseini2015matrix} use a similar approach to ours but for parameter estimation of Gaussian mixtures.
They propose a transformation for each Gaussian component so that the mean and the covariance together can be re-parameterized as an augmented positive-definite matrix with an extra constraint.\footnote{The parameterization violates Assumption 1 since the positive-definite constraint of the augmented matrix and the extra constraint can not be partitioned into two  blocks.} They show that a local minimum of the negative log-likelihood of a mixture automatically  satisfies the extra constraint. Thus, they can employ Riemannian-gradient descent with a retraction map to update the augmented positive-definite matrix, where the extra  constraint can be safely ignored.
It is unclear if this  approach generalizes to variational inference (VI) since generating samples from MOG requires the extra constraint to be satisfied. Thus, the constraint cannot be omitted at each iteration  in VI.
Moreover, it is unclear whether Riemannian gradients and the retraction  are derived from the same metric for the mixture. \citet{hosseini2015matrix} use the gradients and the retraction designed for positive-definite matrices while we derive them  from the joint Fisher metric for MOGs in a principled way.

\section{Derivation of the Improved Rule}
\label{sec:deriv}
\subsection{Gradient Descent}
We first review gradient descent in Euclidean spaces and generalize  it to Riemannian manifolds, where we derive our rule.
Recall that we want to minimize \eqref{eq:bayes} in terms of $\vvarpar$ as:

\vspace{-0.8cm}
\begin{align*}
   \min_{\varpar \in \Omega } \Unmyexpect{q(\lat|\varpar)}{ \sqr{ \ell(\data, \vlat) } }  + \dkls{}{q(\vlat|\vvarpar)}{p(\vlat)} \equiv \elbofinal(\vvarpar).
\end{align*} 
\vspace{-0.7cm}

If $\Omega=\mathbb{R}^d$ is a Euclidean space,\footnote{In this paper, it always uses a Cartesian coordinate system.} we can solve the minimization problem using gradient descent (GD) as:

\vspace{-0.7cm}
\begin{align*}
\text{GD}: \,\,\, \vvarpar \leftarrow \vvarpar - \stepsize  \vg
\end{align*}
\vspace{-0.8cm}

where $\vg=\gradop{\varpar} \elbofinal(\vvarpar)$ denotes a Euclidean gradient and $\stepsize>0$ is a scalar step-size.
We can view the update  as a line (the shortest curve)  $\vL(\stepsize)$ in the Euclidean space $\mathbb{R}^d$ as $\stepsize$ varies.
Given a starting point $\vvarpar$ and a Euclidean direction $-\vg$, the line  is a  differentiable map $\vL(\stepsize)$ so that the following
ordinary differential equation\footnote{It is also known as an initial value problem. } (ODE) is satisfied.

\vspace{-0.75cm}
\begin{align}
\dot{\vL}(0) = -\vg \,\,;\,\,\, \vL(0)  =  \vvarpar; \,\,\,
\ddot{\vL}(\stepsize) = \mathbf{0} \label{eq:line_ode}
\end{align}
\vspace{-0.55cm}

where
$\dot{\vL}(x):=\frac{d \vL(\stepsize)}{d \stepsize}\big|_{\stepsize=x}\,$,
$\ddot{\vL}(x):=\frac{d^2 \vL(\stepsize)}{d \stepsize^2}\big|_{\stepsize=x} $.
The solution of the ODE is $\vL(\stepsize)= \vvarpar -  \stepsize \vg $, which is the GD update.

\subsection{Exact Riemannian Gradient Descent (RGD) }
\label{sec:ergd}

Unfortunately, $\Omega$  usually is not a Euclidean space but a Riemannian manifold with a metric.
We use a metric to characterize distances in the manifold.
A useful  metric for

\vspace{-0.5cm}
statistical manifolds is the FIM
\citep{fisher1922mathematical, rao1945information}.

Now, we generalize gradient descent  in a manifold.
First, we introduce the index convention and the Einstein summation notation used in Riemannian geometry.
The notation is summarized in Table \ref{tab:TableOfNotation}.
We denote a Euclidean gradient  $\vg$ using a subscript.
A Riemannian gradient  $\vngrad$ is denoted by a superscript.
A metric\footnote{A metric is well-defined if it is positive definite everywhere.} is used to characterize inner products and arc length in a manifold.
Given a metric $\vfim$, let $F_{ab}$ denote the element of  $\vfim$ at position $(a,b)$
and  $F^{ca}$ denote the element of $\vfim^{-1}$ at position $(c,a)$.
We use the Einstein notation to omit summation symbols such as 
$F^{ca}F_{ab}:=\sum_a F^{ca}F_{ab}$ .
Therefore, we have
$F^{ca}F_{ab}=I^{c}_{\ b}$, where $I^{c}_{\ b}$ is the element of an identity matrix at position $(c,b)$.
A Riemannian gradient is defined as $\ngrad^c := F^{ca} g_a$, where $g_a$ is the $a$-th element of a  Euclidean gradient $\vg$.
When $\vfim$ is the FIM, a Riemannian gradient becomes a natural gradient.
If the metric $\vfim$ is positive-definite for all\footnote{Such assumption is valid for  minimal EF.} $\vvarpar \in \Omega$,
an approximation family $q(\vlat|\vvarpar)$ induces a Riemannian manifold denoted by $(\Omega, \vfim)$ where $\vvarpar$ is a coordinate system.

\begin{table}[tbp]
\vspace{-0.3cm}
\caption{Table of Index Notation}
\vspace{-0.1cm}
\begin{center}
\begin{minipage}{\textwidth}
\begin{tabular}{r l p{3cm} }
\toprule
$\vvarpar^{[i]}$ & $i$-th block parameter of parameterization $\vvarpar$. \\
$\varpar^{a_i}$ & $a$-th entry of block parameter $\vvarpar^{[i]}$. \\
$\varpar^{a}$, \,  $\varpar^{(a)}$ &  $a$-th entry  of parameterization $\vvarpar$. \\
$g_a$ &  $a$-th entry  of Euclidean gradient $\vg$. \\
$\ngrad^a$, \, $\ngrad^{(a)}$ & $a$-th entry  of Riemannian/natural gradient $\vngrad$. \\
$\fim_{ab}$ & entry  of $\vfim$ with global index $(a,b)$. \\
$\fim^{ab}$ & entry  of $\vfim^{-1}$ with global index $(a,b)$. \\
$\Gamma^{c}_{\ ab}$ & entry  with global index  $(c,a,b)$. \\ 
$\fim^{a_ib_i}$ & entry  with local index $(a,b)$ in block $i$. \\
$\Gamma^{c_i}_{\ a_ib_i}$ & entry  with local index $(c,a,b)$ in block $i$. \\
\bottomrule
\end{tabular}
\end{minipage}
\end{center}
\label{tab:TableOfNotation}
\vspace{-0.6cm}
\end{table}

Like  GD,
RGD can be  derived from a geodesic,\footnote{The geodesic induces an exponential map used in exact RGD.} which is a generalization of the  ``shortest'' curve\footnote{Due to the Euler-Lagrange equation, a geodesic is a stationary curve. However, a geodesic may not be  the shortest curve.} to a  manifold.
Given a starting point $\vvarpar \in \Omega$ and a Riemannian direction $-\vngrad=-\vfim^{-1} \vg$, a geodesic is a differentiable map $\vL(t)$ so that
the following geodesic ODE\footnote{The domain of  $\vL(\stepsize)$ is $\mathbb{R}$ for a complete manifold.} is satisfied.
\begin{align}
& \dot{L}^{c}(0)   = - \fim^{ca} g_a \,\,;\,\,\,  
L^{c}(0)  =  \varpar^{c}  \\
&\ddot{L}^{c}(\stepsize) =- \Gamma_{\ ab}^{c}(\stepsize) \dot{L}^{a}(\stepsize)\dot{L}^{b}(\stepsize)  \label{eq:sed_ode}
\end{align}
\vspace{-0.8cm}

where
$L^c(\stepsize)$ is the $c$-th element of $\vL(\stepsize)$, 
$\dot{L}^{c}(x):= \frac{d L^c(\stepsize)}{d \stepsize}\big|_{\stepsize=x}$, $\, \ddot{L}^{c}(x):=  \frac{d^2 L^c(t)}{d \stepsize^2}\big|_{\stepsize=x} $,
$\Gamma_{\ ab}^{c}(\stepsize):= \Gamma_{\ ab}^{c} \bigr|_{ \varpar= L(\stepsize) }$.
$ \Gamma_{\ ab}^{c}$ is the \emph{Christoffel symbol of the 2nd kind} defined by
\begin{align*}
\Gamma^{c}_{\ a b} := \fim^{cd} \Gamma_{d,a b} \,\,;\,\,\,
\Gamma_{d,a b} := \half \sqr{ \gradop{a}{\fim_{bd}} +\gradop{b}{\fim_{ad}} - \gradop{d}{\fim_{ab}} }
\end{align*} where
$\gradop{a}:=\gradop{\varpar^{a}}$ is for notation simplicity and 
$\Gamma_{d,a b}$ is  the \emph{Christoffel symbol of the 1st kind}.
$\ddot{\vL}$ characterizes the curvature of a geodesic since a manifold is not flat in general.
In Euclidean cases, the metric $\vfim=\vI$ is a constant identity matrix and
\eqref{eq:sed_ode} vanishes since $\Gamma_{d,a b}$ and $\Gamma^{c}_{\ a b}$ are zeros, which implies Euclidean spaces are flat. Therefore, we recover the GD update in \eqref{eq:line_ode} since $\vngrad=\vI^{-1} \vg=\vg$.

Given any parameterization with the FIM, we can compute  $\Gamma_{d,a b}$  by using Eq \eqref{eq:chris} in Appendix \ref{app:chris_general}, which involves extra integrations.
We will show that a BCN parameterization can get rid of the extra integrations (see Theorem \ref{thm:EF_chris}).

\subsection{Our Rule  as an Inexact RGD Update}
\label{sec:irgd}
However, it is hard to exactly solve the geodesic ODE. 
 Inexact RGD is  derived by approximating the geodesic.\footnote{\label{ft:rep1}A retraction map can be derived by approximating the geodesic. An exact RGD update is invariant under parameterization while inexact RGD updates  including NGD  often are not.}
Recall that the original rule is a natural gradient descent (NGD) update. NGD  
can be derived by the first-order approximation of the geodesic $\vL(\stepsize)$ at $\stepsize_0 =0$ with the FIM.
\begin{align*}
\text{NGD}: \,\,\, \vvarpar \leftarrow   \vL(\stepsize_0) +    \dot{\vL}(\stepsize_0)  (\stepsize - \stepsize_0) = \vvarpar - \stepsize  \vngrad
\end{align*} 
\vspace{-0.6cm}

Unfortunately, this approximation is only well-defined in a small neighborhood at $t_0$ with radius $\stepsize$. 
For a stochastic NGD update, the step-size $\stepsize$ is very small, which often result in slow convergence.
Our learning rule addresses this issue, which is indeed a new inexact RGD update.
Moreover, our update can use a bigger step-size and often converges faster than NGD without introducing significant computational overhead in useful cases
such as gamma, Gaussian, MOG.

Consider cases when $\vvarpar=\{ \vvarpar^{[1]}, \dots, \vvarpar^{[m]} \}$ has $m$ blocks. 
We can express a Riemannian gradient as $\vngrad=\{ \vngrad^{[1]}, \dots, \vngrad^{[m]} \}$.
We use the block summation notation to omit the summation signs in \eqref{eq:our_iblr} as $\Gamma_{\ a_ib_i}^{c_i}\ngrad^{a_i}\ngrad^{b_i}:= \sum_{a_i} \sum_{b_i}\Gamma_{\ a_ib_i}^{c_i}\ngrad^{a_i}\ngrad^{b_i}$.  By the global index notation, we have $\Gamma_{\ a_ib_i}^{c_i}\ngrad^{a_i}\ngrad^{b_i}=\sum_{a \in [i]}\sum_{b \in [i]}\Gamma_{\, \, \, \, \, \, ab}^{(c_i)}\ngrad^{a}\ngrad^{b}$, where $[i]$ is the index set for block $i$, $(c_i)$ is the corresponding global index of local index $c_i$, and $a$ and $b$ are global indexes.

We can extend the definition of a BC parameterization to any Riemannian metric $\vfim$.
Given a metric $\vfim$, we have Lemma \ref{lemma:block_identity} for any block $i$ (see Appendix \ref{app:block_riem_general} for a proof) :
\begin{lemma}
\label{lemma:block_identity}
 When $\vvarpar$ is a BC parameterization of metric $\vfim$, we have $\ngrad^{a_i}=\fim^{a_i b_i}g_{b_i}$ and $\Gamma_{\ \ a_ib_i}^{c_i}=\fim^{c_id_i}\Gamma_{d_i,a_ib_i}$.
\end{lemma}

Given a  manifold  equipped with metric  $\vfim$ and a BC parameterization $\vvarpar$,
 consider the solution of the block-wise (geodesic)  ODE\footnote{\label{ft:rep2}$\vR^{[i]}(\stepsize)$ is easier to solve compared to $\vL(\stepsize)$. Note that
 \eqref{eq:rgd_ret} is the minimum requirement of a retraction map \cite{absil2009optimization}.
}  denoted by $\vR^{[i]}(\stepsize)$  for block $i$:
\begin{align}
&\dot{R}^{\ c_i}(0)  = - \fim^{c_i a_i} g_{a_i} \,\,;\,\,\,  
R^{\ c_i}(0)  =  \varpar^{c_i}  \label{eq:rgd_ret} \\
& \ddot{R}^{\ c_i}(\stepsize)= -  \Gamma_{\ a_ib_i}^{c_i}(\stepsize) \dot{R}^{\ a_i}(\stepsize) \dot{R}^{\ b_i}(\stepsize)
\end{align} where
$R^{\ c_i}(0)$,  $\dot{R}^{\ c_i}(0)$, $\ddot{R}^{\ c_i}(\stepsize)$ respectively denote the $c$-th entry of $\vR^{[i]}(0)$, $\dot{\vR}^{[i]}(0)$, and $\ddot{\vR}^{[i]}(\stepsize)$; 
$\Gamma_{\ a_i b_i}^{c_i}(\stepsize):= \Gamma_{\ a_i b_i}^{c_i} \bigr|_{  \varpar^{[i]}= R^{[i]}(\stepsize)}^{\varpar^{[-i]}=R^{[-i]}(0)  }$.

We use $\hat{q}(\vlat|\vvarpar^{[i]})$ to denote $q(\vlat|\vvarpar)$  when
$\vvarpar^{[j]}=\vR^{[j]}(0) $ is known for each block $j$  except the $i$-th block.
 $\fim^{c_i a_i}$ is the entry of $(\vfim^{[i]})^{-1}$ at position $(c,a)$, where $\vfim^{[i]}$ is the sub-block matrix of $\vfim$ for $\vvarpar^{[i]}$.
In fact, $\vfim^{[i]}$ is the induced metric  for $\hat{q}(\vlat|\vvarpar^{[i]})$, and $\vR^{[i]}(\stepsize)$ is a geodesic  for $\hat{q}(\vlat|\vvarpar^{[i]})$ under  BC parameterization $\vvarpar$.
Moreover, if $\vvarpar$ is a BCN parameterization and $\vfim$ is the FIM, we have $\fim_{a_i b_i}=\partial_{\varpar^{a_i}}\partial_{\varpar^{b_i}} A(\vvarpar)$.

We define a curve $\vR(\stepsize):=\{\vR^{[1]}(\stepsize), \dots, \vR^{[m]}(\stepsize)\}$.
By Lemma \ref{lemma:block_identity}, we can show that the first-order approximation of $\vR(\stepsize)$ at $\stepsize_0=0$  induces NGD if $\vfim$ is the FIM and $\vvarpar$ is a BC parameterization.
Appendix \ref{app:ngd_riem} shows this in detail.

We propose to use the second-order approximation\footnote{Our approximation allows us to use a bigger step-size than NGD. 
In many cases, the underlying parameterization constraints are  satisfied regardless of the choice of the step-size and therefore, a line search for the constraint satisfaction is no longer required.} of $\vR(\stepsize)$ at $\stepsize_0=0$ for block $i$,
where $\vvarpar$ is a BC parameterization.
\begin{align}
\text{Our}: \varpar^{c_i} &\leftarrow   R^{c_i}(\stepsize_0) +  \dot{R}^{c_i}(\stepsize_0)  (\stepsize -\stepsize_0) { \color{red}+  \half  \ddot{R}^{c_i}(\stepsize_0) (\stepsize-\stepsize_0)^2 } \nonumber \\
&= \varpar^{c_i} - \stepsize \ngrad^{c_i} \color{red}{- \frac{ \stepsize^2}{2}\Gamma_{\ a_ib_i}^{c_i}\ngrad^{a_i}\ngrad^{b_i} } \label{eq:ibrl_bc}
\end{align} where 
$\Gamma_{\ a_ib_i}^{c_i}$ is computed at $\stepsize_0=0$,
$\ngrad^{c_i} = \fim^{c_i a_i} g_{a_i}$,
and $c_i$ denotes the $c$-th element of the $i$-th block.
 Our rule works for both a BCN parameterization and a BC parameterization.

\citet{song2018accelerating} suggest  using the  second-order approximation of  $\vL(\stepsize)$ at $\stepsize_0=0$, which has to compute the whole Christoffel symbol $\Gamma_{\, \, \, \, a b}^{(c_i)}$.  However, their proposal  does not guarantee
the update stays in the constraint set even in univariate Gaussian cases (see Appendix \ref{app:case}). Moreover, it is inefficient to compute the whole Christoffel symbol since
all cross terms between any two blocks are needed in 
$\Gamma_{\, \, \, \, \, \, ab}^{(c_i)}\ngrad^{a}\ngrad^{b}$.
When a parameterization has $m$ blocks,
 $\Gamma_{\, \, \, \, a_ib_i}^{c_i}\ngrad^{a_i}\ngrad^{b_i} \neq \Gamma_{\, \, \, \, \, \, ab}^{(c_i)}\ngrad^{a}\ngrad^{b}$ since the hidden summations are taken over  entries only in  block $i$  on the left while the summations are taken over entries in  all $m$ blocks  on the right.
This is the key difference between  our method  and their method.
In our method, only the block-wise  $\Gamma_{\, \, \, \, a_i b_i}^{c_i}$  is computed, which makes our method efficient in many cases such as multivariate Gaussians and MOGs. 
Moreover, a BCN parameterization can further simplify the computation of our rule due to
Theorem  \ref{thm:EF_chris} (see Appendix \ref{app:ef_rgvi} for a proof).

\begin{thm}
\label{thm:EF_chris}
Under a BCN parameterization of EF with the FIM, natural gradients and  the Christoffel symbol for each block $i$ can be simplified as
\begin{align*}
 \ngrad^{a_i} = \partial_{\varmean_{a_i}}\elbofinal(\vvarpar) \,\,; \,\, \Gamma_{d_i,a_ib_i} = \half \partial_{\varpar^{a_i}} \partial_{\varpar^{b_i}} \partial_{\varpar^{d_i}}A(\vvarpar)
\end{align*}
where
$\varpar^{a_i}$ is the $a$-th entry of $\vvarpar^{[i]}$;
$\varmean_{a_i}$ is the $a$-th entry of the BC expectation parameter\footnote{Instead of using $\vvarmean^{[i]}$, we use $\vvarmean_{[i]}$ to emphasize that Euclidean gradient for  $\vvarmean_{[i]}$  is equivalent to natural gradient for $\vvarpar^{[i]}$.
}
$\vvarmean_{[i]}:=\Unmyexpect{q}{\big[ \vphi_i\big( \vlat,\vvarpar^{[-i]} \big) \big]  }= \partial_{\varpar^{[i]}}A(\vvarpar)$.
\end{thm}

Since $A(\vvarpar)$ is $C^3$-smooth for block $\vvarpar^{[i]}$ \cite{johansen1979introduction}, we have $\partial_{\varpar^{a_i}} \partial_{\varpar^{b_i}} \partial_{\varpar^{d_i}}A(\vvarpar) =   \partial_{\varpar^{d_i}} \partial_{\varpar^{a_i}} \partial_{\varpar^{b_i}}A(\vvarpar)$.
Thus, by Theorem \ref{thm:EF_chris},
 we have $\Gamma_{\, \, \, \, \, \, a_i b_i}^{c_i} = \half  \partial_{\varmean_{c_i}} \partial_{\varpar^{a_i}} \partial_{\varpar^{b_i}}A(\vvarpar)$.
 

A similar theorem for EF mixtures is  in Appendix \ref{app:cef_rgvi}.

To sum up, our rule is an instance of RGD with a retraction map.
We give a principled way to derive Riemannian gradients and retractions from scratch (see Footnote \ref{ft:rep1},\ref{ft:rep2}).
The convergence analysis could be obtained by existing works \cite{bonnabel2013stochastic} if the retraction satisfies some properties.

\section{Numerical Results\protect\footnote{Our implementation:
\href{https://github.com/yorkerlin/iBayesLRule}{github.com/yorkerlin/iBayesLRule}}
}
\begin{figure*}[t]
\centering
\hspace*{-1.5cm}
\subfigure[]{
\label{fig:a}
	\includegraphics[width=0.26\linewidth]{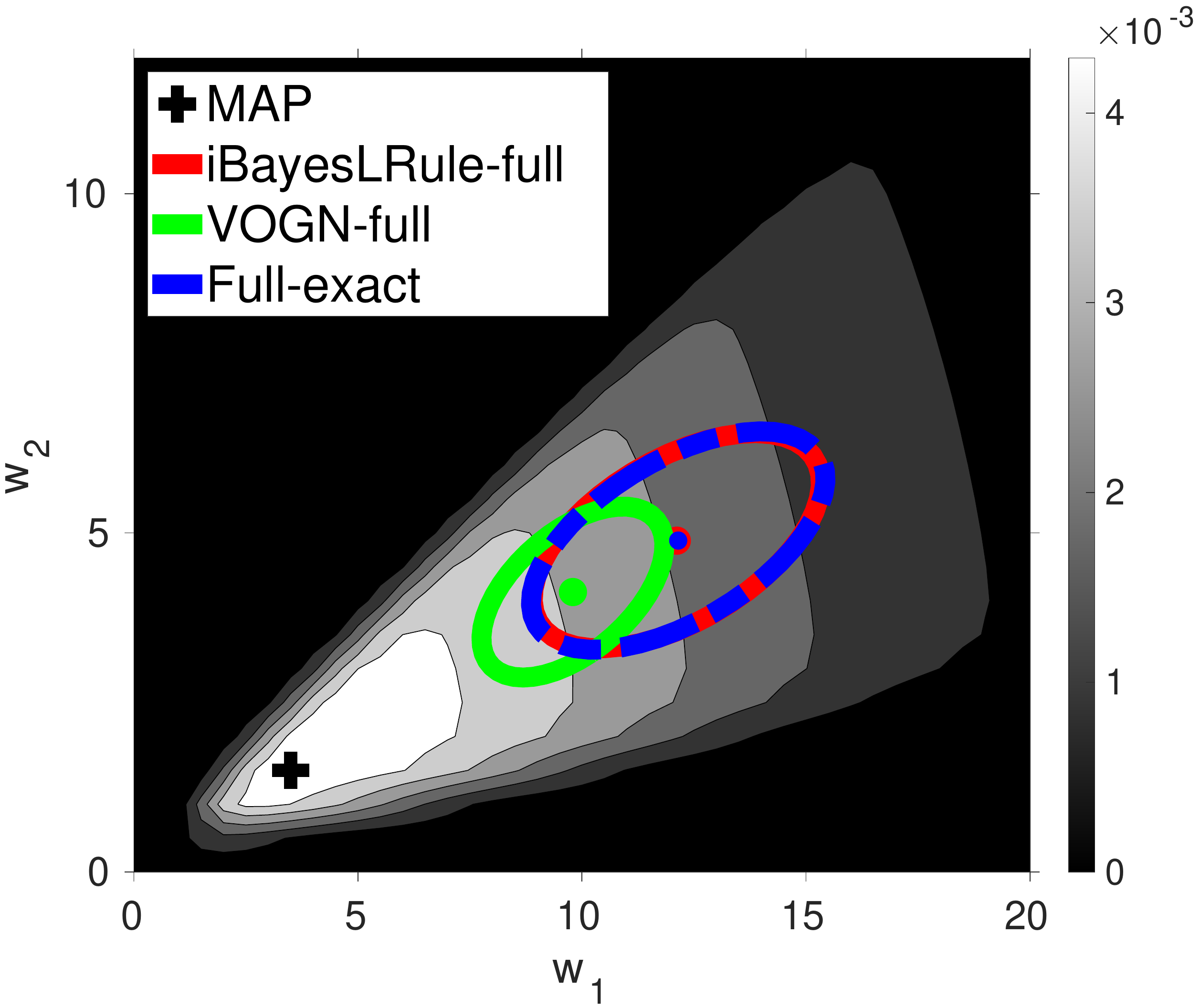}
}
	\hspace*{-0.4cm}
\subfigure[]{
\label{fig:b}
	\includegraphics[width=0.26\linewidth]{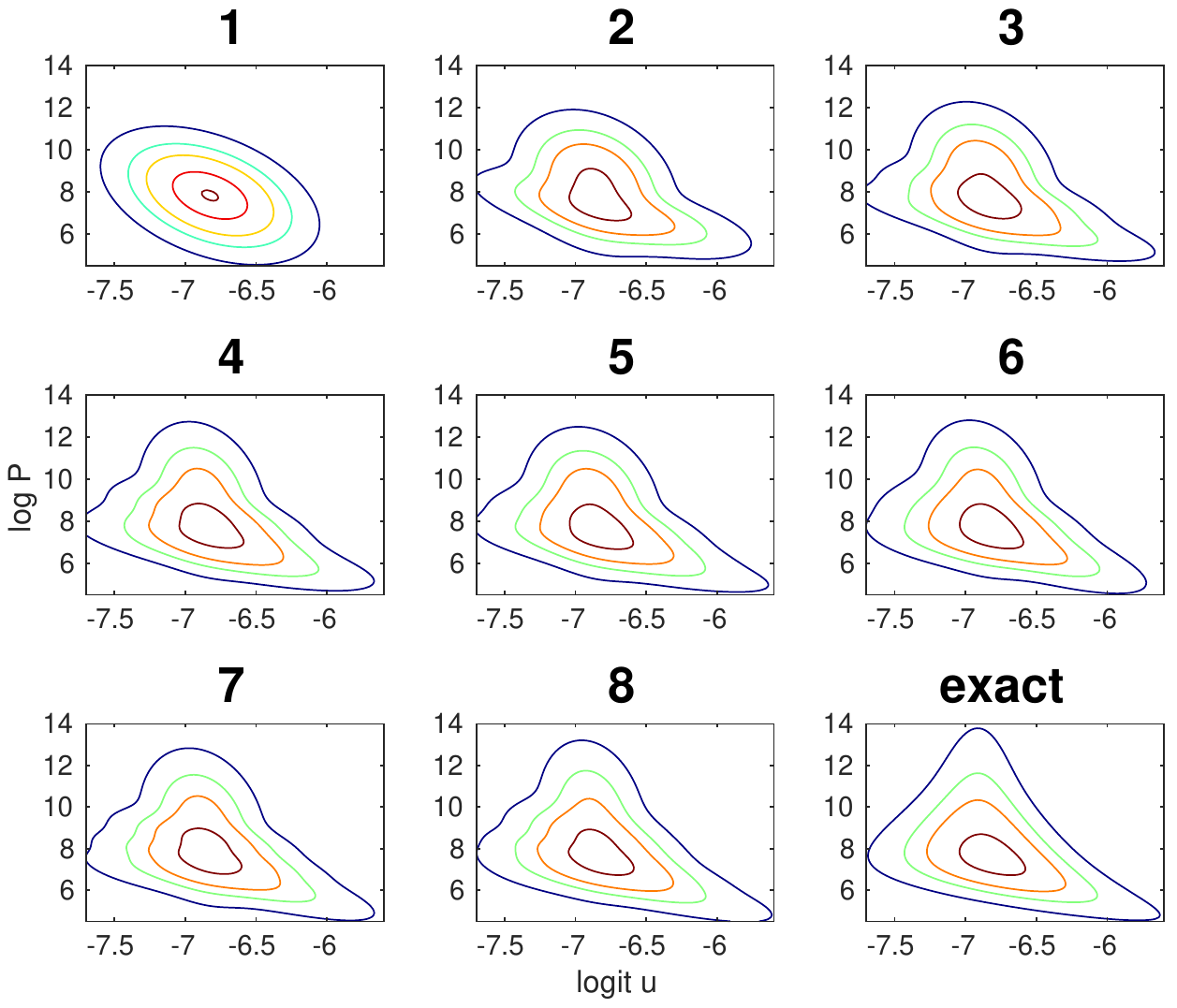}
}
	\hspace*{-0.4cm}
\subfigure[]{
\label{fig:c}
	\includegraphics[width=0.26\linewidth]{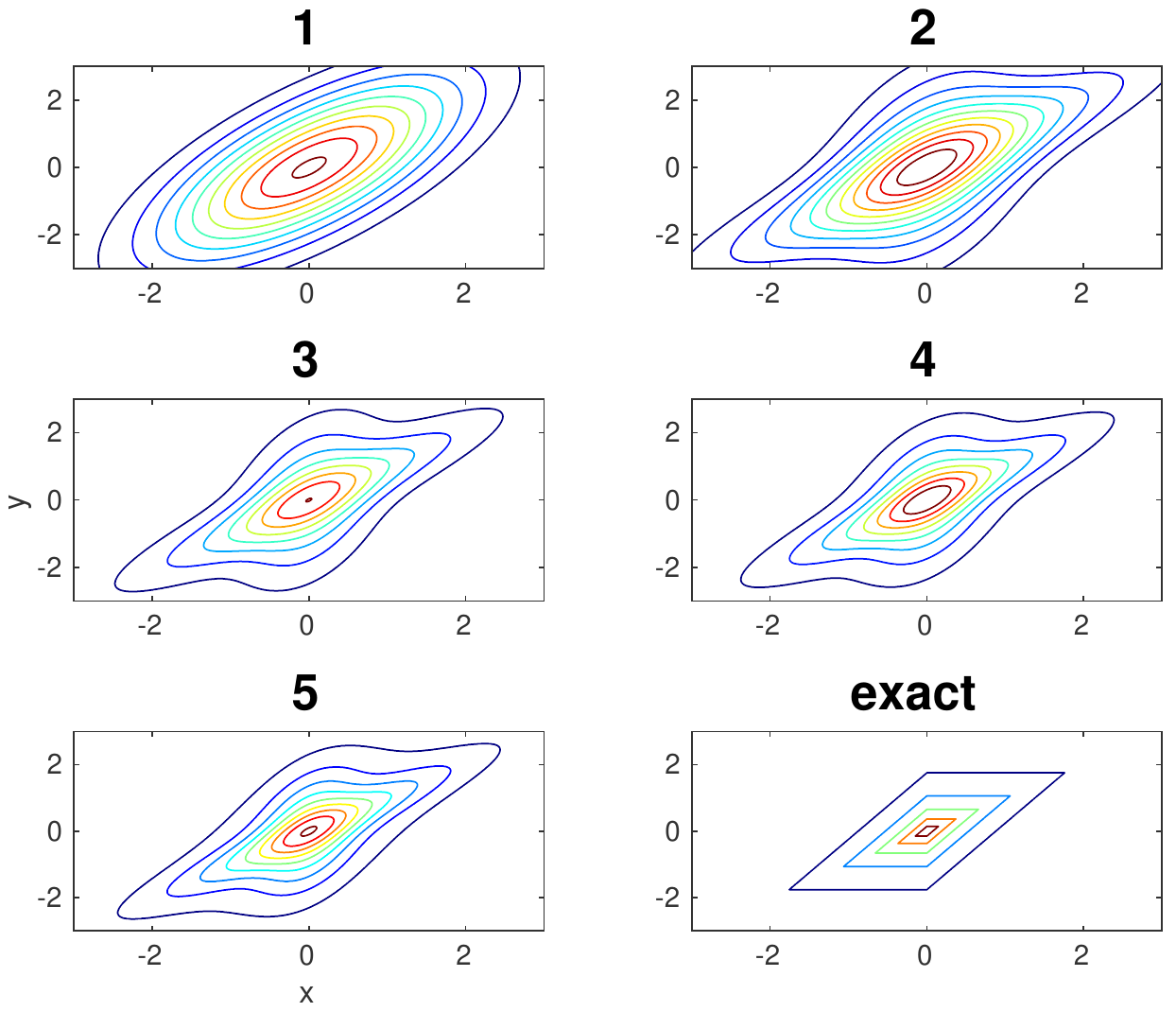}
}
	\hspace*{-0.4cm}
\subfigure[]{
\label{fig:d}
	\includegraphics[width=0.26\linewidth]{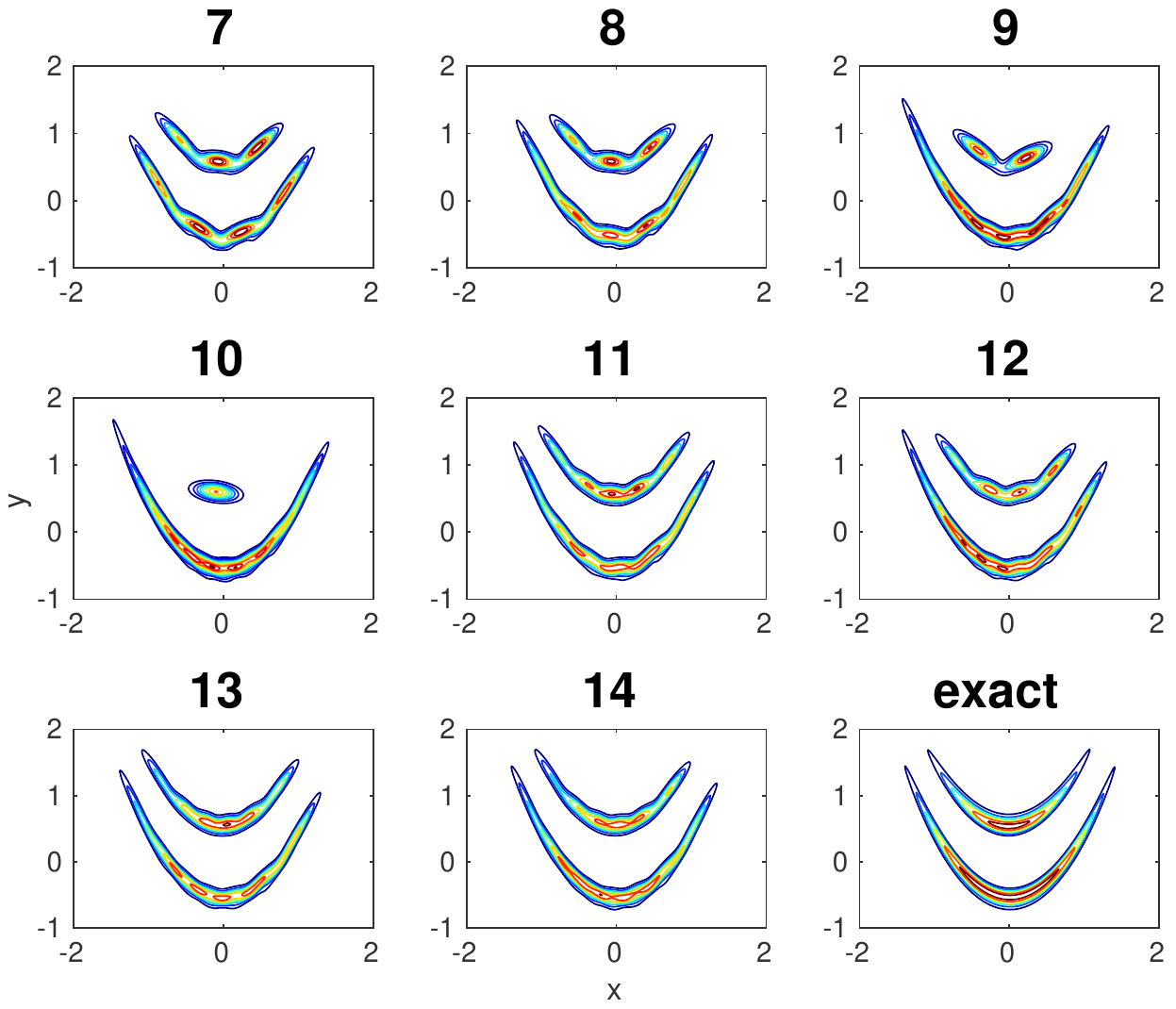}
}
\hspace*{-1.2cm}
\vspace{-0.4cm}
   \caption{Visualization of posterior approximations on 2-D toy examples. Figure \ref{fig:a} shows the Gaussian approximation to fit a Bayesian logistic model, where our approximation matches the exact variational Gaussian approximation.
Figure \ref{fig:b} shows MOG approximation fit to a beta-binomial model in a 2-D problem. The number indicates
 the number of mixture components. By increasing the number of components, we get better results. Figure \ref{fig:c} shows MOG approximation fit to a correlated 2-D Laplace distribution. The number indicates the number of mixtures. We get smooth approximations of the non-smooth distribution.
Figure \ref{fig:d} shows MOG approximation fit to a double banana distribution. The number indicates the number of mixtures, where we only show the last 8 MOG approximations.
The complete MOG approximations can be found in Appendix \ref{app:more_plots}.
 As the number of components increases, we get better results.
 }
\label{figure:toy_examples}
\end{figure*}

\begin{figure*}[t]
	\centering
	\hspace*{-1.2cm}
	\includegraphics[width=0.3\linewidth]{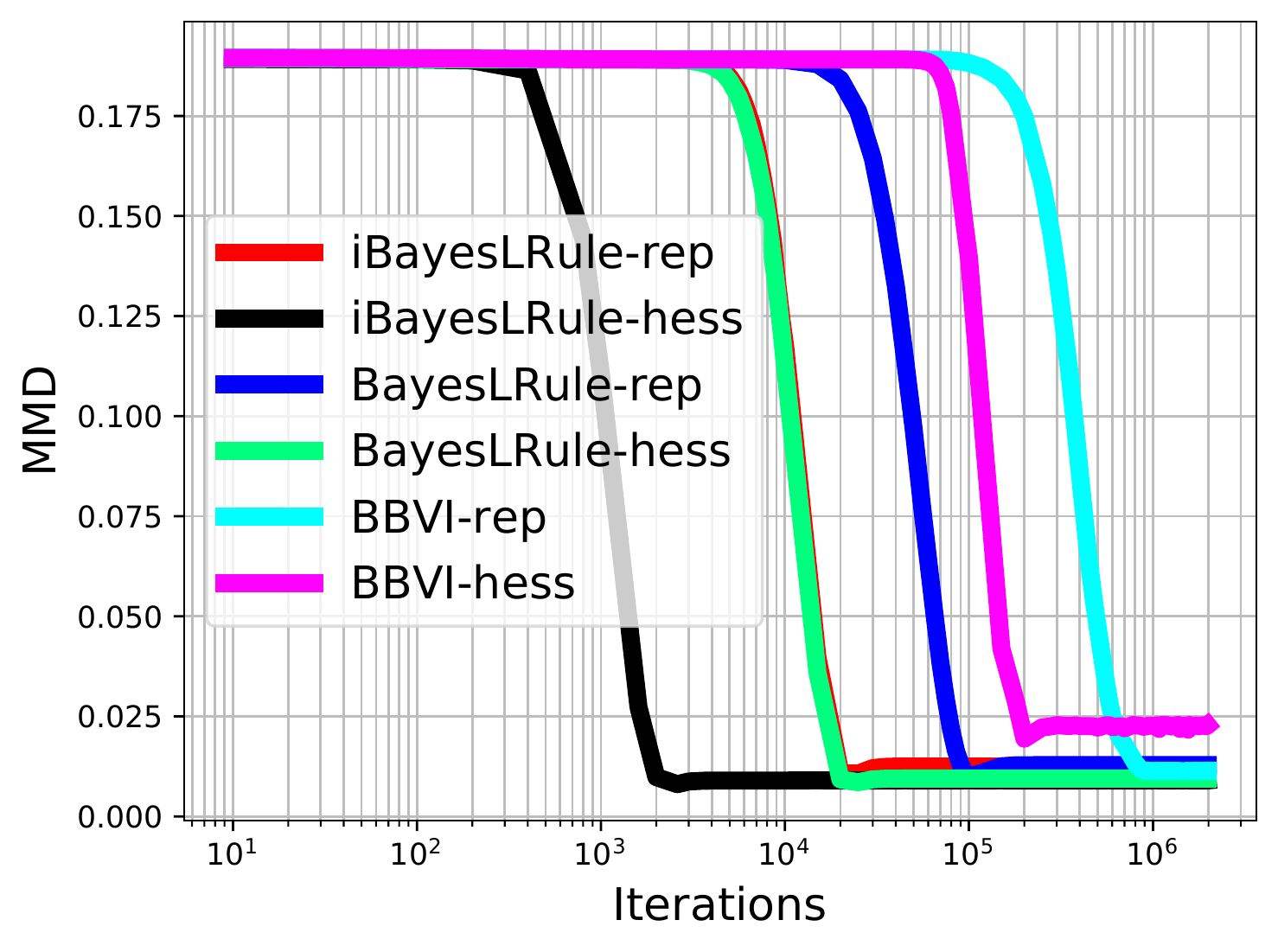}
	\includegraphics[width=0.3\linewidth]{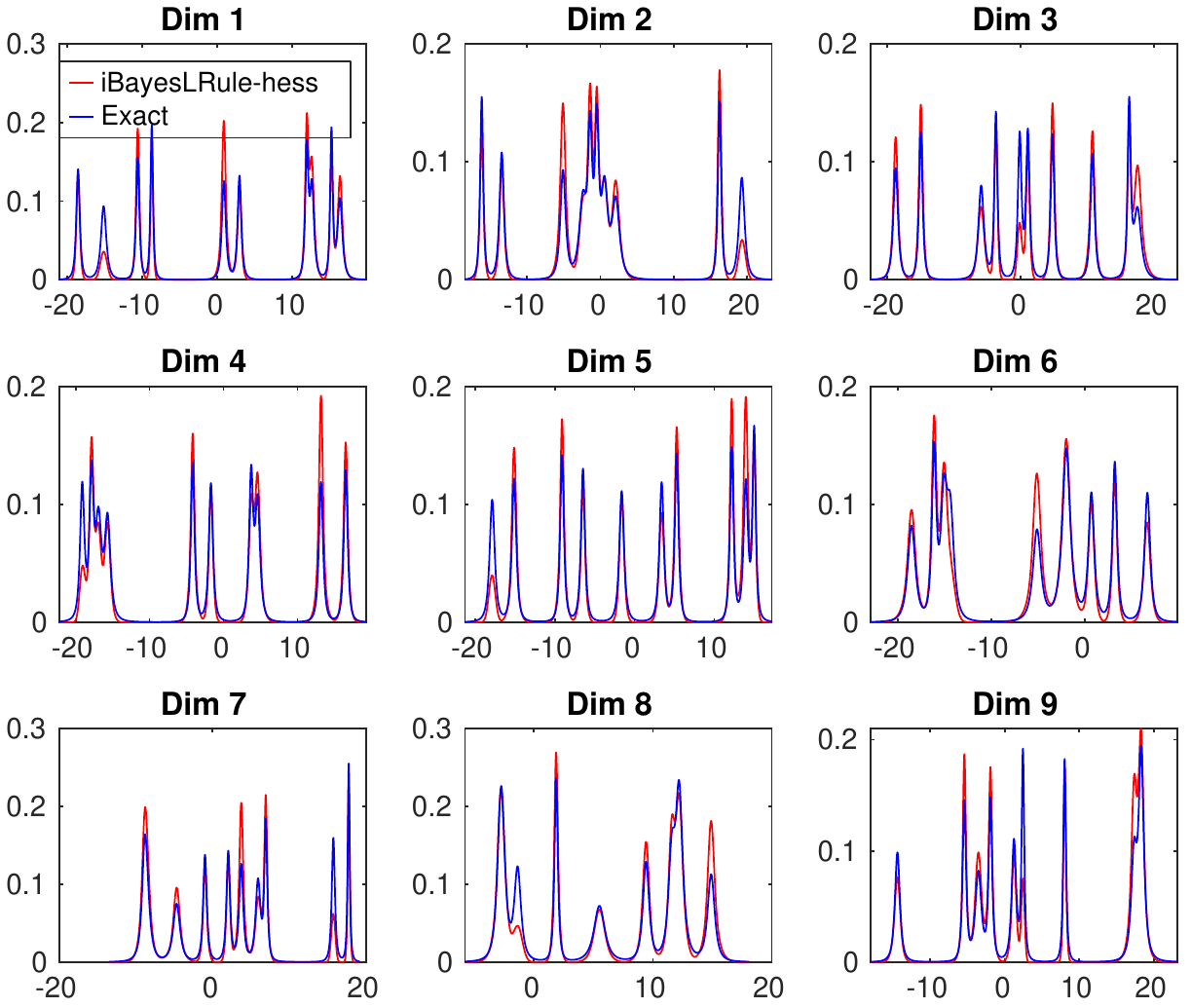}
	\includegraphics[width=0.3\linewidth]{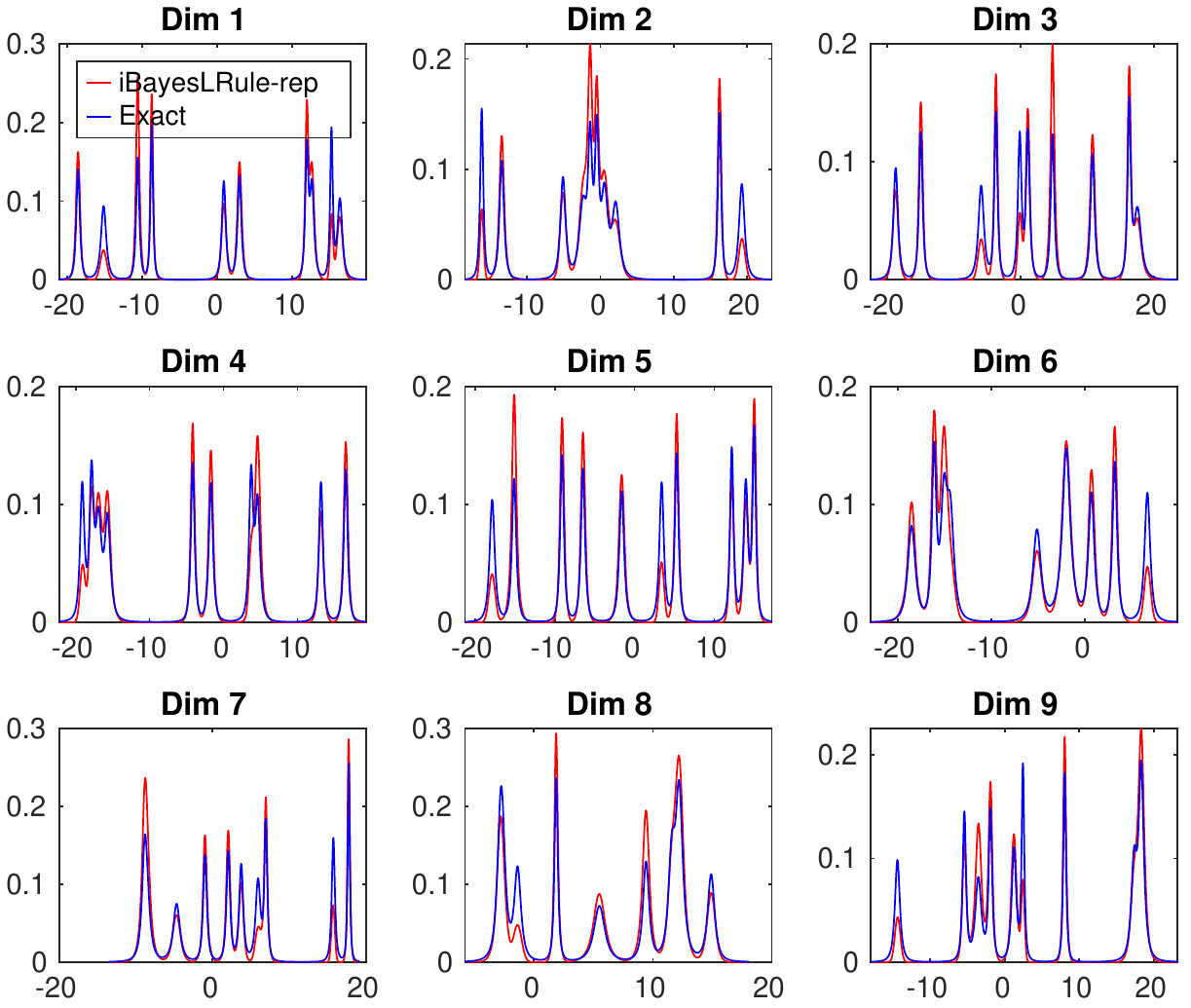}
	\hspace*{-1.2cm}
	\vspace{-0.2cm}
   \caption{Comparison results on a 20-D mixture of Student's t distributions  with $10$ components by MOG approximations. The leftmost figure shows the performance of each  method, where our method outperforms existing methods.
The first 9 dimensions obtained by our method are shown in the figure where MOG approximation fits the marginals well.
We also test a 300-D mixture problem in Appendix \ref{app:more_plots}.
}
	\label{figure:mixT}
\end{figure*}

\begin{figure*}[t]
	\centering
	\hspace*{-1.5cm}
\subfigure[]{
\label{fig:ra}
	\includegraphics[width=0.25\linewidth]{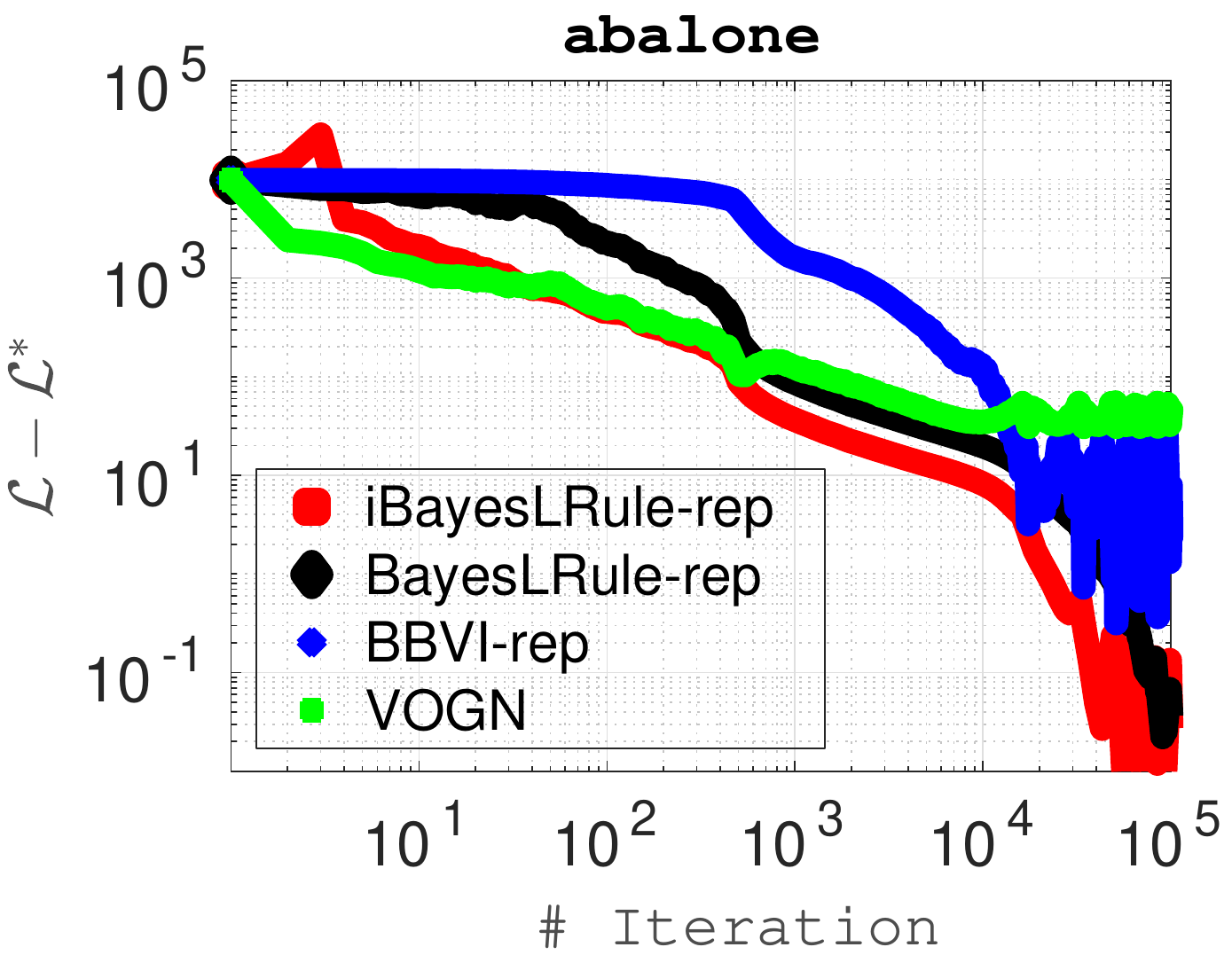}
	}
	\hspace*{-0.2cm}
\subfigure[]{
\label{fig:rb}
	\includegraphics[width=0.25\linewidth]{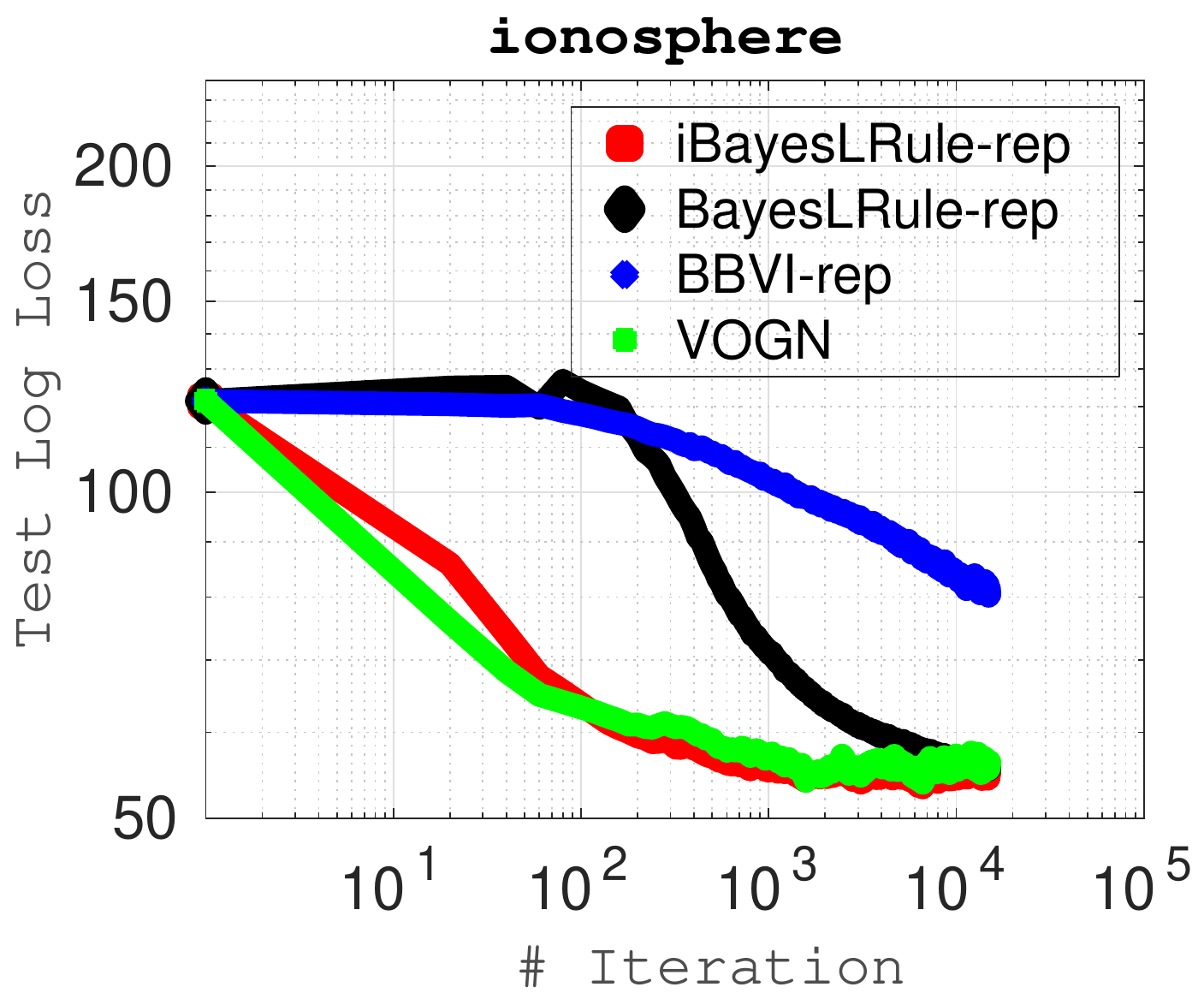}
	}
	\hspace*{-0.2cm}
\subfigure[]{
\label{fig:rc}
	\includegraphics[width=0.25\linewidth]{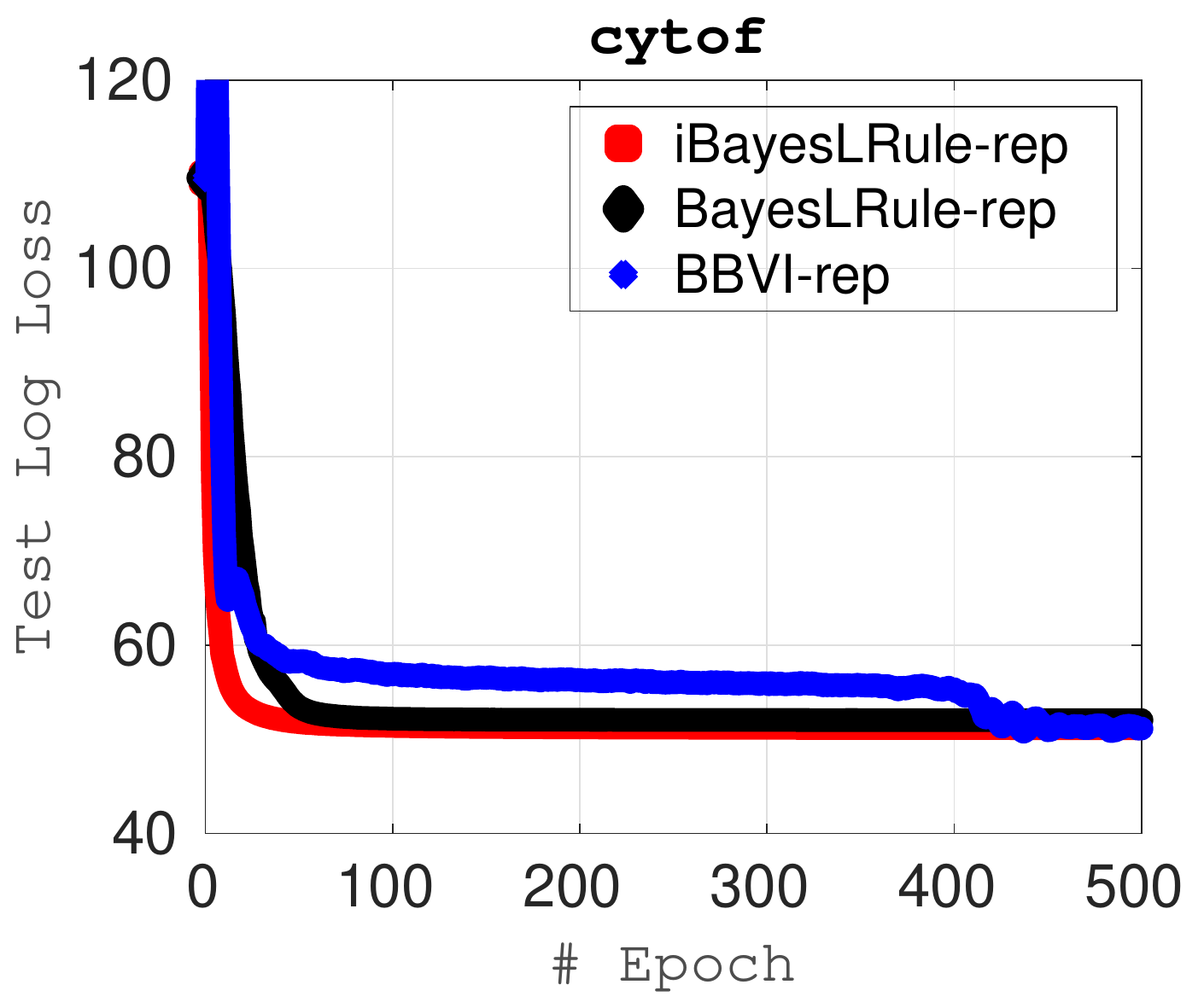}
	}
	\hspace*{-0.2cm}
\subfigure[]{
\label{fig:rd}
	\includegraphics[width=0.25\linewidth]{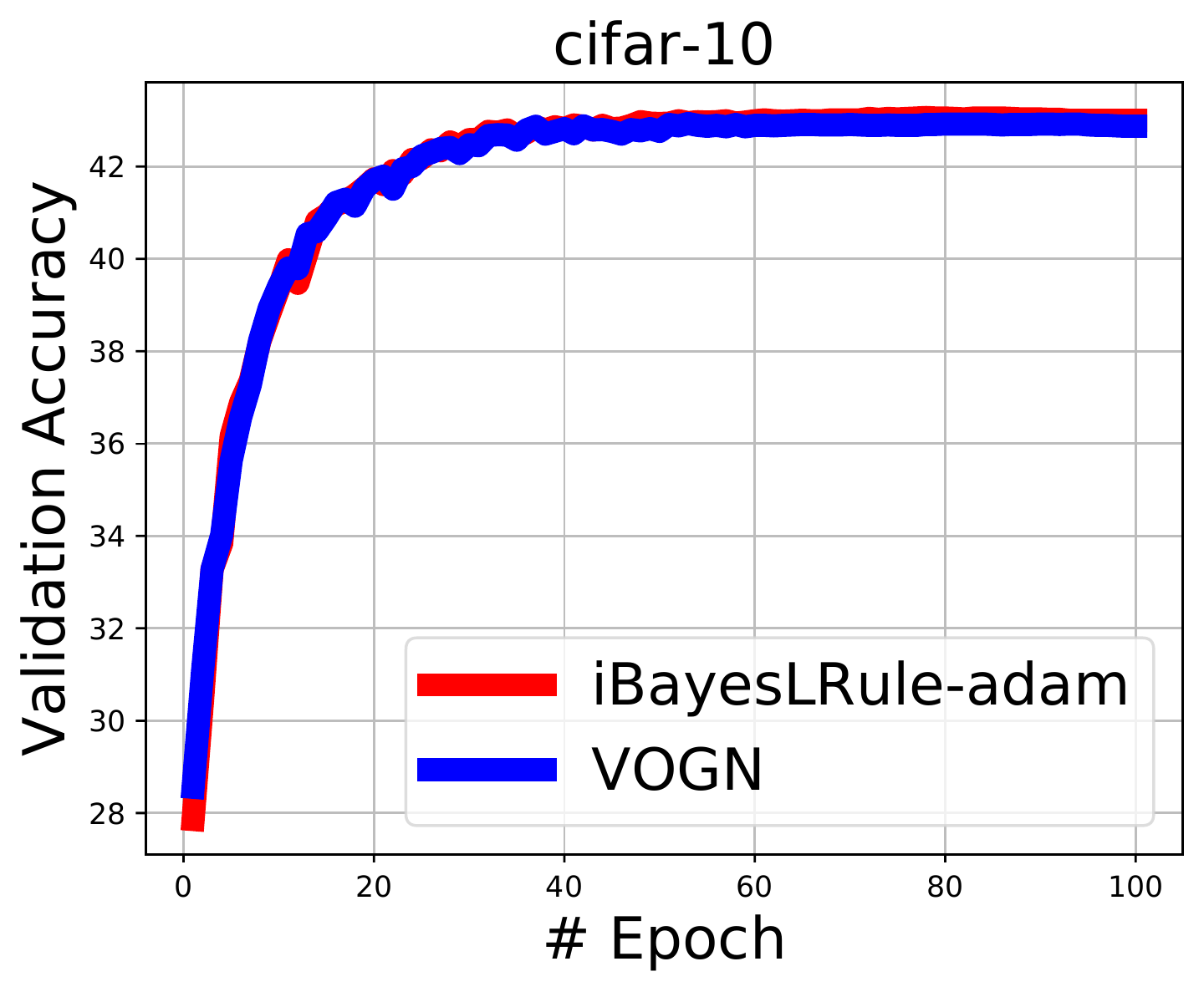}
	}
	\hspace*{-1.5cm}
\vspace{-0.4cm}
   \caption{
   Results on real-world datasets showing the performances of our method (iBayesLRule) highlighted in red compared to BBVI, BayesLRule with a line search, and VOGN.
  Figure \ref{fig:ra}  and \ref{fig:rb} show the performances  using Gaussian approximations with full covariance structure to fit a Bayesian linear regression and a Bayesian logistic regression, respectively, where our method converges faster than  BayesLRule and BBVI and gives a more accurate approximation than VOGN.
   Figure \ref{fig:rc}  shows the performances   using Gamma approximations to fit a Gamma factor model, where our method converges faster.
    Figure \ref{fig:rd}  shows the performances of methods in a Bayesian MLP network with diagonal Gaussian approximations, where our method performs comparably to VOGN.
   }
	\label{figure:real_examples}
\vspace{-0.2cm}
\end{figure*}

\subsection{Results on Synthetic Examples}
\label{sec:syn_examples}
To validate our rule, we visualize our approximations in 2-dimensional toy examples, where we use the re-parametrization trick  suggested  by \citet{lin2019fast,wu-report} (see \eqref{eq:gauss_cov_rep} in Appendix \ref{app:gauss_case} for full  Gaussian  and \eqref{eq:mog_cov_rep} in Appendix \ref{app:mog_case} for mixture of Gaussians (MOG)) 
to compute gradients.
Due to the  trick, $\nabla_\lat^2 \bar{\ell}(\vlat)$ is not needed.
See Figure \ref{figure:banana_plots}-\ref{figure:blr_plots} in Appendix \ref{app:more_plots} for 
more visualization examples such as the banana distribution \citep{haario2001adaptive} and a BNN example taken from \citet{au2020manifold}.
We then compare our method to baseline methods in a higher dimensional example.

We first visualize Gaussian approximations with full covariance structures for the Bayesian Logistic regression example taken from \citet{murphy2013machine} ($N = 60, d = 2$).
Figure \ref{fig:a} shows posterior approximations obtained from various methods. From the figure, our approximation matches the exact variational Gaussian approximation.
For skew-Gaussian approximations \citep{lin2019fast} and mean-field Gaussian approximations, see
Figure \ref{figure:blr_plots} in Appendix \ref{app:more_plots}.

In the second example, we approximate the beta-binomial model  \citep{salimans2013fixed} ($N = 20, d = 2$) by MOG.
The exact posterior is skewed.
From Figure \ref{fig:b}, we see that the approximation matches the exact posterior better and better as we increase the number of mixtures.

In the third example, we approximate a correlated Laplace distribution $\exp( - \bar{\ell}(\vlat) ) =  \lapdist(\lat_1|0,1)  \lapdist(\lat_2|\lat_1,1)$ by using MOG, where 
$ \lapdist(\lat_2|\lat_1,1)=\half \exp(-| \lat_2 - \lat_1 |)$. The distribution is non-smooth and thus
 $\nabla_\lat^2 \bar{\ell}(\vlat)$ does not exist. However,
we can use the re-parametrization trick since
 $\nabla_\lat \bar{\ell}(\vlat)$ exists almost surely.
From Figure \ref{fig:c}, we see that our method gives smooth approximations of the target.

In the fourth example, we approximate the double banana distribution constructed by \citet{detommaso2018stein}. The true distribution has two modes and is skewed.
As shown in  Figure \ref{fig:d}, our MOG approximation approximates the target  better and better when we increase the number of mixtures. 
For a complete plot using the approximation with various components, see
Figure \ref{fig:mb} in Appendix \ref{app:more_plots}.

Finally, we conduct a comparison study on approximations for a mixture of Student's t distributions $\exp( - \bar{\ell}(\vlat) ) =\frac{1}{C}\sum_{k=1}^{C} \Student(\vlat|\vu_k,\vV_k,\alpha)$ with degrees of freedom $\alpha=2$, where $\vlat \in \mathbb{R}^d$.
We generate each entry of location vector $\vu_k$ uniformly in an interval $(-s,s)$. Each shape matrix $\vV_k$ is taken a form of $\vV_k=\vA_k^T \vA_k + \vI_d$, where each entry of the $d \times d$ matrix $\vA_k$ is independently drawn from a Gaussian distribution with mean $0$ and standard deviation $0.1d$.
We approximate the posterior distribution by MOG with $K$ components and
 use the  importance sampling technique to compute gradients as suggested by \citet{lin2019fast} so that the number of Monte Carlo (MC) samples does not depend  the number of Gaussian components $K$.
We compare our method to  existing methods, where the BayesLRule for MOG is proposed by \citet{lin2019fast}.

We consider a case with $K = 25, C = 10, d = 20, s=20$.
For simplicity, we fix the mixing weight to be $\frac{1}{K}$ and only update each Gaussian component with the precision $\vS_c$ and the mean $\vmu_c$ during training.
We use 10 MC samples to compute  gradients, where  gradients are computed using either the re-parametrization trick  (referred to as ``-rep'') as shown in \eqref{eq:mog_cov_rep} in Appendix \ref{app:mog_case}  or the Hessian trick (referred to as ``-hess'') as shown in \eqref{eq:mog_cov_hess} in Appendix \ref{app:mog_case} .
Note that BayesLRule with either the re-parametrization trick or the Hessian trick does not stay in the constraint set.
We use the same initialization and tune the step size by grid search for each method .
The leftmost plot of Figure \ref{figure:mixT} shows the performance. We clearly see that our methods converge fastest, when we use the 
maximum mean discrepancy (MMD) to measure the difference between an approximation and the ground-truth.
The remaining plots of Figure \ref{figure:mixT} show the first 9 marginal distributions of the true distribution and our  approximations with two kinds of gradient estimation, where  MOG closely matches the marginals.
All 20 marginal distributions are in Figure \ref{figure:mixT20_plots} in Appendix \ref{app:more_plots}.
We also consider
a more difficult case with $K=60, C=20, d=300, s=25$ using 10 MC samples.
Figure \ref{figure:mixT300_p1_plots}-\ref{figure:mixT300_p5_plots} in Appendix \ref{app:more_plots} show all 300 marginal distributions obtained by our method.

\subsection{Results on Real Data}
Now, we show results on real-world datasets. We consider four models in our experiments.
The first model is the Bayesian linear regression, where we can obtain the exact solution and the optimal negative ELBO denoted by $\elbofinal^*$.
We present
results for full Gaussian approximations on the ``Abalone'' dataset ($N=4,177, d=8$)
with 3341 chosen for  training. We train the model with mini-batch size 168.
In Figure  \ref{fig:ra}, we plot the difference of ELBO between the exact and an approximation.
We compare our
method (referred to as ``iBayesLRule'' ) to 
the black-box gradient method (referred to as ``BBVI'' ) using the Adam optimizer \citep{kingma2014adam} and 
the original Bayesian learning rule (referred to as ``BayesLRule'' )
with the re-parametrization trick (referred to as ``-rep'') and the VOGN method.
 BBVI requires us to use an unconstrained parametrization.
BayesLRule with the re-parametrization trick does not stay in the constraint set so a line search has to be used in BayesLRule.
We can see that our method converges faster than BayesLRule and BBVI and is more accurate than VOGN.

Next, we consider the Bayesian logistic regression and present
results for full Gaussian approximations on the ``Ionosphere'' dataset ($N=351, d=34$)
with 175 chosen for training. We train the model with mini-batch size 17.
In Figure  \ref{fig:rb}, we plot the test log-loss and compare our
method to BBVI and BayesLRule
with the re-parametrization trick (referred to as ``-rep'').
We also consider the VOGN method proposed for Gaussian approximations.
Note that BayesLRule using the re-parametrization trick does not stay in the constraint set and a line search is used.
From the plot, we can see  our method outperforms BayesLRule and performs comparably to VOGN.

Then, we consider the Gamma factor model \citep{knowles2015stochastic,khan2017conjugate} using Gamma approximations on the ``CyTOF'' dataset ($N=522,656, d=40$) with 300,000 chosen for training, where gradients are computed using
the implicit re-parametrization trick \citep{figurnov2018implicit} (referred to as ``-rep'').
We train the model with mini-batch size 39 and tune the step size for all methods.
In Figure  \ref{fig:rc}, we plot the test log-loss and compare our to BayesLRule and BBVI.
BayesLRule uses a line search since the updates using the re-parametrization trick  do not satisfy the constraint.
Our method outperforms BayesLRule and BBVI.

Finally, we consider a  Bayesian MLP network with 2 hidden layers, where we use 1000 units for each layer.  We
train the network with diagonal Gaussian approximations on
the ``CIFAR-10'' dataset ($N=60,000, d=3\times 32\times 32$) with
50,000 images for training and 10,000 images for validation.
We train the model with mini-batch size 128 and compare our Adam-like update (referred to as ``iBayesLRule-adam'') to VOGN.
We use the same initialization and hyper-parameters in both methods.
In Figure  \ref{fig:rd}, we plot the validation accuracy.
Our method performs similarly to VOGN.

\section{Discussion}
We present an improved learning rule to handle positive-definite parameterization constraints.
We propose a BCN parameterization so that natural gradients and the extra terms are easy to compute. Under this parameterization,  the Fisher matrix and the Christoffel symbols admit a closed-form via differentiation without introducing extra integrations.

Our main focus is  on the derivation of simple and efficient updates that naturally handle positive-definite constraints.
We give examples where our  updates have low iteration cost.
We hope to perform large-scale experiments in the future. 

\section*{Acknowledgements}
We would like to thank Hiroyuki Kasai (Waseda University) for useful discussions at ICML 2019 at the early stage of  this project.
WL is supported by a UBC International Doctoral Fellowship.

\bibliography{refs}
\bibliographystyle{icml2020}


\clearpage
\pagebreak
\onecolumn
\begin{appendices}
In the appendices, we will use the index notation and the Einstein summation notation introduced  in Section \ref{sec:ergd}.

\section{A counter-example for \citet{song2018accelerating}}
\label{app:case}
We show that the update suggested by \citet{song2018accelerating}  does not stay in the constraint set while ours  does.

Let's consider the following univariate Gaussian distribution under a BC parameterization $\vvarpar=\{\mu,\sigma \}$, where $\sigma$ denotes the standard deviation\footnote{ It is also used as an unconstrained parameterization of Gaussian distributions for BBVI.  Technically, this parameterization has a positivity constraint, which is often ignored in practice. In multivariate cases,  the Cholesky factor is used as an unconstrained parameterization, where the positivity constraint in the diagonal elements is often ignored. }. The constraint is $\Omega_1=\mathbb{R}$ and $\Omega_2=\mathbb{S}_{++}^{1}$.
$\ngrad^{(1)}$ and $\ngrad^{(2)}$ are  natural gradients for $\mu$ and $\sigma$, respectively.

\vspace{-0.6cm}
\begin{align*}
q(\lat|\vvarpar) = \exp\crl{ - \half  \left( \frac{\lat- \mu  }{ \sigma } \right)^2  -\half \log( 2\pi) -\log(\sigma) }
\end{align*}

Recall that the Christoffel symbols of the second kind can be computed as
$\Gamma^{c}_{\ a b} = \fim^{cd} \Gamma_{d,a b}$ where  $\Gamma_{d,a b}$
is the  Christoffel symbols of the first kind and $\fim^{cd}$ is the entry of the inverse the FIM, $\vfim^{-1}$, at position $(c,d)$.

Under this parameterization, the FIM  and the Christoffel symbols of the second kind are given below,  where the  Christoffel symbols of the first kind are computed by using Eq. \eqref{eq:chris}.
The computation of the Christoffel symbols can be difficult since the parameterization is not a BCN parameterization.

\vspace{-0.6cm}
\begin{align*}
\fim_{ab}  = \begin{bmatrix}
\frac{1}{\sigma^2} & 0 \\
           0 &  \frac{2}{\sigma^2}
            \end{bmatrix}, \,\,\,
\Gamma^{1}_{\ \ ab}  = \begin{bmatrix}
                0  & -\frac{1}{\sigma} \\
                    -\frac{1}{\sigma} & 0
                   \end{bmatrix}, \,\,\,
\Gamma^{2}_{\ \ ab} = \begin{bmatrix}
                    \frac{1}{2\sigma} & 0 \\
                    0& -\frac{1}{\sigma} 
                   \end{bmatrix}
\end{align*}

The update suggested by  \citet{song2018accelerating} is

\vspace{-0.6cm}
\begin{align*}
 \mu &\leftarrow \mu  - \stepsize \ngrad^{(1)} - \stepsize \ngrad^{(1)} - \frac{\stepsize \times \stepsize}{2} \Gamma^{1}_{\ \ ab} \ngrad^{(a)}\ngrad^{(b)} 
 = \mu - \stepsize \ngrad^{(1)} + \frac{\stepsize^2}{2} \left( \frac{2 \ngrad^{(1)} \ngrad^{(2)}}{\sigma} \right) \\
  \sigma  &\leftarrow \sigma - \stepsize \ngrad^{(2)} 
  - t \ngrad^{(2)} - \frac{\stepsize \times \stepsize}{2} \Gamma^{2}_{\ \ ab} \ngrad^{(a)}\ngrad^{(b)} 
 = \sigma - \stepsize \ngrad^{(2)} + \frac{\stepsize^2}{2} \left( \frac{ 2(\ngrad^{(2)})^2 - (\ngrad^{(1)})^2 }{2\sigma}   \right) 
\end{align*}

Clearly, the updated $\sigma$ does not always satisfy the positivity constraint $\mathbb{S}_{++}^{1}$.

As shown in Eq. \eqref{eq:ibrl_bc},
our rule can be used in not only a BCN parameterization but also a BC parameterization.
Since every block contains only a scalar, we use global indexes such as
$\varpar^{(i)}=\varpar^{a_i}$, $\ngrad^{(i)}=\ngrad^{[i]}$ and $\Gamma_{i,ii}=\Gamma_{a_i, b_ic_i}$ for notation simplicity.
Note that  $\Gamma^{1}_{\ \ 11}=0$ is the entry at the upper-left corner of $\Gamma^{1}_{\ \ ab}$
and $\Gamma^{2}_{\ \ 22}=-\frac{1}{\sigma}$ is the entry at the lower-right corner of $\Gamma^{2}_{\ \ ab}$.
In our update (see Eq. \eqref{eq:ibrl_bc}), we can see the update automatically satisfies the constraint as shown below.
\begin{align*}
 \overbrace{\mu}^{\varpar^{(1)} } &\leftarrow \overbrace{\mu}^{\varpar^{(1)} } 
 - \frac{\stepsize^2}{2} \Gamma^{1}_{\ \ 11} \ngrad^{(1)}\ngrad^{(1)}
 = \mu - \stepsize \ngrad^{(1)} \\
 \underbrace{\sigma}_{\varpar^{(2)} }  &\leftarrow \underbrace{\sigma }_{\varpar^{(2)} }
- \frac{\stepsize^2}{2} \Gamma^{2}_{\ \ 22} \ngrad^{(2)}\ngrad^{(2)} 
 = \sigma  - \stepsize \ngrad^{(2)} + \frac{\stepsize^2}{2} \left( \frac{ (\ngrad^{(2)})^2  }{\sigma}   \right)  
= \underbrace{\frac{1}{2\sigma}}_{>0}\Big[ \underbrace{ \sigma^2 }_{>0}+ \underbrace{\left( \sigma -  \stepsize \ngrad^{(2)}  \right)^2}_{\geq 0}   \Big]
\end{align*}

As we discuss at Section \ref{sec:irgd} of the main text, only the block-wise Christoffel symbol $\Gamma_{i,ii}$ for each block $i$ is required, which becomes essential for multivariate Gaussians and mixture of Gaussians. 

Let's consider another BC parameterization $\vvarpar=\{\mu,v \}$ for the Gaussian distribution, where $v=\sigma^2$ denotes the variance.
Note that we consider the parameterization for univariate Gaussian. For multivariate Gaussian, see Appendix \ref{app:special_case_cov}.
The underlying constraint is $\Omega =  \mathbb{R} \times \mathbb{S}_{++}^{1}$.
$\ngrad^{(1)}$ and $\ngrad^{(2)}$ are  natural gradients for $\mu$ and $v$, respectively.
\begin{align*}
q(\lat|\vvarpar) = \exp\crl{ - \half  \frac{(\lat-\mu)^2}{v}  -\half \log( 2\pi) -\half \log(v) }
\end{align*}
Under this parameterization, the FIM and the Christoffel symbols of the second kind are given below,  where the  Christoffel symbols of the first kind are computed by using Eq. \eqref{eq:chris}.
The computation of the Christoffel symbols can be difficult since the parameterization is not a BCN parameterization.
\begin{align*}
\fim_{ab}  = \begin{bmatrix}
\frac{1}{v}  & 0 \\
           0 &  \frac{1}{2 v^2} 
            \end{bmatrix}, \,\,\,
\Gamma^{1}_{\ \ ab}  = \begin{bmatrix}
                    0 & -\frac{1}{2 v} \\
                    -\frac{1}{2 v} & 0
                   \end{bmatrix}, \,\,\,
\Gamma^{2}_{\ \ ab} = \begin{bmatrix}
                    1 & 0 \\
                    0& -\frac{1}{v} 
                   \end{bmatrix}
\end{align*}
The update suggested by  \citet{song2018accelerating} is
\begin{align*}
 \mu  &\leftarrow \mu -\stepsize  \ngrad^{(1)} - \frac{\stepsize^2}{2} \Gamma^{1}_{\ \ ab} \ngrad^{(a)}\ngrad^{(b)} 
 = \mu  - \stepsize \ngrad^{(1)} + \frac{\stepsize^2}{2} \left( \frac{ \ngrad^{(1)}\ngrad^{(2)}}{ v  } \right) \\
v &\leftarrow v - \stepsize \ngrad^{(2)} - \frac{\stepsize^2}{2} \Gamma^{2}_{\ \ ab} \ngrad^{(a)}\ngrad^{(b)} 
 = v  - \stepsize \ngrad^{(2)} + \frac{\stepsize^2}{2} \left( \frac{ (\ngrad^{(2)})^2  }{v}  - (\ngrad^{(1)})^2   \right) 
\end{align*}

Obviously, the above updated $v$ does not always satisfy the positivity constraint.

Similarly, we use global indexes such as
$\varpar^{(i)}=\varpar^{a_i}$, $\ngrad^{(i)}=\ngrad^{[i]}$ and $\Gamma_{i,ii}=\Gamma_{a_i, b_ic_i}$ for notation simplicity
since every block contains only a scalar.
Note that  $\Gamma^{1}_{\ \ 11}=0$ is the entry at the upper-left corner of $\Gamma^{1}_{\ \ ab}$
and $\Gamma^{2}_{\ \ 22}=-\frac{1}{v}$ is the entry at the lower-right corner of $\Gamma^{2}_{\ \ ab}$.
In our update (see Eq. \eqref{eq:ibrl_bc}), we can see the update automatically satisfies the constraint as shown below.
\begin{align*}
\mu &\leftarrow \mu  - \stepsize \ngrad^{(1)} - \frac{\stepsize^2}{2} \Gamma^{1}_{\ \ 11} \ngrad^{(1)}\ngrad^{(1)}
 = \mu  - \stepsize \ngrad^{(1)} \\
v &\leftarrow v  - \stepsize \ngrad^{(2)} - \frac{\stepsize^2}{2} \Gamma^{2}_{\ \ 22} \ngrad^{(2)}\ngrad^{(2)} 
 = v  - \stepsize \ngrad^{(2)} + \frac{\stepsize^2}{2} \left( \frac{ (\ngrad^{(2)})^2  }{v}   \right)  
= \underbrace{\frac{1}{2 v}}_{>0}\Big[ \underbrace{v^2}_{>0} + \underbrace{ \left( v-  \stepsize \ngrad^{(2)}  \right)^2 }_{\geq 0} \Big] 
\end{align*}

\section{Riemannian Optimization}
\label{app:riem}

\subsection{Proof of Lemma \ref{lemma:block_identity} }
\label{app:block_riem_general}

Let's consider a parameterization $\vvarpar:=\{\vvarpar^{[1]},\dots,\vvarpar^{[m]}\}$ with $m$ blocks for a statistical manifold with metric $\vfim$.
We first define a BC parameterization $\vvarpar$ for a general metric $\vfim$. 
\begin{defn}
\label{assump:BC_general}
{\bf Block Coordinate Parameterization}:
   A parameterization is block coordinate (BC) if the metric $\vfim$ under this parameterization is block-diagonal according to the block structure of the parameterization.
\end{defn}

Recall that we use the following block notation: $\Gamma_{\ a_ib_i}^{c_i}\ngrad^{a_i}\ngrad^{b_i}:=\sum_{a \in [i]}\sum_{b \in [i]}\Gamma_{\, \, \, \, \, \, ab}^{(c_i)}\ngrad^{a}\ngrad^{b}$ where $[i]$ denotes the index set of block $i$, $(c_i)$ is the corresponding global index of $c_i$, and $a$ and $b$ are global indexes. 

Now, we prove Lemma \ref{lemma:block_identity}.
\begin{proof}
By the definition of a Riemannian gradient $\vngrad$, we have
\begin{align*}
\ngrad^{a_i} &= \sum_{b} \fim^{(a_i) b} g_{b}
= \sum_{b \in [i] }\fim^{(a_i) b} g_{b} + \sum_{b \not \in [i]} \underbrace{ \fim^{(a_i) b}}_{0} g_{b} 
=\sum_{b \in [i] }\fim^{(a_i) b} g_{b} =\fim^{a_i b_i}g_{b_i}, 
\end{align*} where in the second step,  $\fim^{(a_i) b}=0$ for any $b \not \in [i]$ (see \eqref{eq:bfim} for visualization) since the parameterization is BC,
and we use the definition of the block summation notation in the last step.

Similarly, we have
\begin{align*}
 \Gamma_{\ \ a_ib_i}^{c_i} &=\sum_{d} \fim^{(c_i)d}\Gamma_{d,(a_i)(b_i)} 
= \sum_{d \in [i]} \fim^{(c_i)d}\Gamma_{d,(a_i)(b_i)} + \sum_{d \not \in [i]} \underbrace{\fim^{(c_i)d}}_{0}\Gamma_{d,(a_i)(b_i)} 
= \sum_{d \in [i]} \fim^{(c_i)d}\Gamma_{d,(a_i)(b_i)} 
=\fim^{c_id_i}\Gamma_{d_i,a_ib_i}
\end{align*}
\end{proof}

\subsection{NGD is a First-order Approximation of $\vR(\stepsize)$}
\label{app:ngd_riem}
Now, we assume parameterization $\vvarpar=\{\vvarpar^{[1]},\dots, \vvarpar^{[m]}\}$ is a BC parameterization with $m$ blocks.
Recall that we define the curve $\vR(\stepsize)$ as $\vR(\stepsize):=\{\vR^{[1]}(\stepsize), \dots, \vR^{[m]}(\stepsize)\}$, where $\vR^{[i]}(\stepsize)$ is the solution of following ODE for block $i$.
\begin{align*}
&\dot{R}^{\ c_i}(0)  = - \fim^{c_i a_i} g_{a_i} \,\,;\,\,\,  
R^{\ c_i}(0)  =  \varpar^{c_i}   \\
& \ddot{R}^{\ c_i}(\stepsize)= -  \Gamma_{\ a_ib_i}^{c_i}(\stepsize) \dot{R}^{\ a_i}(\stepsize) \dot{R}^{\ b_i}(\stepsize)
\end{align*} where
$R^{\ c_i}(0)$,  $\dot{R}^{\ c_i}(0)$, $\ddot{R}^{\ c_i}(\stepsize)$ respectively denote the $c$-th entry of $\vR^{[i]}(0)$, $\dot{\vR}^{[i]}(0)$, and $\ddot{\vR}^{[i]}(\stepsize)$; 
$\Gamma_{\ a_i b_i}^{c_i}(\stepsize):= \Gamma_{\ a_i b_i}^{c_i} \bigr|_{  \varpar^{[i]}= R^{[i]}(\stepsize)}^{\varpar^{[-i]}=R^{[-i]}(0)  }$.

Recall that $\fim^{c_i a_i} $ is the entry of $(\vfim^{[i]})^{-1}$ at position $(c,a)$, where $\vfim^{[i]}$ is the $i$-th block of $\vfim$.
Note that $\vfim$ and $\vngrad$ are computed at $\vvarpar=\vR(0)$.
Since $\vvarpar$
is a  BC parameterization, 
by Lemma \ref{lemma:block_identity}, we have
$\fim^{c_i a_i} g_{a_i} = \ngrad^{c_i}$.

Therefore, when $\vfim$ is the FIM, the first-order approximation of $\vR(\stepsize)$ at $\stepsize_0=0$ is also a NGD update as shown below.
\begin{align*}
 \varpar^{c_i} &\leftarrow   R^{c_i}(\stepsize_0) +  \dot{R}^{c_i}(\stepsize_0)  (\stepsize-\stepsize_0) \\
&= \varpar^{c_i} - \stepsize \ngrad^{c_i}
\end{align*}


\section{Summary of Approximations Considered in This Work}
\label{app:all_example_table}

\begin{table}[ht]
\center
   \caption{ Summary of the Proposed Updates Induced by Our Rule in Various Approximations }
\begin{tabular}{l|l|l|l}
Approximation
   & Parameterization ($\vvarpar$)
   & Constraints 
   & Additional Term
    \\
    \hline
   Inverse Gaussian  (Appendix \ref{app:ig_case}) 
  &$\varpar^{(1)}=\beta^2$
   &$\varpar^{(1)} \in \mathbb{S}_{++}^1$
   &
   $\frac{\stepsize^2}{2} \left( \frac{3}{4\varpar^{(1)}}\right) \left( \ngrad^{(1)} \right)^2$
    \\
    ~
  & $\varpar^{(2)}=\alpha$
   &$\varpar^{(2)} \in \mathbb{S}_{++}^1$
   &
   $\frac{\stepsize^2}{2}\left( \frac{1}{\varpar^{(2)}}\right) \left( \ngrad^{(2)} \right)^2$
    \\
    \hline
    
  Gamma
  (Appendix \ref{app:gamma_case})
  &$\varpar^{(1)}=\alpha$
  &$\varpar^{(1)} \in \mathbb{S}_{++}^1$
  &$- \frac{\stepsize^2}{2}  \frac{ \hessop{\varpar^{(1)}} {\psi(\varpar^{(1)})} + \frac{1}{\left( \varpar^{(1)} \right)^2 } }{ 2 \left( \gradop{\varpar^{(1)}}{\psi(\varpar^{(1)})} - \frac{1}{\varpar^{(1)}} \right) } \left( \ngrad^{(1)} \right)^2$
  \\
  ~
  &$\varpar^{(2)}=\frac{\beta}{\alpha}$
  &$\varpar^{(2)} \in \mathbb{S}_{++}^1$
  &$\frac{\stepsize^2}{2}\left( \frac{1}{\varpar^{(2)}}\right) \left( \ngrad^{(2)} \right)^2$
  \\
  \hline
  
  Exponential
  (Appendix \ref{app:exp_case})
  &$\varpar^{(1)}=\varpar$
  &$\varpar^{(1)} \in \mathbb{S}_{++}^1$
  &$\frac{\stepsize^2}{2}\left( \frac{1}{\varpar^{(1)}}\right) \left( \ngrad^{(1)} \right)^2$
  \\
  \hline
  
  Multivariate Gaussian 
  (Appendix \ref{app:gauss_case})
  &$\vvarpar^{[1]}=\vmu$
  &$\vvarpar^{[1]} \in \mathbb{R}^d$
  &$\mathbf{0}$
  \\
  ~
  &$\vvarpar^{[2]}=\vSigma^{-1}$
  &$\vvarpar^{[2]} \in \mathbb{S}^{d \times d}_{++}$
  & $\frac{\stepsize^2}{2}\vngrad^{[2]} \left( \vvarpar^{[2]} \right)^{-1} \vngrad^{[2]}$
  \\
  \hline
 
   Mixture of Gaussians
   (Appendix \ref{app:mog_case}) 

  &$\{\vvarpar_c^{[1]}\}_{c=1}^{K}=\{\vmu_c\}_{c=1}^{K}$
  &$\vvarpar_c^{[1]} \in \mathbb{R}^d$ 
  &$\mathbf{0}$
  \\
  ~
  &$\{\vvarpar_c^{[2]}\}_{c=1}^{K}=\{\vSigma_c^{-1}\}_{c=1}^{K}$
  &$\vvarpar_c^{[2]} \in \mathbb{S}^{d \times d}_{++}$
  & $\frac{\stepsize^2}{2}\vngrad_c^{[2]} \left( \vvarpar_c^{[2]} \right)^{-1} \vngrad_c^{[2]}$
  \\
  
  ~
     &$\vvarpar_\mix=\{ \log(\pi_c/(1-\sum_{k=1}^{K-1}\pi_{k}) ) \}_{c=1}^{K-1}$
  &$\vvarpar_\mix \in \mathbb{R}^{K-1}$ 
  &$\mathbf{0}$\footnote{We do not compute the additional term in MOG since $\vvarpar_\mix \in \mathbb{R}^{K-1}$ is unconstrained.} 
  \\
   
  \hline
  Skew Gaussian 
   (Appendix \ref{app:skewg_case}) 
  &$\vvarpar^{[1]}=\begin{bmatrix} \vmu \\ \valpha \end{bmatrix}$
  &$\vvarpar^{[1]} \in \mathbb{R}^{2d}$ 
  &$\mathbf{0}$
  \\
  ~
  &$\vvarpar^{[2]}=\vSigma^{-1}$
  &$\vvarpar^{[2]} \in \mathbb{S}^{d \times d}_{++}$
  & $\frac{\stepsize^2}{2}\vngrad^{[2]} \left( \vvarpar^{[2]} \right)^{-1} \vngrad^{[2]}$
  \\
  \hline
\end{tabular}
\label{tab:examples}
\end{table}

Recall that we give Assumption 1-3 for exponential family distributions in Section \ref{sec:IBLR}.  We also 
extend  Assumption 1-3 to exponential family  mixtures as shown in Appendix \ref{app:cef_rgvi}.

In  Appendix \ref{app:ig_case}, \ref{app:gamma_case}, \ref{app:exp_case}, \ref{app:gauss_case},  \ref{app:mog_case}, \ref{app:skewg_case}, we show that Assumption 1-3 are satisfied and the additional term for each approximation is simplified.
In the corresponding appendix,  we also show how to compute natural gradients with the (implicit) reparameterization trick for each approximation listed in Table 
\ref{tab:examples}.

\section{Exponential Family (EF) Approximation}
\label{app:ef_rgvi}
\subsection{Christoffel Symbols}
\label{app:chris_general}
We first show how to simplify the Christoffel symbols of the first kind.
The FIM and the corresponding Christoffel symbols of the first kind are defined as follows.
\begin{align*}
 \fim_{ab} := -\Unmyexpect{q(\anyvar|\varpar)}\sqr{ \crossop{a}{b}{ \log q(\vanyvar|\vvarpar)} }; \,\,\,\,
  \Gamma_{d,a b}  := \half \sqr{ \gradop{a}{\fim_{bd}} +\gradop{b}{\fim_{ad}} - \gradop{d}{\fim_{ab}} }
\end{align*} where we denote $\gradop{a}{}=\gradop{\varpar^{a}}{}$ for notation simplicity.

Since
$\gradop{a}{\fim_{bd}} =-\Unmyexpect{q(\anyvar|\varpar)}\sqr{ \crossop{b }{d}{\log q(\vanyvar|\vvarpar)} \gradop{a }{ \log q(\vanyvar|\vvarpar)} }
-\Unmyexpect{q(\anyvar|\varpar)}\sqr{ \crossrdop{a}{b}{d}{ \log q(\vanyvar|\vvarpar)} }$, the Christoffel symbols of the first kind induced by the FIM can be computed as follows
, where $\vvarpar$ can be any parameterization.
\begin{align}
  \Gamma_{d,a b}
=   \half \Big[ & \Unmyexpect{q(\anyvar|\varpar)}\sqr{ \crossop{a}{b}{\log q(\vanyvar|\vvarpar)} \gradop{d} {\log q(\vanyvar|\vvarpar)} }    
-\Unmyexpect{q(\anyvar|\varpar)}\sqr{ \crossop{b }{d}{\log q(\vanyvar|\vvarpar)} \gradop{a }{ \log q(\vanyvar|\vvarpar)} } \nonumber \\
& -\Unmyexpect{q(\anyvar|\varpar)}\sqr{ \crossop{a }{d}{\log q(\vanyvar|\vvarpar)} \gradop{b }{ \log q(\vanyvar|\vvarpar)} } 
-\Unmyexpect{q(\anyvar|\varpar)}\sqr{ \crossrdop{a}{b}{d}{ \log q(\vanyvar|\vvarpar)} } \Big] \label{eq:chris}
\end{align} 
Note that Eq \ref{eq:chris} is also applied to a general distribution beyond exponential family.
However, the Christoffel symbol is not easy to compute due to  extra integrations in Eq \ref{eq:chris} and the FIM can be singular in general.
The Christoffel symbol could be easy to compute for an exponential family distribution under a BCN parameterization since we  compute the symbol via  differentiation without  the extra integrations. Moreover, the FIM is always positive-definite under a BCN parameterization.
Theorem \ref{thm:EF_chris} show this.

\subsection{Proof of Theorem \ref{thm:EF_chris}}
\label{app:proof_1}
In this case, $q(\vanyvar|\vvarpar)$ is an EF distribution. 
Since $\vvarpar$ is a BCN parameterization, 
given that $\vvarpar^{[-i]}$ is known, $q(\vanyvar|\vvarpar)$ is a one-parameter EF distribution as
\begin{align*}
 q(\vanyvar|\vvarpar) = h_i(\vanyvar,\vvarpar^{[-i]})\exp\sqr{ \myang{\vphi_i(\vanyvar,\vvarpar^{[-i]}), \vvarpar^{[i]}} - A(\vvarpar)} 
\end{align*}
Therefore, we have the following identities given $\vvarpar^{[-i]}$ is known.
\begin{align*}
 \crossop{a_i}{b_i}{\log q(\vanyvar|\vvarpar)} = -\crossop{a_i}{b_i}{A(\vvarpar)}; \, \,\,\,
 \Unmyexpect{q(\anyvar|\varpar)}\sqr{ \gradop{a_i} {\log q(\vanyvar|\vvarpar)}}  = 0 
\end{align*} where 
$\gradop{a_i}{}=\gradop{\varpar^{a_i}}{}$ for notation simplicity.

Using the above identities, we have
\begin{align*}
\Unmyexpect{q(\anyvar|\varpar)}\sqr{ \crossop{a_i}{b_i}{\log q(\vanyvar|\vvarpar)} \gradop{d_i} {\log q(\vanyvar|\vvarpar)} } 
= -\crossop{a_i}{b_i}{A(\vvarpar)} \underbrace{ \Unmyexpect{q(\anyvar|\varpar)}\sqr{ \gradop{d_i} {\log q(\vanyvar|\vvarpar)}} }_{0}
= 0
\end{align*}

Therefore, by Eq. \eqref{eq:chris}, $\Gamma_{d_i, a_ib_i}$ can be computed as follows
 \begin{align*}
   \Gamma_{d_i,a_i b_i}
= -\half \Unmyexpect{q(\anyvar|\varpar)}\sqr{ \crossrdop{a_i}{b_i}{d_i}{ \log q(\vanyvar|\vvarpar)} }
=  \half \crossrdop{a_i}{b_i}{d_i}{A(\vvarpar)}
 \end{align*}

Let $\vvarmean_{[i]}=\Unmyexpect{q(\anyvar|\varpar)}\sqr{ \vphi_i(\vanyvar) }$ denote the block coordinate expectation (BCE) parameter.
We have
\begin{align*}
0 =  \Unmyexpect{q(\anyvar|\varpar)}\sqr{ \gradop{a_i} {\log q(\vanyvar|\vvarpar)}} = \varmean_{a_i} -\gradop{a_i}{A(\vvarpar)}
\end{align*} where $\varmean_{a_i}$ denotes the $a$-th element of $\vvarmean_{[i]}$.

Therefore, we know that  $\varmean_{a_i} =\gradop{a_i}{A(\vvarpar)}$

Recall that the $i$-th block of $\vfim$ denoted by $\vfim^{[i]}$, can be computed as
\begin{align*}
\fim_{a_ib_i}  = - \Unmyexpect{q(\anyvar|\varpar)}\sqr{ \crossop{{b_i}}{{a_i}}{\log q(\vanyvar|\vvarpar)} } 
= \partial_{{b_i}}\partial_{{a_i}}{A(\vvarpar)} 
= \partial_{{b_i}} \sqr{ \partial_{{a_i}}{A(\vvarpar)} } = \gradop{\varpar^{b_i}} \varmean_{a_i} 
\end{align*} where $\partial_{{b_i}}=\partial_{\varpar^{b_i}}$ is for notation simplicity.

Recall that  $\vvarpar$ is a BC parameterization with $n$ blocks and  $\vfim$ is block diagonal as shown below.
\begin{align}
 \vfim =
\begin{bmatrix}
 \vfim^{[1]}  & \dots & \mathbf{0}\\
 \vdots &  \ddots & \vdots \\
 \mathbf{0}  & \dots & \vfim^{[n]}\\
 \end{bmatrix} \label{eq:bfim}
\end{align}

Recall that $\fim^{ab}$ denotes the element of $\vfim^{-1}$ with global index $(a,b)$ and 
$\fim^{a_ib_i}$ denotes the element of $\left(\vfim^{[i]}\right)^{-1}$ with local index $(a,b)$ in block $i$.

If $\vfim^{[i]}$  is positive definite everywhere, we have
\begin{align*}
\fim^{a_ib_i} &= \gradop{\varmean_{a_i}} \varpar^{b_i}
\end{align*}

Note that $\vfim^{[i]}$ is positive definite everywhere when $q(\vanyvar|\vvarpar^{[i]},\vvarpar^{[-i]})$ is a one-parameter minimal EF distribution given  $\vvarpar^{[-i]}$ is known
(See Theorem 1 of \citet{lin2019fast}).

By Lemma \ref{lemma:block_identity}, Riemannian gradient $\ngrad^{a_i}$ can be computed as
\begin{align*}
 \ngrad^{a_i} = \fim^{a_i b_i} g_{b_i}=\sqr{  \gradop{\varmean_{a_i}} \varpar^{b_i}} \sqr{ \gradop{\varpar^{b_i}}\elbofinal } = \gradop{\varmean_{a_i}} \elbofinal
\end{align*} where $g_{b_i}=\gradop{\varpar^{b_i}}\elbofinal$ is a Euclidean gradient.

\section{Example: Gaussian Approximation}
\label{app:gauss_case}
We consider the following parameterization $\vvarpar=\{ \vmu , \vS \}$, where $\vmu$ is the mean and $\vS$ is the precision. The open-set constraint is $\Omega_1 = \mathbb{R}^d$ and $\Omega_2 = \mathbb{S}^{d \times d}_{++}$.
Under this parameterization, the distribution can be expressed as below.
\begin{align*}
q(\lat|\vvarpar)=
\exp\Big( -\half \vlat^T\vS \vlat + \vlat^T\vS\vmu  - A(\vvarpar)\Big) 
\end{align*} where
$A(\vvarpar)=\half\big[ \vmu^T\vS\vmu - \log \left|\vS/(2\pi)\right| \big]$

\begin{lemma}
The Fisher information matrix under this parameterization is block diagonal with two blocks
\begin{align*}
\vfim = \begin{bmatrix}
\vfim_{\mu} &  \mathbf{0}  \\
\smash[b]{\underbrace{ \mathbf{0} }_{\vfim_{\mu S}} } & \vfim_{S} \\
        \end{bmatrix},\\
\end{align*} where $\vfim_{\mu S}= -\Unmyexpect{q(\lat)}\sqr{ \gradop{\mathrm{vec}(S)}\gradop{\mu}  \log q(\vlat|\vmu,\vS) }$ and
 $\vfim_{S}= -\Unmyexpect{q(\lat)}\sqr{  \gradop{\mathrm{vec}(S)}^2 \log q(\vlat|\vmu,\vS) }$.
 
Therefore, $\vvarpar=\{ \vmu , \vS \}$ is a BC parameterization.
\end{lemma}
\begin{proof}
We denote the $i$-th element of $\vmu$ using $\mu^i$. Similarly, we denote the element of $\vS$ at position $(j,k)$ using $S^{jk}$.
We prove this statement by showing cross terms in the Fisher information matrix denoted by $\vfim_{\mu S}$ are all zeros.
To show $\vfim_{\mu S}=\mathbf{0}$, it is equivalent to show
$- \Unmyexpect{q(\lat|\varpar)}\sqr{ \crossop{S^{jk}}{\mu^i} {\log q(\vlat|\vvarpar)} } = 0$ each $\mu^i$ and $S^{jk}$.

Notice that
$\Unmyexpect{q(\lat|\varpar)}\sqr{\vlat}=\vmu$.
We can obtain the above expression since
\begin{align*}
 \Unmyexpect{q(\lat|\varpar)}\sqr{ \crossop{S^{jk}}{\mu^i} {\log q(\vlat|\vvarpar)} }  &=
 \Unmyexpect{q(\lat|\varpar)}\sqr{ \gradop{S^{jk}} {\left( \vlat^T \vS  \ve_i - \ve_i^T \vS \vmu \right)} } \\
&= \Unmyexpect{q(\lat|\varpar)}\sqr{  \left( \vlat^T \vI_{jk} \ve_i - \ve_i^T \vI_{jk} \vmu \right) } \\
&= \Unmyexpect{q(\lat|\varpar)}\sqr{  \left( \ve_i^T \vI_{jk} \left( \vlat- \vmu   \right) \right) } \\
&= \ve_i^T \vI_{jk} \underbrace{\Unmyexpect{q(\lat|\varpar)}\sqr{  \vlat- \vmu   }}_{\mathbf{0}} = 0
\end{align*} where $\ve_i$ denotes an one-hot vector where all entries are zeros except the $i$-th entry with value 1, and $\vI_{jk}$ denotes an one-hot matrix where all entries are zeros except the entry at position $(j,k)$ with value 1.

The above expression also implies that
$\Unmyexpect{q(\lat|\varpar)}\sqr{\crossop{S}{\mu^i} {\log q(\vlat|\vvarpar)}  }=\mathbf{0}$.
\end{proof}

Now, we show that
$\vvarpar=\{ \vmu , \vSigma \}$ is also a BC parameterization.
Note that
\begin{align*}
- \Unmyexpect{q(\lat|\varpar)}\sqr{ \crossop{\Sigma^{jk}}{\mu^i} {\log q(\vlat|\vvarpar)} } 
=  -\Unmyexpect{q(\lat|\varpar)}\big[\mathrm{Tr} \big\{ (\gradop{\Sigma^{jk}}{ \vS } )  \crossop{S}{\mu^i} {\log q(\vlat|\vvarpar)}   \big\}   \big] = -\mathrm{Tr} \big\{ (\gradop{\Sigma^{jk}}{ \vS } )  \underbrace{\Unmyexpect{q(\lat|\varpar)}\big[\crossop{S}{\mu^i} {\log q(\vlat|\vvarpar)}    \big] }_{\mathbf{0}}  \big\} =  0.
\end{align*} 
Since $\vfim_{\mu \Sigma} =- \Unmyexpect{q(\lat|\varpar)}\sqr{ \crossop{\mathrm{vec}(\Sigma)}{\mu} {\log q(\vlat|\vvarpar)} }$ and 
$- \Unmyexpect{q(\lat|\varpar)}\sqr{ \crossop{\Sigma^{jk}}{\mu^i} {\log q(\vlat|\vvarpar)} }=0$ from above expression for any $i$, $j$, and $k$, we have
$\vfim_{\mu \Sigma}=\mathbf{0}$.
Therefore, 
$\vvarpar=\{ \vmu , \vSigma \}$ is also a BC parameterization since the cross terms of FIM under this new parameterization denoted by $\vfim_{\mu \Sigma}$ are zeros.

We denote the Christoffel symbols of the first kind and the second kind for $\vmu$ as ${\Gamma_{a_1,b_1 c_1}}$ and ${\Gamma^{a_1}_{\ \ \ \ b_1 c_1}}$, respectively.
\begin{lemma}
\label{lemma:gauss_gamma1}
All entries of ${\Gamma^{a_1}_{\ \ b_1 c_1}}$ are zeros.
\end{lemma}

\begin{proof}
We will prove this by showing that  all entries of  $\Gamma_{a_1 ,b_1 c_1}$  are zeros.
For notation simplicity, we  use ${\Gamma_{a,b c}}$ to denote $\Gamma_{a_1 ,b_1 c_1}$ in the  proof.
Let $\mu^a$ denote the $a$-th element of $\vmu$.
The following expression holds for any valid $a$, $b$, and $c$.
\begin{align*}
 \Gamma_{a,bc} = \half \Unmyexpect{q(\lat|\varpar)}\sqr{ \crossrdop{\mu^b}{\mu^c}{\mu^a}{A(\vvarpar)} } = 0
\end{align*}
We can obtain the above expression since
\begin{align*}
 \Unmyexpect{q(\lat|\varpar)}\sqr{ \crossrdop{\mu^b}{\mu^c}{\mu^a}{A(\vvarpar)} } =  \Unmyexpect{q(\lat|\varpar)}\sqr{ \crossop{\mu^b}{\mu^c}{\left(\ve_a^T \vS  \vmu \right)} } =  \Unmyexpect{q(\lat|\varpar)}\sqr{ \gradop{\mu^b}{\left(\ve_a^T \vS  \ve_c \right)} } = 0
\end{align*} where in the last step we use the fact that 
$\vS$, $\ve_a$, and $\ve_c$ do not depend on $\vmu$.
\end{proof}

Similarly, we denote the Christoffel symbols of the second kind for $\mathrm{vec}(\vS)$ as ${\Gamma^{a_2}_{\ \ \ \ b_2 c_2}}$.
Note that $\vS$ is now a matrix.
It is possible but tedious to directly compute the Christoffel symbol and  element-wisely validate the expression of the additional term for $\vS$.
Below, we give an alternative approach to identify the additional term for $\vS$ as shown in the proof of
Lemma \ref{lemma:gauss_s_chris}.

Recall that $\vR^{[2]}(\stepsize)$ is 
the solution of the following ODE for block $\mathrm{vec}(\vS)$:
\begin{align*}
&\dot{R}^{\ a_2}(0)  = -\ngrad^{a_2} \,\,;\,\,\,  
R^{\ a_2}(0)  =  S^{a_2}   \\
& \ddot{R}^{\ a_2}(\stepsize)= -  \Gamma_{\ \ \ b_2c_2}^{a_2}(\stepsize) \dot{R}^{\ b_2}(\stepsize) \dot{R}^{\ c_2}(\stepsize),
\end{align*} where $R^{\ a_2}(\stepsize)$ denotes the $a$-th element of $\vR^{[2]}(\stepsize)$
and $S^{a_2}$ denotes the $a$-th entry of $\mathrm{vec}(\vS)$.

\begin{lemma}
\label{lemma:gauss_s_chris}
The additional term for $\vS$ is $
\mathrm{Mat}( {\Gamma^{a_2}_{\ \ b_2 c_2}} {\ngrad^{b_2}} {\ngrad^{c_2}} )
=- \vngrad^{[2]}    \vS^{-1}  \vngrad^{[2]}   
$ where ${\ngrad^{a_2}}$ denotes the $a$-th element of $\mathrm{vec}(\vngrad^{[2]})$.
\end{lemma}

\begin{proof}
As discussed in Sec \ref{sec:deriv}, $\mathbf{R}^{[i]}(\stepsize)$ is a (block coordinate) geodesic given $\vvarpar^{[-i]}$ is known.
In this case,
given that $\vmu$ is known, $\mathbf{R}^{[2]}(\stepsize)$  has the following closed-form expression \citep{pennec2006riemannian,fletcher2004principal,minh2017covariances}.
\begin{align*}
\mathrm{Mat}(\vR^{[2]}(\stepsize)) &=  \vU \mathrm{Exp}( \stepsize \vU^{-1}  \vngrad^{[2]}  \vU^{-1} ) \vU  
\end{align*} where $\vU= \vS^{\half}$ denotes the matrix square root and $\mathrm{Exp}(\vX):=\vI+\sum_{n=1}^{\infty} \frac{\vX^{n}}{n!} $ denotes the matrix exponential function.\footnote{The function is well-defined since the matrix series is absolutely convergent element-wisely.}

The additional term for $\vS$ can be obtained as follows.
\begin{align*}
- \mathrm{Mat}( \Gamma^{a_2}_{\ \ \ \ b_2 c_2} \ngrad^{b_2} \ngrad^{c_2} )
&= \mathrm{Mat}(\ddot{\vR}^{[2]}(0)) \\
&= \mathrm{Mat}(\nabla_{\stepsize}^2{\vR^{[2]}(\stepsize)}\big|_{\stepsize=0}) \\
&= \nabla_{\stepsize}^2{\mathrm{Mat}(\vR^{[2]}(\stepsize))}\big|_{\stepsize=0} \\
&=  \nabla_{\stepsize}^2 {\left( \vU \mathrm{Exp}( \vU^{-1} \stepsize \vngrad^{[2]}  \vU^{-1} ) \vU \right)} \big|_{\stepsize=0}\\
&= \vU  \nabla_{\stepsize}^2 {\left( \mathrm{Exp}( \vU^{-1} \stepsize \vngrad^{[2]}   \vU^{-1} )  \right)} \big|_{\stepsize=0}\vU  \\
&= \vU  ( \vU^{-1} \vngrad^{[2]}  \vU^{-1} ) ( \vU^{-1} \vngrad^{[2]}  \vU^{-1} ) \vU  \\
&= \vU  ( \vU^{-1} \vngrad^{[2]}  \vS^{-1} \vngrad^{[2]} \vU^{-1} ) \vU  \\
&= \vngrad^{[2]}    \vS^{-1}  \vngrad^{[2]}   
\end{align*} where we use the following expression to move from step 5 to step 6.
\begin{align*}
\nabla_{\stepsize}^2 {\mathrm{Exp}(\stepsize\vX)}\big|_{\stepsize=0} &= \nabla_{\stepsize}^2 {\left(\vI+\sum_{n=1}^{\infty} \frac{\left(\stepsize\vX\right)^{n}}{n!}\right)} \big|_{\stepsize=0} = \vX^2
\end{align*}
\end{proof}

Finally, by Lemma 
\ref{lemma:gauss_gamma1} and \ref{lemma:gauss_s_chris}, the update induced by the proposed rule is
\begin{align*}
 \mu^c & \leftarrow \mu^c - \stepsize \ngrad^{c_1}  - \frac{ \stepsize \times \stepsize}{2} \overbrace{\Gamma_{\ a_1b_1}^{c_1}}^{0}\ngrad^{a_1}\ngrad^{b_1} \\ 
s^c & \leftarrow s^c- \stepsize \ngrad^{c_2}  - \frac{ \stepsize \times \stepsize}{2}  \Gamma_{\ a_2b_2}^{c_2} \ngrad^{a_2}\ngrad^{b_2} 
\end{align*} where $s^c$ is the $c$-th element of $\mathrm{vec}(\vS)$.

Therefore, we have
\begin{align*}
\vmu &\leftarrow \overbrace{\vmu}^{\mathrm{vec}(\mu^c) } - \stepsize \overbrace{ \vngrad^{[1]} }^{\mathrm{vec}( \ngrad^{c_1} ) }  \\
\vS &\leftarrow \underbrace{ \vS}_{ \mathrm{Mat}(s^c) } - \stepsize \underbrace{\vngrad^{[2]}}_{ \mathrm{Mat}(\ngrad^{c_2}) } + \frac{ \stepsize \times \stepsize}{2}  \underbrace{  \vngrad^{[2]}   \vS^{-1}  \vngrad^{[2]}}_{ - \mathrm{Mat}(  \Gamma_{\ a_2b_2}^{c_2} \ngrad^{a_2}\ngrad^{b_2})  }
\end{align*}

\subsection{Proof of Theorem  \ref{thm:gauss_update}}
\label{app:gauss_proof}

Now, we give a proof of Theorem  \ref{thm:gauss_update}.
\begin{proof}
   First note that $\hat{\vG}=\vS-\Unmyexpect{q}{\sqr{ \nabla_{\lat}^2 \bar{\ell}(\vlat) } }$ is a symmetric matrix. Let $\vL$ be the Cholesky of the current $\vS = \vL \vL^T$.  We can simplify the right hand side of \eqref{eq:ivon_2} as follows:
\begin{align*}
  & (1-\stepsize) \vS  + \stepsize \Unmyexpect{q}{\sqr{ \nabla_{\lat}^2 \bar{\ell}(\vlat) } }  + { \frac{\stepsize^2}{2} \hat{\vG} \vS^{-1} \hat{\vG} } 
 =  \vS - \stepsize \hat{\vG} + \frac{\stepsize^2}{2}\hat{\vG} \vS^{-1} \hat{\vG}  
=  \half \left(  \vS + \left( \vL- \stepsize \hat{\vG} \vL^{-T} \right) \left( \vL^T- \stepsize \vL^{-1} \hat{\vG} \right)  \right) 
 =   \half \Big( \vS + \vU^T \vU \Big) ,
\end{align*}
where  $\vU:=\vL^T- \stepsize \vL^{-1} \hat{\vG}$.
Since the current $\vS$ is positive-definite, and $ \vU^T \vU $ is positive semi-definite, we know that the update for $\vS$ is positive-definite.
\end{proof}

\subsection{Natural Gradients and  the Reparameterization Trick}
\label{app:gauss_ng}
Since $\vvarpar=\{\vmu,\vS\}$ is a BCN parameterization of a exponential family distribution, gradients w.r.t. BC expectation parameters are natural gradients for BC natural parameters as shown in  Theorem \ref{thm:EF_chris}.

Given that $\vS$ is known, the BC expectation parameter is $\vm_{[1]}=\Unmyexpect{q(\lat)}\sqr{ \vS \vlat}= \vS \vmu$.
In this case, we know that $\gradop{ \mu }{\elbofinal}= \vS \gradop{m_{[1]}}{\elbofinal}$.
Therefore, the natural gradient w.r.t. $\vmu$ is $\vngrad^{[1]}= \gradop{ m_{[1]}}{\elbofinal}= \vS^{-1}\gradop{\mu}{\elbofinal}= \vSigma \gradop{\mu}{\elbofinal}$.

Likewise, 
 given that $\vmu$ is known, the BC expectation parameter is $\vm_{[2]}=\Unmyexpect{q(\lat)}\sqr{ -\half \vlat \vlat^T +  \vmu  \vlat^T }=\half \left( \vmu \vmu^T -  \vS^{-1} \right)$.
Therefore,  the natural gradient w.r.t. $\vS$ is $\vngrad^{[2]}=\gradop{m_{[2]}}{\elbofinal}=-2\gradop{S^{-1} }{\elbofinal}=-2 \gradop{\Sigma}{\elbofinal}$.

Recall that $\elbofinal(\vvarpar)=\Unmyexpect{q(\lat|\varpar)}\sqr{ \ell(\data,\vlat) -\log p(\vlat) + \log q(\vlat|\vvarpar) }$, by the Gaussian identities \citep{opper2009variational,sarkka2013bayesian} (see \citet{wu-report} for a derivation of these identities), we have
\begin{align}
\gradop{\mu} \elbofinal(\vvarpar) &= \gradop{\mu} \sqr{ \Unmyexpect{q(\lat|\varpar)}\sqr{ \ \ell(\data,\vlat) -\log p(\vlat) } - \half \log| 2\pi e \vSigma| }\nonumber \\
&= \gradop{\mu} \sqr{ \Unmyexpect{q(\lat|\varpar)}\sqr{ \ \ell(\data,\vlat) -\log p(\vlat) }  } \nonumber \\
&= \Unmyexpect{q(\lat|\varpar)}\sqr{ \nabla_\lat \sqr{ \ell(\data,\vlat) -\log p(\vlat) } }  \label{eq:gauss_mean_rep} \\
\gradop{\Sigma} \elbofinal(\vvarpar) &= \gradop{\Sigma} \sqr{ \Unmyexpect{q(\lat|\varpar)}\sqr{ \ \ell(\data,\vlat) -\log p(\vlat) } - \half \log| 2\pi e \vSigma| } \nonumber \\
&= \gradop{\Sigma} \sqr{ \Unmyexpect{q(\lat|\varpar)}\sqr{ \ \ell(\data,\vlat) -\log p(\vlat) }  } -\half \vSigma^{-1} \nonumber \\
&= \half  \Unmyexpect{q(\lat|\varpar)}\sqr{ \vSigma^{-1} (\vlat-\vmu) \nabla_\lat^T \sqr{ \ell(\data,\vlat) -\log p(\vlat) }  } -\half \vSigma^{-1} \label{eq:gauss_cov_rep} \\
 &= \half \Unmyexpect{q(\lat|\varpar)}\sqr{ \nabla_\lat^2 \sqr{ \ell(\data,\vlat) -\log p(\vlat) } } -\half \vSigma^{-1} \label{eq:gauss_cov_hess}
\end{align} where \eqref{eq:gauss_mean_rep} is also known as the reparameterization trick for the mean,
\eqref{eq:gauss_cov_rep} is also known as the reparameterization trick for the covariance, and we call \eqref{eq:gauss_cov_hess}  the Hessian trick.

Using Monte Carlo approximation, we have
\begin{align*}
 \gradop{\mu} \elbofinal &\approx  \nabla_\lat \sqr{ \ell(\data,\vlat) -\log p(\vlat) } \\
 \gradop{\Sigma} \elbofinal &\approx \frac{1}{4} \sqr{ \bar{\vS} + \bar{\vS}^T }  -\half \vSigma^{-1}  & \text{ referred to as ``-rep'' } \\
 \gradop{\Sigma} \elbofinal &\approx \half \sqr{\nabla_\lat^2 \sqr{ \ell(\data,\vlat) -\log p(\vlat) } } -\half \vSigma^{-1}  & \text{ referred to as ``-hess'' }
\end{align*} where $\bar{\vS}:=\vSigma^{-1} (\vlat-\vmu) \nabla_\lat^T \sqr{ \ell(\data,\vlat) -\log p(\vlat) }$
and $\vlat \sim q(\vlat|\vvarpar)=\gauss(\vlat|\vmu,\vSigma)$.

\subsection{Adam-like Update}
\label{app:adam_diag_gauss}
We consider to solve the following  problem, 
where
we use a diagonal Gaussian approximation
$q(\vlat|\vmu,\vs)=\gauss(\vlat|\vmu,\vs)$ and $\vs=\vsigma^{-2}$.
\begin{align*}
\min_{\mu, s} \elbofinal(\vmu, \vs)=  \textrm{E}_{q(\lat|\mu,s)}\sqr{\left(\sum_{i=1}^{N}\ell_i(\vlat)\right) - \log \gauss(\vlat|\mathbf{0}, \lambda^{-1}\vI) +\log q(\vlat|\vmu,\vs) } 
\end{align*}

Note that
\begin{align*}
\gradop{\mu}{\elbofinal(\vmu,\vs)} &:=  \sum_{i=1}^{N} \gradop{\mu}{\Unmyexpect{q(\lat|\mu,s)}\sqr{ \ell_i(\vlat)}} + \lambda \vmu \\
\gradop{\sigma^2}{\elbofinal(\vmu,\vs)} &:=  \sum_{i=1}^{N} \gradop{\sigma^{2}}{\Unmyexpect{q(\lat|\mu,s)}\sqr{ \ell_i(\vlat)}} +\half \lambda - \half \vs 
\end{align*} where $\gradop{\mu}{\Unmyexpect{q(\lat|\mu,s)}\sqr{\ell_i(\vlat)}}$ and $\gradop{\sigma^{2}}{\Unmyexpect{q(\lat|\mu,s)}\sqr{ \ell_i(\vlat)}}$ can be computed by the reparameterization trick with MC approximations where 
$\vlat \sim \gauss(\vlat|\vmu,\vs)$.
\begin{align*}
 \gradop{\mu}{\Unmyexpect{q(\lat|\mu,s)}\sqr{  \ell_i(\vlat)} } = \Unmyexpect{q(\lat|\mu,s)}\sqr{\nabla_{\lat}{\ell_i(\vlat)}}  & \approx \nabla_{\lat}{\ell_i(\vlat)} \\
 \gradop{\sigma^{2}}{\Unmyexpect{q(\lat|\mu,s)}\sqr{ \ell_i(\vlat)}} =  \half \Unmyexpect{q(\lat|\mu,s)}\sqr{ \vs \odot \left( \vlat - \vmu \right) \odot \nabla_{\lat}{\ell_i(\vlat)}} & \approx \half\sqr{ \vs \odot \left( \vlat - \vmu \right)} \odot \nabla_{\lat}{\ell_i(\vlat)}
\end{align*}

The natural gradients can be computed as follows.
\begin{align*}
\vngrad_k^{[1]} &=  \vsigma_k^2\left( \gradop{\mu} {\elbofinal(\vmu,\vs)}\big|_{\mu=\mu_k,s=s_k}\right)\\
\vngrad_k^{[2]} &= -2 \gradop{\sigma^2} {\elbofinal(\vmu,\vs)}\big|_{\mu=\mu_k,s=s_k}
\end{align*}

The  update induced by our rule with exponential decaying step-sizes and the natural momentum \citep{khan18a} shown in blue  is given as follows.
\begin{align*}
 \vmu_{k+1} & = \vmu_k - \stepsize_1  \vngrad_k^{[1]} + \stepsize_2 { \color{blue} \vsigma_k^{2} \odot \vsigma^{-2}_{k-1} \odot \left( \vmu_k -\vmu_{k-1} \right) }\\
\vsigma_{k+1}^{-2} & = \vsigma_k^{-2} - \stepsize_3 \vngrad_k^{[2]} + \frac{\stepsize_3^2}{2} \vngrad_k^{[2]} \odot \vsigma_k^2 \odot \vngrad_k^{[2]}  
\end{align*} where $\stepsize_1=\stepsize (1-r_1) \frac{ 1-r_2^k}{ 1-r_1^k}$, $\stepsize_2=r_1 \frac{ 1-r_2^k}{ 1-r_1^k} \frac{ 1-r_1^{k-1}}{ 1-r_2^{k-1}}$,
and $\stepsize_3=(1-r_2)$.

Recall that $\vs=\vsigma^{-2}$. The proposed update can be expressed as
\begin{align*}
 \vmu_{k+1} & = \vmu_k - \stepsize (1-r_1) \frac{ 1-r_2^k}{ 1-r_1^k}  \hat{\vs}_k^{-1} \odot \vg_k + r_1 \frac{ 1-r_2^k}{ 1-r_1^k} \frac{ 1-r_1^{k-1}}{ 1-r_2^{k-1}} \hat{\vs}_k^{-1} \odot \hat{\vs}_{k-1} \odot \left( \vmu_k -\vmu_{k-1} \right) \\
\hat{\vs}_{k+1} & = \hat{\vs}_k + (1-r_2) \vh_k + \frac{(1-r_2)^2}{2}\vh_k \odot \hat{\vs}_k^{-1} \odot \vh_k  \\
\vs_{k+1} &= N \hat{\vs}_{k+1}
\end{align*} where $\vg_k:=\frac{1}{N}\sum_{i=1}^{N} \gradop{\mu}{\Unmyexpect{q(\lat|\mu,s)}\sqr{  \ell_i(\vlat)}}\big|_{\mu=\mu_k,s=s_k} + \frac{\lambda}{N} \vmu_k$ and 
$\vh_k:= \frac{2}{N}\sum_{i=1}^{N} \gradop{\sigma^{2}}{\Unmyexpect{q(\lat|\mu,s)}\sqr{ \ell_i(\vlat)}}\big|_{\mu=\mu_k,s=s_k}  +\frac{\lambda}{N} - \hat{\vs}_k $.

Let's define $\vm_k:=\frac{ 1-r_1^{k-1}}{ \stepsize(1-r_2^{k-1})}  \hat{\vs}_{k-1}\odot \left( \vmu_{k-1} -\vmu_{k} \right)$.
We can further simplify the above update as shown below.
\begin{align*}
 \vmu_{k+1} & = \vmu_k - \stepsize (1-r_1) \frac{ 1-r_2^k}{ 1-r_1^k} \hat{\vs}_k^{-1} \odot \vg_k + \stepsize r_1 \frac{ 1-r_2^k}{ 1-r_1^k}    \hat{\vs}_k^{-1} \odot \left( \frac{ 1-r_1^{k-1}}{ \stepsize (1-r_2^{k-1})} \hat{\vs}_{k-1}\odot \left( \vmu_k -\vmu_{k-1} \right) \right) \\
 &= \vmu_k - \stepsize\frac{ 1-r_2^k}{ 1-r_1^k}  \hat{\vs}_k^{-1} \odot \sqr{(1-r_1)\vg_k + r_1 \vm_k } \\
 \vm_{k+1} &=\frac{ 1-r_1^{k}}{ \stepsize(1-r_2^{k})}  \hat{\vs}_{k}\odot \left( \vmu_{k} -\vmu_{k+1} \right) \\
 &=\frac{ 1-r_1^{k}}{ \stepsize(1-r_2^{k})}\stepsize\frac{ 1-r_2^k}{ 1-r_1^k}  \sqr{(1-r_1)\vg_k + r_1 \vm_k }  \\
 &= (1-r_1)\vg_k + r_1 \vm_k \\
\hat{\vs}_{k+1} & = \hat{\vs}_k + (1-r_2) \vh_k + \frac{(1-r_2)^2}{2}\vh_k \odot \hat{\vs}_k^{-1} \odot \vh_k  \\
 & = \half \sqr{ \hat{\vs}_k + ( \hat{\vs}_k + (1-r_2) \vh_k) \odot \hat{\vs}_k^{-1} \odot ( \hat{\vs}_k + (1-r_2) \vh_k) } \\
\vs_{k+1} &= N \hat{\vs}_{k+1}
\end{align*}
where $\vlat \sim q(\vlat|\vmu_k,\vs_k)$, 
$ \vg_k \approx \nabla_{\lat}{\ell_i(\vlat)} + \frac{\lambda}{N} \vmu_k$, and
$ \vh_k \approx  \sqr{ (N\hat{\vs}_k) \odot \left( \vlat - \vmu \right)} \odot \nabla_{\lat}{\ell_i(\vlat)} +\frac{\lambda}{N} - \hat{\vs}_k $.

\subsection{  \citet{tran2019variational} is a special case of our update}
\label{app:special_case_cov}
In the Gaussian case,  \citet{tran2019variational} consider the following update by using  parameterization $\vvarpar=\{\vmu, \vSigma\}$, where $\vSigma$ is the covariance matrix.
  \begin{align}
  \vmu & \leftarrow \vmu - \stepsize  \vSigma ( \gradop{\mu}\elbofinal ) \\
  \vSigma & \leftarrow \vSigma - \stepsize  \vngrad^{[2]} +\frac{\stepsize \times \stepsize}{2} \vngrad^{[2]}    \vSigma^{-1}  \vngrad^{[2]}  
  = Ret(\vSigma, -\stepsize \vngrad^{[2]}) \label{eq:special_gauss}.
  \end{align} where the natural gradient\footnote{\label{ft:mis}There is a typo in Algorithm 2 of \citet{tran2019variational}. The natural gradient for $\vSigma$ should be  $2 \vSigma ( \gradop{\Sigma}\elbofinal )  \vSigma$ instead of $\vSigma ( \gradop{\Sigma}\elbofinal ) \vSigma$. Note that the Riemannian gradient for a positive-definite matrix is  $\vSigma ( \gradop{\Sigma}\elbofinal ) \vSigma$  (see Table 1 of \citet{hosseini2015matrix}) while the natural/Riemannian gradient for Gaussian distribution w.r.t. $\vSigma$ is $2 \vSigma ( \gradop{\Sigma}\elbofinal )  \vSigma$. } for $\vSigma$ is  $\vngrad^{[2]}:=2 \vSigma ( \gradop{\Sigma}\elbofinal )  \vSigma$ and 
 the retraction map is $Ret(\vSigma, \vb):= \vSigma + \vb + \half \vb \vSigma^{-1} \vb$.

However, \citet{tran2019variational} do not justify the use of the retraction map, which is just one of retraction maps developed for positive definite matrices. In this section, we show that how to derive this update from our rule.

As shown in Eq. \eqref{eq:ibrl_bc},
our rule can be used under not only a BCN parameterization but also a BC parameterization.
Now, we show that our rule  can recover the above update using the parameterization $\vvarpar=\{ \vmu ,\vSigma\}$.
Recall that this parameterization is a BC parameterization.
It only requires us to show that natural gradients and the additional terms are described in Eq. \eqref{eq:special_gauss}.

Given that  $\vSigma$ is known, $\vmu$ is the natural parameter and the expectation parameter is $\vm_{[1]}=\Unmyexpect{q(\lat)}\sqr{ \vSigma^{-1} \vlat}= \vSigma^{-1} \vmu$
as shown in Appendix  \ref{app:gauss_ng}.
Therefore, the natural gradient w.r.t. $\vmu$ is $\vngrad^{[1]}= \gradop{m_{[1]}}{\elbofinal} =  \vSigma \gradop{\mu}{\elbofinal}$.
  
Now, we show that the natural gradients w.r.t. $\vSigma$ is 
\begin{align*}
\vngrad^{[2]} &= 2 \vSigma ( \gradop{\Sigma}\elbofinal )  \vSigma  
\end{align*}

A proof using matrix calculus is provided below. See \citet{malago2015information} for alternative proofs.
By matrix calculus, we have
\begin{align*}
&-\Unmyexpect{q(\lat)}\sqr{ \gradop{\Sigma^{ij}} \gradop{\Sigma} \sqr{ \log q(\vlat|\vmu,\vSigma) }  } \\
=& \Unmyexpect{q(\lat)}\sqr{ \gradop{\Sigma^{ij}} \gradop{\Sigma} \sqr{ \half (\vlat-\vmu)^T \vSigma^{-1} (\vlat-\vmu) + \half \log|\vSigma/(2\pi)| }  } \\
=& \half\Unmyexpect{q(\lat)}\sqr{ \gradop{\Sigma^{ij}}  \sqr{ -  \vSigma^{-1} (\vlat-\vmu)(\vlat-\vmu)^T\vSigma^{-1} + \vSigma^{-1} }  } \\
=& \half\Unmyexpect{q(\lat)}\sqr{  - \gradop{\Sigma^{ij}} \sqr{\vSigma^{-1} } (\vlat-\vmu)(\vlat-\vmu)^T\vSigma^{-1} - \vSigma^{-1}  (\vlat-\vmu)(\vlat-\vmu)^T \gradop{\Sigma^{ij}} \sqr{\vSigma^{-1} }+ \gradop{\Sigma^{ij}} \sqr{ \vSigma^{-1}}  }   \\
=& \half\Unmyexpect{q(\lat)}\sqr{  - \gradop{\Sigma^{ij}} \sqr{\vSigma^{-1} } (\vlat-\vmu)(\vlat-\vmu)^T\vSigma^{-1} - \vSigma^{-1}  (\vlat-\vmu)(\vlat-\vmu)^T \gradop{\Sigma^{ij}} \sqr{\vSigma^{-1} }+ \gradop{\Sigma^{ij}} \sqr{ \vSigma^{-1}}  }   \\
=& -  \half \gradop{\Sigma^{ij}} \sqr{\vSigma^{-1} } \underbrace{\Unmyexpect{q(\lat)}\sqr{  (\vlat-\vmu)(\vlat-\vmu)^T }}_{\vSigma} \vSigma^{-1} - \half \vSigma^{-1} \underbrace{ \Unmyexpect{q(\lat)}\sqr{  (\vlat-\vmu)(\vlat-\vmu)^T }}_{\vSigma} \gradop{\Sigma^{ij}} \sqr{\vSigma^{-1} }   + \half \gradop{\Sigma^{ij}} \sqr{ \vSigma^{-1}}  \\
=& \half \sqr{  - \gradop{\Sigma^{ij}} \sqr{\vSigma^{-1} } \vI  - \vI \gradop{\Sigma^{ij}} \sqr{\vSigma^{-1} }+ \gradop{\Sigma^{ij}} \sqr{ \vSigma^{-1}}    } \\
=&  -\half \gradop{\Sigma^{ij}} \sqr{ \vSigma^{-1}}  
\end{align*}
Therefore, the block matrix of the FIM related to $\vSigma$ is $\vfim_{\Sigma}:= -\Unmyexpect{q(\lat)}\sqr{  \gradop{\mathrm{vec}(\Sigma)}^2 \sqr{ \log q(\vlat|\vmu,\vSigma) }  }= -\half \gradop{\mathrm{vec}(\Sigma)} \sqr{ \mathrm{vec}(\vSigma^{-1}) }$ due to the above expression.
Note that $\vfim_{\Sigma}^{-1}=-2 \gradop{\mathrm{vec}(\Sigma^{-1})} \sqr{ \mathrm{vec}(\vSigma) }$.

Note that $\vngrad^{[2]}$ is the natural gradient for $\vSigma$.
Since $\vvarpar=\{\vmu,\vSigma\}$ is a BC parameterization, by Lemma \ref{lemma:block_identity},
the natural gradient w.r.t. $\mathrm{vec}(\vSigma)$ is
\begin{align*}
\mathrm{vec}(\vngrad^{[2]}) &:= \vfim_{\Sigma}^{-1} \mathrm{vec}( \gradop{\Sigma}\elbofinal ) \\
&= -2 \gradop{\mathrm{vec}(\Sigma^{-1})} \sqr{ \mathrm{vec}(\vSigma) }\mathrm{vec}( \gradop{\Sigma}\elbofinal ) \\
&= -2 \gradop{\mathrm{vec}(\Sigma^{-1})} \sqr{ \mathrm{vec}(\vSigma) } \gradop{\mathrm{vec}(\Sigma)}\elbofinal  \\
&= -2  \gradop{\mathrm{vec}(\Sigma^{-1})}\elbofinal  \\
&= -2  \mathrm{vec}(\gradop{\Sigma^{-1}}\elbofinal ) 
\end{align*} where we obtain the fourth step using the chain rule.

Therefore, we have $\vngrad^{[2]} = -2 \gradop{\Sigma^{-1}}\elbofinal$.
By matrix calculus, we have
\begin{align*}
 \gradop{\Sigma^{-1}}\elbofinal = - \vSigma  ( \gradop{\Sigma}\elbofinal ) \vSigma
\end{align*}
Finally,  we have
\begin{align*}
\vngrad^{[2]} = 2  \vSigma  ( \gradop{\Sigma}\elbofinal ) \vSigma
\end{align*}

Now, we show that the additional term for $\vmu$ is $\mathbf{0}$ under parameterization $\vvarpar=\{\vmu, \vSigma\}$.
Since $\vvarpar$ is a BC parameterization, by Lemma \ref{lemma:gauss_gamma1}, all entries of ${\Gamma^{a_1}_{\ \ b_1 c_1}}$ for $\vmu$ are zeros.
Therefore, the additional term for $\vmu$ is $\mathbf{0}$.

We denote the Christoffel symbol of the second kind for $\mathrm{vec}(\vSigma)$ as ${\Gamma^{a_2}_{\ \ \ \ b_2 c_2}}$.
Now, we show that the additional term for $\vSigma$ is $\frac{t \times t}{2} \vngrad^{[2]}  \vSigma^{-1}  \vngrad^{[2]}$.
It is equivalent to show
$ \mathrm{Mat}( {\Gamma^{a_2}_{\ \ b_2 c_2}} {\ngrad^{b_2}} {\ngrad^{c_2}} ) = - \vngrad^{[2]}  \vSigma^{-1}  \vngrad^{[2]}$.

Recall that the natural gradient for $\vS=\vSigma^{-1}$ is $\vG=-2 \gradop{\Sigma}\elbofinal$.
Under parameterization $\bar{\vvarpar}=\{\vmu,\vS \}$, 
$\bar{\vR}^{[2]}(\stepsize)$  has the following closed-form expression, which is used in the proof of  Lemma \ref{lemma:gauss_s_chris}.
\begin{align*}
\mathrm{Mat}(\bar{\vR}^{[2]}(\stepsize)) &=  \vU \mathrm{Exp}( \stepsize \vU^{-1}  \vG  \vU^{-1} ) \vU  
\end{align*} where $\vU= \vS^{\half}$ and $\mathrm{Exp}(\vX):=\vI+\sum_{n=1}^{\infty} \frac{\vX^{n}}{n!} $.

Note that $\vSigma=\vS^{-1}$.
Therefore, under parameterization $\vvarpar=\{\vmu,\vSigma \}$, we have
\begin{align*}
\overbrace{\mathrm{Mat}(\mathbf{R}^{[2]}(\stepsize))}^{\vSigma_{\text{new}}} & =   \big[ \,\, \overbrace{\mathrm{Mat}(\bar{\vR}^{[2]}(\stepsize)) }^{\vS_{\text{new}}} \,\, \big]^{-1} \\
&= ( \vU \mathrm{Exp}( \stepsize \vU^{-1}  \vG  \vU^{-1} ) \vU )^{-1} \\
&=  \vU^{-1} \mathrm{Exp}( - \stepsize \vU^{-1}  \vG  \vU^{-1} ) \vU^{-1} \\
&=  \vSigma^{1/2} \mathrm{Exp}( - \stepsize \vSigma^{1/2}  \vG  \vSigma^{1/2} ) \vSigma^{1/2} \\
&=  \vSigma^{1/2} \mathrm{Exp}(   \stepsize \vSigma^{1/2}  ( 2 \gradop{\Sigma}\elbofinal ) \vSigma^{1/2} ) \vSigma^{1/2} \\
&=  \vSigma^{1/2} \mathrm{Exp}(   \stepsize \vSigma^{-1/2} \underbrace{[ 2 \vSigma  (\gradop{\Sigma}\elbofinal) \vSigma ]}_{ \vngrad^{[2]} } \vSigma^{-1/2} ) \vSigma^{1/2} \\
&=  \vSigma^{1/2} \mathrm{Exp}(   \stepsize \vSigma^{-1/2}  \vngrad^{[2]}  \vSigma^{-1/2} ) \vSigma^{1/2},
\end{align*} where we use the identity 
$ (\mathrm{Exp}(  \stepsize \vU^{-1}  \vG  \vU^{-1} ) )^{-1}=\mathrm{Exp}( - \stepsize \vU^{-1}  \vG  \vU^{-1} )$.

Note that a geodesic is invariant under parameterization.
Alternatively, we can obtain the above equation by using the fact that 
$\mathbf{R}^{[2]}(\stepsize)$ is a geodesic of Gaussian distribution with a constant mean.

Using a similar proof as shown in Lemma \ref{lemma:gauss_s_chris}, the additional term for $\vSigma$ is
\begin{align*}
\mathrm{Mat}( {\Gamma^{a_2}_{\ \ b_2 c_2}} {\ngrad^{b_2}} {\ngrad^{c_2}} )
=- \vngrad^{[2]}    \vSigma^{-1}  \vngrad^{[2]}   
\end{align*}
 where $\Gamma^{a_2}_{\ \ b_2 c_2}$ is the Christoffel symbol of the second kind for $\mathrm{vec}(\vSigma)$ and  ${\ngrad^{a_2}}$ denotes the $a$-th element of $\mathrm{vec}(\vngrad^{[2]})$.

\section{Example: Gamma Approximation}
\label{app:gamma_case}
We consider the gamma distribution under the parameterization $\vvarpar=\{ \varpar^{[1]}, \varpar^{[2]}\}$, where $\varpar^{[1]}=\alpha$
and $\varpar^{[2]}=\frac{\beta}{\alpha}$.

Since every block contains only a scalar, we use global indexes such as $\varpar^{(i)}=\varpar^{[i]}$ , 
$\varpar^{(i)}=\varpar^{a_i}$ and $\Gamma_{i,ii}=\Gamma_{a_i, b_ic_i}$ for notation simplicity.
The open-set constraint is $\Omega_1 = \mathbb{S}_{++}^1$ and $\Omega_2=\mathbb{S}_{++}^1$.
Under this parameterization, we can express the distribution as below.
\begin{align*}
q(\lat|\vvarpar)=\lat^{-1}
\exp\Big(   \varpar^{(1)}\log  \lat  - \lat\varpar^{(1)}\varpar^{(2)} - A(\vvarpar)\Big) 
\end{align*} where
$A(\vvarpar)=\log\mathrm{Ga}(\varpar^{(1)})- \varpar^{(1)}\left( \log \varpar^{(1)}+ \log \varpar^{(2)}\right)$ and $\mathrm{Ga}(\cdot)$ is the gamma function.
\begin{lemma}
\label{lemma:gamma_bc}
The Fisher information matrix is diagonal under this parameterization.
It implies that this parameterization is a BC parameterization.
\end{lemma}
\begin{proof}
Notice that $\Unmyexpect{q(\lat|\varpar)}\sqr{\lat}=\frac{1}{\varpar^{(2)}}$. The Fisher information matrix is diagonal as shown below.
\begin{align*}
\vfim &= - \Unmyexpect{q(\lat|\varpar)}\sqr{ \hessop{\varpar} {\log q(\lat|\vvarpar)} } \\
&= - \Unmyexpect{q(\lat|\varpar)} \begin{bmatrix}
    - \hessop{\varpar^{(1)}}{A(\vvarpar)} & \left(-\lat+\frac{1}{\varpar^{(2)} } \right)  \\
  \left(-\lat+\frac{1}{\varpar^{(2) }} \right)  & - \hessop{\varpar^{(2)}}{A(\vvarpar)}
   \end{bmatrix} \\
 &=  \Unmyexpect{q(\lat|\varpar)} \begin{bmatrix}
     \hessop{\varpar^{(1)}}{A(\vvarpar)} & 0 \\
  0  &  \hessop{\varpar^{(2)}}{A(\vvarpar)}
   \end{bmatrix}  \\ 
   &= \begin{bmatrix}
\gradop{\varpar^{(1)}}{\psi(\varpar^{(1)})} - \frac{1}{\varpar^{(1)}} & 0 \\
0 & \frac{\varpar^{(1)} }{ \left( \varpar^{(2)} \right)^{2} }
 \end{bmatrix}
\end{align*} where $\psi(\cdot)$ denotes the digamma function.
\end{proof}

\begin{lemma}
 $\vvarpar$ is a BCN  parameterization.
\end{lemma}
\begin{proof}
By Lemma  \ref{lemma:gamma_bc}, we know that $\vvarpar$ is a BC parameterization.
Now, we show that $\vvarpar=\{\varpar^{(1)}, \varpar^{(2)}\}$ is a BCN parameterization.
Clearly, each  $\varpar^{(i)} \in \mathbb{S}_{++}^{1}$ has all degrees of freedom. 

The gamma distribution which can be written as following exponential form:
\begin{align*}
q(\lat|\varpar^{(1)},\varpar^{(2)})=
\lat^{-1}
\exp\Big(   \varpar^{(1)}\log  \lat  - \lat\varpar^{(1)}\varpar^{(2)} - A(\vvarpar)\Big) 
\end{align*} 
Considering two blocks with $\varpar^{(1)}$ and $\varpar^{(2)}$ respectively, we can express this distribution in the following two ways where the first equation is for the $\varpar^{(1)}$ block while the second equation is for the $\varpar^{(2)}$ block:
\begin{align*}
   q(\lat|\varpar^{(1)},\varpar^{(2)}) &= \underbrace{ \lat^{-1}}_{h_1(\lat,\varpar^{(2)})} \exp\Big(  \myang{ \underbrace{ \log\lat - \lat \varpar^{(2)} }_{\phi_1(\lat, \varpar^{(2)})}, \varpar^{(1)}}    - A(\vvarpar)\Big)  \\
    &= \underbrace{\lat^{-1} \exp( \varpar^{(1)} \log \lat ) }_{h_2(\lat,\varpar^{(1)})} \exp\Big(  \myang{ \underbrace{ -\lat \varpar^{(1)} }_{\phi_2(\lat,\varpar^{(1)})}, \varpar^{(2)}}   - A(\vvarpar)\Big) 
\end{align*}
 
Therefore, by the definition of BCN, we know that $\vvarpar$ is a BCN parameterization.

\end{proof}

Using this BCN parameterization, the Christoffel symbols can be readily computed  as below.
\begin{align*}
\Gamma_{1,11} &= \half \gradrdop{\varpar^{(1)}}{A(\vvarpar)} = \half\big( \hessop{\varpar^{(1)}}{\psi(\varpar^{(1)})} + \frac{1}{\left( \varpar^{(1)} \right)^2 }\big)\, ,
\quad 
\Gamma_{2,22} = \half \gradrdop{\varpar^{(2)}}{A(\vvarpar)} = - \frac{ \varpar^{(1)} }{ \left( \varpar^{(2)} \right)^3 } \\
\Gamma^{1}_{\ \ 11} & = \frac{\Gamma_{1,11}}{F_{11}} = \frac{ \hessop{\varpar^{(1)}}{\psi(\varpar^{(1)})} + \frac{1}{\left( \varpar^{(1)} \right)^2 } }{ 2 \left( \gradop{\varpar^{(1)}}{\psi(\varpar^{(1)})} - \frac{1}{\varpar^{(1)}} \right) }\, , \quad
\Gamma^{2}_{\ \ 22}  = \frac{\Gamma_{2,22}}{F_{22}} = -\frac{1}{\varpar^{(2)}}
\end{align*}

\subsection{Proof of Theorem \ref{thm:gamma_update} }
\label{app:gamma_proof}
We first prove the following lemma.
\begin{lemma}
\label{lemma:chris_gamma}
$\Gamma^{1}_{\ \ 11} < -\frac{1}{\varpar^{(1)}}$ when $\varpar^{(1)}>0$. 
\end{lemma}
\begin{proof}
By Eq 1.4 at \citet{batir2005some} and the last inequality at page 13 of \citet{koumandos2008monotonicity},
we have the following inequalities when $\varpar^{(1)}>0$.
\begin{align}
 \gradop{\varpar^{(1)}}{\psi(\varpar^{(1)})} - \frac{1}{\varpar^{(1)}} & > \frac{1}{2 \left(\varpar^{(1)}\right)^2} >0  & \text{  \citet{batir2005some}}
 \label{eq:gamma_ieq1} \\
 \hessop{\varpar^{(1)}}{\psi(\varpar^{(1)})} & <   \frac{1}{\left( \varpar^{(1)}\right)^2} - \frac{ 2  \gradop{\varpar^{(1)}}{\psi(\varpar^{(1)})} }{ \varpar^{(1)} } & \text{  \citet{koumandos2008monotonicity}}
 \label{eq:gamma_ieq2}
\end{align}
By \eqref{eq:gamma_ieq2}, we have
\begin{align*}
 \hessop{\varpar^{(1)}}{\psi(\varpar^{(1)})} +\frac{1}{\left( \varpar^{(1)}\right)^2} <   \frac{2}{\left( \varpar^{(1)}\right)^2} - \frac{ 2  \gradop{\varpar^{(1)}} {\psi(\varpar^{(1)})} }{ \varpar^{(1)} } = \frac{2}{ \varpar^{(1)}}\left( \frac{1}{\varpar^{(1)}} - \gradop{\varpar^{(1)}}{\psi(\varpar^{(1)})} \right)
\end{align*}
Since $\gradop{\varpar^{(1)}}{\psi(\varpar^{(1)})} - \frac{1}{\varpar^{(1)}}>0$, we have
\begin{align*}
 2 \Gamma^{1}_{\ \ 11} = \frac{ \hessop{\varpar^{(1)}}{\psi(\varpar^{(1)})} +\frac{1}{\left( \varpar^{(1)}\right)^2} }{ \gradop{\varpar^{(1)}}{\psi(\varpar^{(1)})} - \frac{1}{\varpar^{(1)}}} < - \frac{2}{ \varpar^{(1)} }
\end{align*} which shows
$\Gamma^{1}_{\ \ 11} < -\frac{1}{\varpar^{(1)}}$.
\end{proof}

Now, We give a proof  for Theorem \ref{thm:gamma_update}.
\begin{proof}
The proposed update for $\varpar^{(1)}$ with step-size $\stepsize$ is given below.
\begin{align*}
\varpar^{(1)} &\leftarrow \varpar^{(1)}  - \stepsize \ngrad^{(1)} - \frac{\stepsize^2}{2} \left( \Gamma^{1}_{\ \ 11}  \right) \left( \ngrad^{(1)} \right)^2 \\
&> \varpar^{(1)}  - \stepsize \ngrad^{(1)} + \frac{\stepsize^2}{2} \left( \frac{1}{\varpar^{(1)}}  \right) \left( \ngrad^{(1)} \right)^2 \\
&= \frac{1}{2\varpar^{(1)}}\sqr{ 2\left( \varpar^{(1)} \right)^2 - 2 \stepsize \ngrad^{(1)} \varpar^{(1)} + \left( \stepsize \ngrad^{(1)} \right)^2 } \\
&= \underbrace{\frac{1}{2\varpar^{(1)}}}_{>0}\sqr{\underbrace{ \left(  \varpar^{(1)} \right)^2}_{>0} + \underbrace{\left( \varpar^{(1)}-  \stepsize \ngrad^{(1)}  \right)^2}_{\geq 0}  } 
\end{align*} where in the second step we use the inequality 
$\Gamma^{1}_{\ \ 11} < -\frac{1}{\varpar^{(1)}}$ shown in  Lemma \ref{lemma:chris_gamma} since the current/old 
$\varpar^{(1)}>0$.

Similarly, we can show the update for $\varpar^{(2)}$ also satisfies the constraint.
\begin{align*}
\varpar^{(2)} &\leftarrow \varpar^{(2)}  - \stepsize \ngrad^{(2)} + \frac{\stepsize^2}{2}\left( \frac{1}{\varpar^{(2)}}\right) \left( \ngrad^{(2)} \right)^2 \\
&= \underbrace{ \frac{1}{2\varpar^{(2)}} }_{>0}\sqr{ \underbrace{\left( \varpar^{(2)} \right)^2}_{>0} + \underbrace{\left( \varpar^{(2)}-  \stepsize \ngrad^{(2)}  \right)^2}_{\geq 0}  } 
\end{align*}

It is obvious to see that the proposed update satisfies the underlying constraint.
\end{proof}

\subsection{Natural Gradients}
\label{app:gamma_ng}
Recall that $\ngrad$ are the natural-gradients, which can be computed as shown below.
\begin{align*}
\ngrad^{(1)} = \frac{\gradop{\varpar^{(1)} }{\elbofinal}}{  \gradop{\varpar^{(1)}}{\psi(\varpar^{(1)})} - \frac{1}{\varpar^{(1)}}} , \,\,
\ngrad^{(2)} =  \frac{ \left( \varpar^{(2)} \right)^2}{ \varpar^{(1)} } \gradop{\varpar^{(2)}}{\elbofinal}
\end{align*}
Recall that $\varpar^{(1)}=\alpha$ and $\varpar^{(2)}=\frac{\beta}{\alpha}$.
Using the chain rule, we know that
\begin{align*}
 \gradop{ \varpar^{(1)}}{\elbofinal} = \gradop{ \alpha }{\elbofinal} +\frac{\beta}{\alpha}\gradop{ \beta }{\elbofinal}  ,  \,\,\, \gradop{\varpar^{(2)}}{\elbofinal}= \alpha \gradop{\beta }{\elbofinal}
\end{align*} 
$\gradop{\alpha }{\elbofinal}$ and $\gradop{ \beta }{\elbofinal}$ can be computed by the implicit reparameterization trick \citep{salimans2013fixed,figurnov2018implicit}.

\section{Example: Exponential Approximation}
\label{app:exp_case}
In this case, there is only one block with a scalar. We use global indexes such as $\varpar^{(1)}=\varpar^{[1]}$ and $\Gamma_{1,11}=\Gamma_{a_1, b_1c_1}$ for notation simplicity.
We consider an exponential distribution under the natural parameterization $\varpar=\varpar^{(1)}$ with the open-set constraint $\Omega = \mathbb{S}_{++}^1$:
\begin{align*}
q(\lat|\varpar)= \exp \left( -\varpar^{(1)} \lat-A(\varpar) \right) 
\end{align*} where $A(\varpar)=- \log \varpar^{(1)}$. 
The FIM is a scalar $\fim_{11} = \frac{1}{\left(\varpar^{(1)}\right)^2} $.
It is obvious that $\varpar$ is a BCN parameterization.
the Christoffel symbols can be readily computed as below.
\begin{align*}
\Gamma_{1,11} = \half \gradrdop{\varpar^{(1)}}{A(\varpar)} = -\frac{1}{\left(\varpar^{(1)}\right)^3}\, , \quad
\Gamma^{1}_{\ \ 11}  = \frac{\Gamma_{1,11}}{F_{11}} = - \frac{1}{\varpar^{(1)}}
\end{align*}
The proposed natural-gradient update with step-size $\stepsize$ is
\begin{align*}
\varpar^{(1)} &= \varpar^{(1)}  - \stepsize \ngrad^{(1)} \color{red}{ + \frac{\stepsize^2}{2}\left( \frac{1}{\varpar^{(1)}}\right) \left( \ngrad^{(1)} \right)^2} 
\end{align*}
where  $\ngrad^{(1)}$ is the natural-gradient.
Note that $\ngrad^{(1)}$ is the natural-gradient, which can be computed as shown below.
\begin{align*}
\ngrad^{(1)} = \left(\varpar^{(1)}\right)^2\gradop{ \varpar^{(1)}}{\elbofinal}.
\end{align*} where $\gradop{ \varpar^{(1)}}{\elbofinal}$ can be computed by the implicit reparameterization trick
as $\gradop{ \varpar^{(1)}}{\elbofinal}\approx \sqr{\gradop{\varpar}{\lat}} \sqr{\gradop{ \lat }{b(\lat)}}$.
where 
$\lat \sim q(\lat|\varpar^{(1)})$ and 
$ b(\lat) :=\bar{\ell}(\lat)+\log q(\lat|\varpar^{(1)})$

\begin{lemma}
The proposed update satisfies the underlying constraint.
\end{lemma}
\begin{proof}
The proposed natural-gradient update with step-size $\stepsize$ is given below.
\begin{align*}
\varpar^{(1)} &\leftarrow \varpar^{(1)}  - \stepsize \ngrad^{(1)} + \frac{\stepsize^2}{2}\left( \frac{1}{\varpar^{(1)}}\right) \left( \ngrad^{(1)} \right)^2 \\
&= \frac{1}{2\varpar^{(1)}}\sqr{ 2\left( \varpar^{(1)} \right)^2 - 2 \stepsize \ngrad^{(1)} \varpar^{(1)} + \left( \stepsize \ngrad^{(1)} \right)^2 } \\
&= \frac{1}{2\varpar^{(1)}}\sqr{ \left( \varpar^{(1)} \right)^2 + \left( \varpar^{(1)}-  \stepsize \ngrad^{(1)}  \right)^2  } 
\end{align*} 
It is obvious to see that the proposed update satisfies the underlying constraint.
\end{proof}

\subsection{Implicit reparameterization gradient}

Now, we discuss how to compute the gradients w.r.t. $\varpar$  using the implicit reparameterization trick.
To use the implicit reparameterization trick, we have to compute the following term.
\begin{align*}
\gradop{\varpar}{\lat} = - \frac{ \gradop{\varpar}{Q(\lat|\varpar)} }{q(\lat|\varpar)} 
= - \frac{ \gradop{\varpar}{\left( 1-  \exp(-\varpar\lat) \right)} }{ \varpar \exp(-\varpar\lat) } 
= - \frac{ \lat \exp(-\varpar\lat)  }{ \varpar \exp(-\varpar\lat) } 
= - \frac{ \lat   }{ \varpar } 
\end{align*} where $Q(\lat|\varpar)$ is the C.D.F. of $q(\lat|\varpar)$.

\section{Example: Inverse Gaussian Approximation}
\label{app:ig_case}
We consider the following  distribution.
\begin{align*}
q(\lat|\alpha, \beta)=  \sqrt{\frac{1}{2\pi \lat^3}}  \exp \left( -\frac{\lat\alpha \beta^2}{2}  - \frac{\alpha}{2\lat}+  \frac{\log \alpha}{2} + \alpha\beta  \right) 
\end{align*}
where $\{ \frac{1}{\beta},  \alpha \}$ is a BC parameterization.

We consider a BCN parameterization $\vvarpar=\{\varpar^{[1]},\varpar^{[2]}\}$, where $\varpar^{[1]}=\beta^2$ and $\varpar^{[2]}=\alpha$ and the open-set constraint is $\Omega_1= \mathbb{S}_{++}^1$ and $ \Omega_2=\mathbb{S}_{++}^1$.
Since every block contains only a scalar, we use global indexes such as $\varpar^{(i)}=\varpar^{[i]}$ and $\Gamma_{i,ii}=\Gamma_{a_i, b_ic_i}$ for notation simplicity.
Under this parameterization, we can re-express the distribution as
\begin{align*}
q(\lat|\vvarpar)= \sqrt{\frac{1}{2\pi \lat^3}}  \exp \left(  -\frac{\lat}{2} \varpar^{(1)} \varpar^{(2)} - \frac{ \varpar^{(2)}}{2\lat} -A(\vvarpar) \right) 
\end{align*} where $A(\vvarpar)=-\frac{\log \varpar^{(2)}}{2} - \varpar^{(2)}\sqrt{\varpar^{(1)}} $.


\begin{lemma}
\label{lemma:fisher_ig}
The FIM is (block) diagonal under this parameterization.
\end{lemma}
\begin{proof}
Notice that $\Unmyexpect{q(\lat|\varpar)}\sqr{\lat}=\frac{1}{\sqrt{\varpar^{(1)}}}$.
The FIM is (block) diagonal as shown below.
\begin{align*}
\vfim &= - \Unmyexpect{q(\lat|\varpar)}\sqr{ \hessop{\varpar}{\log q(\lat|\vvarpar)} } \\
&= - \Unmyexpect{q(\lat|\varpar)} \begin{bmatrix}
    - \hessop{\varpar^{(1)}} {A(\vvarpar)}& \half\left(-\lat+\frac{1}{\sqrt{ \varpar^{(1)} }} \right)  \\
  \half\left(-\lat+\frac{1}{\sqrt{ \varpar^{(1)} }} \right)  & - \hessop{\varpar^{(2)}} {A(\vvarpar)}
   \end{bmatrix} \\
 &=  \Unmyexpect{q(\lat|\varpar)} \begin{bmatrix}
     \hessop{\varpar^{(1)}} {A(\vvarpar)} & 0 \\
  0  &  \hessop{\varpar^{(2)}} {A(\vvarpar)}
   \end{bmatrix}  \\ 
   &= \begin{bmatrix}
 \frac{1}{4} \left( \varpar^{(1)} \right)^{-3/2} \varpar^{(2)} & 0 \\
0 & \frac{1}{2} \left(\varpar^{(2)}\right)^{-2}
 \end{bmatrix}
\end{align*}
\end{proof}

It is easy to show that $\vvarpar$ is a BCN parameterization since 
$\vvarpar$ satisfies Assumption 1 to 3.

Due to the BCN parameterization, the Christoffel symbols can be readily computed as below.
\begin{align*}
\Gamma_{1,11} &= \half \gradrdop{\varpar^{(1)}}{A(\vvarpar)} = -\frac{3}{16} \left( \varpar^{(1)} \right)^{-5/2} \varpar^{(2)}\, , \quad
\Gamma_{2,22} = \half \gradrdop{\varpar^{(2)}}{A(\vvarpar)} = -\frac{1}{2} \left( \varpar^{(2)} \right)^{-3} \\
\Gamma^{1}_{\ \ 11} & = \frac{\Gamma_{1,11}}{F_{11}} = - \frac{3}{4\varpar^{(1)}} \, , \quad
\Gamma^{2}_{\ \ 22}  = \frac{\Gamma_{2,22}}{F_{22}} = - \frac{1}{\varpar^{(2)}}
\end{align*}

The proposed natural-gradient update with step-size $\stepsize$ is
\begin{align*}
\varpar^{(1)} &\leftarrow \varpar^{(1)}  - \stepsize \ngrad^{(1)}  \color{red}{+ \frac{\stepsize^2}{2} \left( \frac{3}{4\varpar^{(1)}}\right) \left( \ngrad^{(1)} \right)^2} \\
\varpar^{(2)} &\leftarrow \varpar^{(2)}  - \stepsize \ngrad^{(2)}  \color{red}{+ \frac{\stepsize^2}{2}\left( \frac{1}{\varpar^{(2)}}\right) \left( \ngrad^{(2)} \right)^2 }
\end{align*}

\begin{lemma}
The update above satisfies the underlying constraint.
\end{lemma}
\begin{proof}
The proposed natural-gradient update with step-size $\stepsize$ is given below.
\begin{align*}
\varpar^{(1)} &\leftarrow \varpar^{(1)}  - \stepsize \ngrad^{(1)} + \frac{\stepsize^2}{2} \left( \frac{3}{4\varpar^{(1)}}\right) \left( \ngrad^{(1)} \right)^2 \\
&= \frac{1}{4\varpar^{(1)}}\sqr{ 4\left( \varpar^{(1)} \right)^2 - 4 \stepsize \ngrad^{(1)} \varpar^{(1)} + \frac{3}{2}\left( \stepsize \ngrad^{(1)} \right)^2 } \\
&= \frac{1}{4\varpar^{(1)}}\sqr{ \underbrace{ \left( 2\varpar^{(1)}  -  \stepsize \ngrad^{(1)} \right)^2}_{\text{Term I}} + \underbrace{ \frac{1}{2}\left( \stepsize \ngrad^{(1)} \right)^2}_{\text{Term II}} } \\
\varpar^{(2)} &\leftarrow \varpar^{(2)}  - \stepsize \ngrad^{(2)} + \frac{\stepsize^2}{2}\left( \frac{1}{\varpar^{(2)}}\right) \left( \ngrad^{(2)} \right)^2 \\
&= \frac{1}{2\varpar^{(2)}}\sqr{ 2\left( \varpar^{(2)} \right)^2 - 2 \stepsize \ngrad^{(2)} \varpar^{(2)} + \left( \stepsize \ngrad^{(2)} \right)^2 } \\
&= \frac{1}{2\varpar^{(2)}}\sqr{ \left( \varpar^{(2)} \right)^2 + \left( \varpar^{(2)}-  \stepsize \ngrad^{(2)}  \right)^2  } 
\end{align*} 
Note that Term I and Term II cannot be both zero at the same time when $\varpar^{(1)}>0$.
A similar argument can be made for the update about $\varpar^{(2)}$.
Therefore, the proposed update satisfies the underlying constraint.
\end{proof}

Recall that $\ngrad$ are the natural-gradients, which can be computed as shown below.
\begin{align*}
\ngrad^{(1)} =  \frac{4   }{\varpar^{(2)}} \left( \varpar^{(1)} \right)^{3/2}\gradop{\varpar^{(1)}}{\elbofinal}, \,\,
\ngrad^{(2)} =  2\left( \varpar^{(2)} \right)^2 \gradop{\varpar^{(2)}}{\elbofinal}
\end{align*}
Using the chain rule, we know that
\begin{align*}
 \gradop{\varpar^{(1)}}{\elbofinal} = \frac{1}{ 2\beta} \gradop{\beta}{\elbofinal}  ,  \,\,\, \gradop{\varpar^{(2)}}{\elbofinal}= \gradop{\alpha}{\elbofinal}
\end{align*}
$\gradop{\alpha}{\elbofinal}$ and $\gradop{\beta}{\elbofinal}$ can be computed by the implicit reparameterization trick \citep{salimans2013fixed,figurnov2018implicit}
as $\gradop{\eta}{\elbofinal} \approx \sqr{\gradop{\eta}{\lat}}\sqr{\nabla_{\lat}{b(\lat)}}$, where $\eta=\{\alpha,\beta\}$, $\lat \sim q(\lat|\alpha,\beta)$ and $ b(\lat) :=\bar{\ell}(\lat)+\log q(\lat|\alpha,\beta)$

\subsection{Implicit reparameterization gradient}
Now, we discuss how to compute the gradients w.r.t. $\alpha$ and $\beta$ using the implicit reparameterization trick.
To use the implicit reparameterization trick, we have to compute the following term.
\begin{align*}
\gradop{\eta}{\lat} &= - \frac{ \gradop{\eta}{Q(\lat|\veta)} }{q(\lat|\veta)} \\
&= -\frac{ \gradop{\eta}{\sqr{ \Phi( \sqrt{\frac{\alpha}{\lat}} \left(\lat \beta-1 \right) ) + \exp(2\alpha\beta)\Phi(-\sqrt{\frac{\alpha}{\lat}} \left(\lat\beta+1\right) )} }}{\sqrt{\frac{1}{2\pi \lat^3}}   \exp \left( -\frac{\lat\alpha \beta^2}{2}  - \frac{\alpha}{2\lat}+  \frac{\log \alpha}{2} + \alpha\beta  \right)}
\end{align*} where $\veta=\{\alpha,\beta\}$, $Q(\lat|\veta)$ is the C.D.F. of the inverse Gaussian distribution, and $\Phi(x)=\int_{-\infty}^{x} \gauss(t|0,1) dt$ is the C.D.F. of the standard Gaussian distribution.
We use the following fact to simplify the above expression.
\begin{align*}
\delta(\lat,\alpha,\beta) :=\frac{\exp(2\alpha\beta)\Phi(-\sqrt{\frac{\alpha}{\lat}} \left(\lat\beta+1 \right) ) }{ \gauss( \sqrt{\frac{\alpha}{\lat}}\left(\lat\beta-1 \right)|0,1 ) } 
= \frac{ \Phi( -\sqrt{ \frac{\alpha}{\lat} } \left( \lat\beta+1\right) )  }{\gauss( -\sqrt{ \frac{\alpha}{\lat} } \left( \lat\beta+1\right)   |0,1) }
\end{align*} where $\delta(\lat,\alpha,\beta)$  is known as the Mills ratio of Gaussian distribution.
Using this fact, we can get the simplified expressions as follows.
\begin{align*}
\gradop{\alpha} {\lat} &= \frac{\lat}{\alpha} - 2\beta \lat^{3/2}\alpha^{-1/2}\delta(\lat,\alpha,\beta) \\
\gradop{\beta} {\lat} &=  - 2 \lat^{3/2}\alpha^{1/2}\delta(\lat,\alpha,\beta) 
\end{align*} where we compute  $\log (\delta(\lat,\alpha,\beta))$ for numerical stability since
the logarithm of Gaussian cumulative distribution function can be computed by using existing libraries, such as
the scipy.special.log\_ndtr() function.

In fact, we have closed-form expressions of gradients of the entropy term as shown below.
\begin{align*}
\Unmyexpect{q(\lat|\eta)}\sqr{-\log q(\lat|\veta)} &= \half \sqr{ -\log \alpha - 3\left( \log \beta + \exp(2\alpha\beta) E_1(2\alpha\beta)  \right) + 1 + \log(2\pi)}  \\
\gradop{\alpha}{\Unmyexpect{q(\lat|\eta)}\sqr{-\log q(\lat|\veta)}} &= \frac{1}{\alpha} - 3\beta \exp(2\alpha\beta) E_1(2\alpha\beta) \\
\gradop{\beta}{\Unmyexpect{q(\lat|\eta)}\sqr{-\log q(\lat|\veta)}} &= -3\alpha \exp(2\alpha\beta) E_1(2\alpha\beta) 
\end{align*} where $E_1(x):= \int_{x}^{\infty} \frac{e^{-t}}{t} dt $ is the exponential integral.
It is not numerical stable to compute the product $\exp(x) E_1(x)$ when $x>100$.
In this case, we can use the asymptotic expansion (see Eq 3 at \citet{tseng1998numerical}) for the exponential integral to approximate the product as shown below.
\begin{align*}
\exp(x) E_1(x) \approx \frac{1}{x}\sqr{ 1+\sum_{n=1}^{N} \frac{ (-1)^n n! }{x^n} } \,\,\, \text{ when } x>100,
\end{align*} where $N$ is an integer such as $N\leq x< N+1$.

\section{Mixture of Exponential Family Distributions}
\label{app:cef_rgvi}
Let's consider the following mixture of exponential family distributions $q(\vlat)=\int q(\vlat,\vmix) d\vmix$.
The joint distribution $q(\vlat,\vmix|\vvarpar)=q(\vmix|\vvarpar_\mix)q(\vlat|\vmix,\vvarpar_\lat)$ is known as the conditional exponential family (CEF) defined by  \citet{lin2019fast}.
\begin{align*}
q(\vmix|\vvarpar_\mix) &:= h_{\mix}(\vmix)\exp\sqr{ \myang{\vphi_\mix(\vmix), \vvarpar_\mix} - A_{\mix}(\vvarpar_\mix)} \\
q(\vlat|\vmix,\vvarpar_\lat) &:= h_{\lat}(\vmix,\vlat)\exp\sqr{ \myang{\vphi_\lat(\vmix,\vlat), \vvarpar_\lat} - A_{\lat}(\vvarpar_\lat,\vmix)}  
\end{align*} where $\vvarpar=\{\vvarpar_\lat,\vvarpar_\mix\}$.

We will use the joint Fisher information matrix(FIM) suggested by \citet{lin2019fast} as the metric $\vfim$ to derive our improved learning rule for mixture approximations.

\subsection{The Joint Fisher Information Matrix and the Christoffel Symbol}
\citet{lin2019fast} propose to use the FIM of the joint distribution $q(\vmix,\vlat|\vvarpar)$ , where they refer this FIM as the joint FIM.
The joint FIM and the corresponding Christoffel symbol of the first kind are defined as follows.
\begin{align*}
 \fim_{ab} &:= -\Unmyexpect{q(\mix,\lat|\varpar)}\sqr{ \crossop{a}{b}{ \log q(\vmix,\vlat|\vvarpar)} } \\
  \Gamma_{d,a b} & := \half \sqr{ \gradop{a}{\fim_{bd}} +\gradop{b}{\fim_{ad}} - \gradop{d}{\fim_{ab}} }
\end{align*} where we denote
$\gradop{a}{}=\gradop{\varpar^{a}}{}$ for notation simplicity.

Like the exponential family cases as shown in Eq. \eqref{eq:chris}, the  Christoffel symbol of the first kind can be computed as
\begin{align}
  \Gamma_{d,a b}
=   \half \Big[ & \Unmyexpect{q(\mix,\lat|\varpar)}\sqr{ \crossop{a}{b}{\log q(\vmix,\vlat|\vvarpar)} \gradop{d} {\log q(\vmix,\vlat|\vvarpar)} }    
-\Unmyexpect{q(\mix,\lat|\varpar)}\sqr{ \crossop{b }{d}{\log q(\vmix,\vlat|\vvarpar)} \gradop{a }{ \log q(\vmix,\vlat|\vvarpar)} } \nonumber \\
& -\Unmyexpect{q(\mix,\lat|\varpar)}\sqr{ \crossop{a }{d}{\log q(\vmix,\vlat|\vvarpar)} \gradop{b }{ \log q(\vmix,\vlat|\vvarpar)} } 
-\Unmyexpect{q(\mix,\lat|\varpar)}\sqr{ \crossrdop{a}{b}{d}{ \log q(\vmix,\vlat|\vvarpar)} } \Big] 
\label{eq:chris_cef}
\end{align}

\subsection{The BCN Parameterization}
\label{app:BCN_CEF}

Now, we show that how to simplify the computation of the Christoffel symbol by extending the BCN parameterization for this kind of mixtures.

To this end,  we first assume that $\vvarpar$ can be partitioned with $(m+n)$ blocks to satisfy Assumption 1 in the main text.
\begin{align*}
\vvarpar = \{ \underbrace{\vvarpar_\lat^{[1]}, \dots, \vvarpar_\lat^{[m]}}_{\boldsymbol{\varpar_\lat}}, \underbrace{\vvarpar_\mix^{[m+1]}, \dots, \vvarpar_\mix^{[m+n]}}_{\boldsymbol{\varpar_\mix}} \} 
\end{align*}

Then,  we  extend the definition of BC parameterization to the conditional exponential family, which is similar to Assumption 2 in the main text and a concrete example of Definition \ref{assump:BC_general} in Appendix \ref{app:block_riem_general}.

{\bf Assumption 2 [Block Coordinate Parameterization] :} \emph{
   A parameterization satisfied Assumption 1 is block coordinate (BC) if the {\color{red}joint FIM} under this parameterization is block-diagonal according to the block structure of the parameterization.
}

\citet{lin2019fast} show that, for any parameterization $\vvarpar=\{\vvarpar_\lat,\vvarpar_\mix\}$, the joint FIM has the following two blocks: $\vfim_\lat$ for block $\vvarpar_\lat$ and $\vfim_\mix$ for block $\vvarpar_\mix$ as given  below.
\begin{align*}
\vfim =\begin{bmatrix}
\vfim_\lat & \mathbf{0} \\
\mathbf{0} & \vfim_\mix
       \end{bmatrix}
\end{align*}

Assumption 2 implies that $\vfim_\mix$ and $\vfim_\lat$ are both block-diagonal according to the block structure of $\vvarpar_\mix$ and $\vvarpar_\lat$, respectively. The block diagonal structure is given below if $\vvarpar=\{ \vvarpar_\lat^{[1]}, \dots, \vvarpar_\lat^{[m]}, \vvarpar_\mix^{[m+1]}, \dots, \vvarpar_\mix^{[m+n]} \}$ is a BC parameterization.
\begin{align*}
 \vfim =\begin{bmatrix}
  \overbrace{  \begin{bmatrix}
 \vfim_\lat^{[1]}  & \dots & \mathbf{0}\\
 \vdots &   \ddots & \vdots \\
 \mathbf{0}  & \dots & \vfim_\lat^{[m]}\\
 \end{bmatrix}}^{\vfim_\lat} &\mathbf{0} \\
 \mathbf{0} &  \underbrace{
 \begin{bmatrix}
 \vfim_\mix^{[m+1]}  & \dots & \mathbf{0}\\
 \vdots &   \ddots & \vdots \\
 \mathbf{0}  & \dots & \vfim_\mix^{[m+n]}\\
 \end{bmatrix}}_{\vfim_\mix} 
                   \end{bmatrix}
\end{align*}

{\bf Assumption 3 [Block Natural Parameterization for the Conditional Exponential-Family] :}
   \emph{
      For a conditional exponential-family distribution $q(\vmix,\vlat|\vvarpar)=q(\vmix|\vvarpar_\mix)q(\vlat|\vmix,\vvarpar_\lat)$,
       a parameterization $\vvarpar=\{\vvarpar_\lat,\vvarpar_\mix\}$ has
      the following properties.
      \begin{itemize}
       \item  $\vvarpar_\mix$ is a BCN parameterization of the exponential family distribution $q(\vmix|\vvarpar_\mix)$ as defined  in the main text.
       \item $\vvarpar_\lat$ is a parameterization of $q(\vlat|\vmix,\vvarpar_\lat)$, where
there exist  function $\phi_{\lat_i}$ and $h_{\lat_i}$ for each block $\vvarpar_\lat^{[i]}$ such that conditioning on $\vmix$, $q(\vlat|\vmix,\vvarpar_\lat)$ can be re-expressed as a minimal conditional exponential family  distribution (see \citet{lin2019fast} for the definition of the minimality) given that the rest of blocks $\vvarpar_\lat^{[-i]}$  are known.
   \begin{align*}
       q(\vlat|\vmix,\vvarpar_\lat) \equiv h_{\lat_i}(\vmix,\vlat,\vvarpar_\lat^{[-i]}) \exp\big[ \myang{\vphi_{\lat_i}(\vmix,\vlat,\vvarpar_\lat^{[-i]}), \vvarpar_\lat^{[i]}}  - A_{\lat}(\vvarpar_\lat,\vmix) \big]
   \end{align*} 
      \end{itemize}
   }

We say $\vvarpar=\{\vvarpar_\lat,\vvarpar_\mix\}$ is a BCN parameterization for the mixture if it satisfies Assumption 1 to 3.

Many mixture approximations studied in \citet{lin2019fast} have a BCN parameterization.
For concrete examples, see Appendix \ref{app:mog_case} and  \ref{app:skewg_case}.

\subsection{Our Learning Rule for Mixture Approximations}
Now, we are ready to discuss the learning rule for mixture approximations.
Under a BC parameterization $\vvarpar=\{\vvarpar_\lat,\vvarpar_\mix\}$, our learning rule remains the same as shown below. 
\begin{align*}
\varpar^{c_i}  \leftarrow \varpar^{c_i} - \stepsize \ngrad^{c_i} \color{red}{ - \frac{\stepsize^2}{2} \Gamma_{\ a_ib_i}^{c_i} \ngrad^{a_i}\ngrad^{b_i} }
\end{align*} where block $i$ can be either a block of $\vvarpar_\mix$ or $\vvarpar_\lat$.

First, note that the sub-block matrix $\vfim_\mix$ of the joint FIM  is indeed the FIM of $q(\vmix|\vvarpar_\mix)$. Furthermore, $q(\vmix|\vvarpar_\mix)$ is an exponential family distribution.
If $\vvarpar=\{\vvarpar_\lat,\vvarpar_\mix\}$ is a BCN parameterization, it is easy to see that the computation of the Christoffel symbol for $\vvarpar_\mix$ is
 exactly the same as the exponential family cases as discussed in Appendix \ref{app:ef_rgvi}.
 
Furthermore, we can simplify the Christoffel symbol for $\vvarpar_\lat$ due to the following Theorem.
\begin{thm}
\label{thm:CEF_chris}
If $\vvarpar$ is a BCN parameterization of a conditional exponential family (CEF) with the joint FIM,
 natural gradient and 
the Christoffel symbol of the first kind for block $\vvarpar_\lat^{[i]}$  can be simplified as
 \begin{align*}
 \ngrad^{a_i} = \partial_{\varmean_{\lat_{a_i}}}\elbofinal  \,\,; \,\,  \Gamma_{d_i,a_ib_i} =  \half \Unmyexpect{q(\mix|\varpar_\mix)}\sqr{ \crossrdop{\varpar_{\lat^{a_i}}}{\varpar_{\lat^{b_i}}}{\varpar_{\lat^{d_i}}}{ A_{\lat}(\vvarpar_\lat,\vmix) } } 
\end{align*}
 where
 $\varpar_\lat^{a_i}$ is the $a$-th element of $\vvarpar_\lat^{[i]}$;
$\varmean_{\lat_{a_i}}$ denotes the $a$-th element of the block coordinate expectation parameter $\vvarmean_{\lat_{[i]}}=\Unmyexpect{q(\mix,\lat|\varpar)}\sqr{ \vphi_{\lat_i}(\vmix,\vlat,\vvarpar_\lat^{[-i]}) }= \Unmyexpect{q(\mix|\varpar_\mix)}\sqr{ \partial_{\varpar_\lat^{[i]}} A_{\lat}(\vvarpar_\lat,\vmix) }$.
\end{thm}

\subsection{Proof of Theorem  \ref{thm:CEF_chris}}
\begin{proof}

We assume $\vvarpar_\lat=\{\vvarpar_\lat^{[1]},\cdots,\vvarpar_\lat^{[m]}\}$ is partitioned with $m$ blocks.

Since $\vvarpar$ is a BCN  parameterization,
conditioning on $\vmix$ and given $\vvarpar_\lat^{[-i]}$ and $\vvarpar_\mix$ are known,
we can re-express $q(\vlat|\vmix,\vvarpar_\lat)$ as
\begin{align*}
 q(\vlat|\vmix, \vvarpar_\lat) = h_{\lat_i}(\vlat,\vmix,\vvarpar_\lat^{[-i]})\exp\sqr{ \myang{\vphi_{\lat_i}(\vlat,\vmix,\vvarpar_\lat^{[-i]}), \vvarpar_\lat^{[i]}} - A_\lat(\vvarpar_\lat,\vmix)} 
\end{align*}
where  $q(\vlat|\vmix,\vvarpar_\lat)$ is also a one-parameter EF distribution conditioning on $\vvarpar_\lat^{[-i]}$ and $\vmix$. Similarly, we have the following results.
\begin{align*}
 \crossop{a_i}{b_i}{\log q(\vlat|\vmix,\vvarpar_\lat)} &= -\crossop{a_i}{b_i}{A_\lat(\vvarpar_\lat,\vmix)}\\
 \Unmyexpect{q(\lat|\mix,\varpar_\lat)}\sqr{ \gradop{a_i} {\log q(\vlat|\vmix,\vvarpar_\lat)}}  &= 0 
\end{align*}
where $\gradop{a_i}{}=\gradop{\varpar_\lat^{a_i}}{}$ is for notation simplicity.
Using the above identities, we have
\begin{align*}
 \Unmyexpect{q(\lat,\mix|\varpar)}\sqr{ \crossop{a_i}{b_i}{\log q(\vlat,\vmix|\vvarpar)} \gradop{d_i} {\log q(\vlat,\vmix|\vvarpar)} }
&= \Unmyexpect{q(\lat,\mix|\varpar)}\sqr{ \crossop{a_i}{b_i}{\log q(\vlat|\vmix,\vvarpar_\lat)} \gradop{d_i} {\log q(\vlat|\vmix,\vvarpar_\lat)} } \\
&= \Unmyexpect{q(\mix|\varpar_\mix)}\sqr{ \Unmyexpect{q(\lat|\mix,\varpar_\lat)}\sqr{ \crossop{a_i}{b_i}{\log q(\vlat|\vmix,\vvarpar_\lat)} \gradop{d_i} {\log q(\vlat|\vmix,\vvarpar_\lat)} } }\\
&= -\Unmyexpect{q(\mix|\varpar_\mix)}\Big[ \crossop{a_i}{b_i}{A_\lat(\vvarpar_\lat,\vmix)}
 \underbrace{\Unmyexpect{q(\lat|\mix,\varpar_\lat)}\sqr{ \gradop{d_i} {\log q(\vlat|\vmix,\vvarpar_\lat)}}}_{0} \Big]\\
&= 0 
\end{align*}

Therefore, by Eq. \eqref{eq:chris_cef}, we can simplify the Christoffel symbol for $\vvarpar_\lat^{[i]}$  as follows.
 \begin{align*}
  \Gamma_{d_i,a_i b_i}
=&   -\half \Unmyexpect{q(\lat,\mix|\varpar)}\sqr{ \crossrdop{a_i}{b_i}{d_i}{ \log q(\vlat,\vmix|\vvarpar)} } \\
=&   -\half \Unmyexpect{q(\mix|\varpar_\mix)}\sqr{ \crossrdop{a_i}{b_i}{d_i}{ \log q(\vanyvar|\vmix,\vvarpar_\lat)} } \\
=&   \half \Unmyexpect{q(\mix|\varpar_\mix)}\sqr{ \crossrdop{a_i}{b_i}{d_i}{ A_\lat(\vvarpar_\lat,\vmix) } } 
\end{align*} where we use $d_i$ to denote the $d$-th entry of block $\vvarpar_\lat^{[i]}$.

Likewise, let $\vvarmean_{\lat_{[i]}}=\Unmyexpect{q(\lat,\mix|\varpar)}\sqr{ \vphi_{\lat_i}(\vlat,\vmix,\vvarpar_\lat^{[-i]}) }$ denote the block coordinate expectation parameter.
We have
\begin{align*}
 0 = \Unmyexpect{q(\mix|\varpar_\mix)}\Big[ \underbrace{ \Unmyexpect{q(\lat|\mix,\varpar_\lat)}\sqr{ \gradop{a_i} {\log q(\vlat|\vmix,\vvarpar_\lat)}}}_{0}\Big] =
 \varmean_{\lat_{a_i}} - \Unmyexpect{q(\mix|\varpar_\mix)}\sqr{ \gradop{a_i}{ A_\lat(\vvarpar_\lat,\vmix)} } 
\end{align*} where $\varmean_{\lat_{a_i}} $ denotes the $a$-th element of  $\vvarmean_{\lat_{[i]}}$.

Therefore, we know that $ \varmean_{\lat_{a_i}} = \Unmyexpect{q(\mix|\varpar_\mix)}\sqr{ \gradop{a_i}{ A_\lat(\vvarpar_\lat,\vmix)} } $.

Recall that the sub-block of the joint FIM for $\vvarpar_\lat^{[i]}$ denoted by $\vfim_\lat^{[i]}$ can be computed as
\begin{align*}
F_{a_ib_i} & = - \Unmyexpect{q(\lat,\mix|\varpar)}\sqr{ \crossop{b_i}{a_i}{\log q(\vlat,\vmix|\vvarpar)} } \\
&=- \Unmyexpect{q(\lat,\mix|\varpar)}\sqr{ \crossop{b_i}{a_i}{\log q(\vlat|\vmix,\vvarpar_\lat)}} \\
&=- \Unmyexpect{q(\lat,\mix|\varpar)}\sqr{ - \crossop{b_i}{a_i}{A_\lat(\vvarpar_\lat,\vmix)}} \\
&= \Unmyexpect{q(\mix|\varpar_\mix)}\sqr{\crossop{b_i}{a_i}{A_\lat(\vvarpar_\lat,\vmix)}} \\
&= \gradop{b_i} \Unmyexpect{q(\mix|\varpar_\mix)}\sqr{\gradop{a_i}{A_\lat(\vvarpar_\lat,\vmix)}} \\
&= \gradop{b_i} \varmean_{\lat_{a_i}}
\end{align*} where we use the fact that $\vvarpar_\mix$ does not depend on $\varpar_\lat^{b_i} \in \vvarpar_\lat$  and
$\gradop{b_i}{}=\gradop{\varpar_\lat^{b_i}}{}$  to move from the fourth step to the fifth step.

Recall that when $\vvarpar$ is a BC parameterization, the joint FIM $\vfim$ is block-diagonal as shown below.
\begin{align*}
 \vfim =\begin{bmatrix}
  \overbrace{  \begin{bmatrix}
 \vfim_\lat^{[1]}  & \dots & \mathbf{0}\\
 \vdots &   \ddots & \vdots \\
 \mathbf{0}  & \dots & \vfim_\lat^{[m]}\\
 \end{bmatrix}}^{\vfim_\lat} &\mathbf{0} \\
 \mathbf{0} &  \underbrace{
 \begin{bmatrix}
 \vfim_\mix^{[m+1]}  & \dots & \mathbf{0}\\
 \vdots &   \ddots & \vdots \\
 \mathbf{0}  & \dots & \vfim_\mix^{[m+n]}\\
 \end{bmatrix}}_{\vfim_\mix} 
                   \end{bmatrix}
\end{align*}

If $\vfim_\lat^{[i]}$  is positive definite everywhere, we have
\begin{align*}
\fim^{a_ib_i} &= \gradop{\varmean_{\lat_{a_i}}} \varpar_\lat^{b_i}
\end{align*}

The above assumption is true if
given that  $\vvarpar_\lat^{[-i]}$ and $\vvarpar_\mix$ are known, $q(\vmix,\vlat|\vvarpar)$ is a one-parameter minimal CEF distribution (See Theorem 2 of \citet{lin2019fast}).

The above result implies that we can  compute natural gradients as follows.
\begin{align*}
 \ngrad^{a_i} = \fim^{a_i b_i} g_{b_i} = \sqr{\gradop{\varmean_{\lat_{a_i}}} \varpar_\lat^{b_i}}\sqr{ \partial_{\varpar_\lat^{b_i}} \elbofinal }  = \gradop{\varmean_{\lat_{a_i}}} \elbofinal
\end{align*} where $g_{b_i}=\partial_{\varpar_\lat^{b_i}} \elbofinal$.

\end{proof}

If we can interchange the differentiations and the integration, we can show,
by Theorem \ref{thm:CEF_chris},
 we have $\Gamma_{\, \, \, \, \, \, a_i b_i}^{c_i} = \half \crossrdop{\varmean_{\lat_{c_i}}}{\varpar_{\lat^{a_i}}}{\varpar_{\lat^{b_i}}} \Unmyexpect{q(\mix|\varpar_\mix)}\sqr{ { A_{\lat}(\vvarpar_\lat,\vmix) } } $ since $A_{\lat}(\vvarpar_\lat,\vmix)$ is $C^3$-smooth w.r.t. $\vvarpar_\lat^{[i]}$ .

\section{Example: Finite Mixture of Gaussians Approximation}
\label{app:mog_case}
We consider a K-mixture of Gaussians under this parameterization $\vvarpar=\{ \{ \vmu_c, \vS_c \}_{c=1}^{K}, \vvarpar_\mix \}$
\begin{align*}
q(\vlat|\vpi, \{\vmu_c,\vS_c\}_{c=1}^{K}) &  = \sum_{c=1}^{K} \pi_c \gauss(\vlat|\vmu_c,\vS_c)
\end{align*} where $\pi_c$ is the  mixing weight so that $\sum_{c=1}^{K}\pi_c=1$ , $\vS_c=\vSigma_c^{-1}$, $\vvarpar_\mix = \{ \log(\pi_c/\pi_K) \}_{c=1}^{K-1} $ and $\pi_K=1-\sum_{c=1}^{K-1}\pi_c$.
The  constraints are  $\vvarpar_\mix \in \mathbb{R}^{K-1}$, $ \vmu_c \in \mathbb{R}^d$, and $\vS_c \in \mathbb{S}^{d \times d}_{++}$.

Under this parameterization, the joint distribution can be expressed as below.
\begin{align*}
q(\vlat,\mix|\vvarpar) &= q(\mix|\vvarpar_\mix) q(\vlat|\mix, \{\vmu_c, \vS_c\}_{c=1}^{K}) \\
q(\mix|\vvarpar_\mix) &=  \exp( \sum_{c=1}^{K-1} \mathbb{I}(\mix=c) \varpar_{\mix_c} - A_\mix(\vvarpar_\mix) ) \\
q(\vlat|\mix, \{\vmu_c, \vS_c\}_{c=1}^{K}) &=
\exp\Big( \sum_{c=1}^{K} \mathbb{I}(\mix=c)\sqr{ -\half  \vlat^T\vS_c\vlat + \vlat^T\vS_c\vmu_c} - A_\lat(\{\vmu_c, \vS_c\}_{c=1}^{K},\mix)\Big) 
\end{align*} where 
$\mixCompA(\vmu_c, \vS_c)=\half\sqr{ \vmu_{c}^T\vS_{c}\vmu_{c} - \log \left|\vS_{c}/(2\pi)\right| }$,
$A_\lat(\{\vmu_c, \vS_c\}_{c=1}^{K},\mix)= \sum_{c=1}^{K} \mathbb{I}(w=c) \mixCompA(\vmu_c, \vS_c)$,
$\varpar_{\mix_c}=\log(\frac{\pi_c}{\pi_K})$,
$A_\mix(\vvarpar_\mix) = \log (1+ \sum_{c=1}^{K-1} \exp(\varpar_{\mix_c}) ) $.

\begin{lemma}
\label{lemma:mog_bc}
 The joint FIM is block diagonal under this parameterization.
\begin{align*}
\vfim = \begin{bmatrix}
 \begin{bmatrix} 
 \vfim_{\mu_1} & {\color{green}\mathbf{0}} \\
 {\color{green}\mathbf{0}} &   \vfim_{S_1}
\end{bmatrix}  &   \cdots & {\color{blue}\mathbf{0}} & {\color{red}\mathbf{0}} \\
 \vdots & \ddots & \vdots & \vdots \\
  {\color{blue} \mathbf{0}}  & \cdots & 
 \begin{bmatrix} 
 \vfim_{\mu_K} & {\color{green}\mathbf{0}} \\
{\color{green}\mathbf{0}} &   \vfim_{S_K}
\end{bmatrix} & {\color{red} \mathbf{0} } \\
\\
 {\color{red}  \mathbf{0} }& \cdots & {\color{red}\mathbf{0}} &  \vfim_\mix \\
                   \end{bmatrix}
\end{align*}
Therefore, this parameterization is a BC parameterization.
 
\end{lemma}
\begin{proof}
We will prove this lemma by showing that all cross terms are zeros.

Case 1: First, we will show that cross terms (shown in red) between $\vvarpar_\mix$ and $\vvarpar_\lat:=\{ \vmu_c, \vS_c \}_{c=1}^{K} $ are zeros.

Let's denote $\varpar_\mix^{i}$ be an element of $\vvarpar_\mix$ and $\varpar_\lat^{j}$ be an element of $\vvarpar_\lat$.
By the definition, each cross term in this case is defined as belows.
\begin{align*}
& - \Unmyexpect{q(\lat,\mix|\varpar)}\sqr{\crossop{\varpar_\mix^{i}}{\varpar_\lat^{j}}{\log q(\vlat,\mix|\vvarpar)} } 
= - \Unmyexpect{q(\lat,\mix|\varpar)}\sqr{\crossop{\varpar_\mix^{i}}{\varpar_\lat^{j}}{\big( \log q(\mix|\vvarpar_\mix) + \log q(\vlat|\mix, \vvarpar_\lat \big)} }  = 0
\end{align*}

Case 2: Next, we will show that cross terms between (shown in blue) any two Gaussian components are zeros.

Let's denote $\varpar_a^{i}$ be an element of $\{ \vmu_a, \vS_a \}$ and $\varpar_b^{j}$ be an element of $\{ \vmu_b, \vS_b \}$, where $a \neq b$.

By the definition, each cross term in this case is defined as belows.
\begin{align*}
 & - \Unmyexpect{q(\lat,\mix|\varpar)}\sqr{\crossop{\varpar_a^{i}}{\varpar_b^{j}}{\log q(\vlat,\mix|\vvarpar)} } \\
= & - \Unmyexpect{q(\lat,\mix|\varpar)}\sqr{\crossop{\varpar_a^{i}}{\varpar_b^{j}} {\big( \log q(\vlat|\mix,  \{\vmu_c, \vS_c\}_{c=1}^{K} \big)} } \\
= & - \Unmyexpect{q(\lat,\mix|\varpar)}\Big[\mathbb{I}(w=b)\gradop{\varpar_a^{i}}{ \underbrace{\Big( \gradop{\varpar_b^{j}} {\sqr{ -\half   \vlat^T\vS_b\vlat + \vlat^T\vS_b\vmu_b  - \mixCompA(\vmu_{b},\vS_{b}  )}} \Big) }_{u(\boldsymbol{\lat},\boldsymbol{\mu}_b,\boldsymbol{\Sigma}_b)} } \Big] = 0
\end{align*}
It is obvious that the above expression is 0 since $\gradop{\varpar_a^{i}}u(\vlat,\vmu_b,\vSigma_b)=0$  when $a \neq b$.

Case 3: Finally, we will show that for each component $a$, cross terms (shown in green) between $\vmu_a$ and $\vS_a$ are zeros.

Let's denote $\mu^{i}_a$ be the $i$-th element of $\vmu_a$ and $ S^{jk}_a$ be the element of $\vS_a$ at position $(j,k)$.
Furthermore, $\ve_i$ denotes an one-hot vector where all entries are zeros except the $i$-th entry with value 1, and $\vI_{jk}$ denotes an one-hot matrix where all entries are zeros except the entry at position $(j,k)$ with value 1.
By the definition, the cross term is defined as belows.
\begin{align*}
& - \Unmyexpect{q(\lat,\mix|\varpar)}\sqr{\crossop{\mu^{i}_a}{S^{jk}_a}{\log q(\vlat,\mix|\vvarpar)} } \\
=& - \Unmyexpect{q(\lat,\mix|\varpar)}\sqr{\crossop{\mu^{i}_a}{S^{jk}_a} {\big( \log q(\vlat|\mix,  \{\vmu_c, \vS_c\}_{c=1}^{K} \big)} } \\
=& - \Unmyexpect{q(\lat,\mix|\varpar)}\sqr{\crossop{\mu^{i}_a}{S^{jk}_a} { \Big(  \mathbb{I}(w=a)\sqr{ -\half   \vlat^T\vS_a\vlat + \vlat^T\vS_a\vmu_a  - \mixCompA(\vmu_{a},\vS_{a} )}\Big) } } \\
=& - \Unmyexpect{q(\lat,\mix|\varpar)}\sqr{  \mathbb{I}(w=a)\sqr{  \ve_i^T\vI_{jk} \vlat -\ve_i^T\vI_{jk}\vmu_a}  } \\
=& - \Unmyexpect{q(\lat,\mix|\varpar)}\sqr{   \mathbb{I}(w=a)  \ve_i^T\vI_{jk} \vlat    }  +  \Unmyexpect{q(\lat,\mix|\varpar)}\sqr{   \mathbb{I}(w=a)\ve_i^T\vI_{jk}\vmu_a} \\
=& - \pi_a \ve_i^T\vI_{jk}\vmu_a + \pi_a \ve_i^T\vI_{jk}\vmu_a = 0 
\end{align*} where we use the following fact in the last step.
\begin{align*}
\Unmyexpect{q(\lat,\mix|\varpar)}\sqr{  \mathbb{I}(w=a) \vlat }& =\pi_a \vmu_a\\
\Unmyexpect{q(\lat,\mix|\varpar)}\sqr{  \mathbb{I}(w=a)  } & =\pi_a
\end{align*}
\end{proof}

\begin{lemma}
 The parameterization $\vvarpar=\{ \{ \vmu_c, \vS_c \}_{c=1}^{K}, \vvarpar_\mix \}$ is a BCN parameterization.
 \end{lemma}
 \begin{proof}
 Clearly, this parameterization satisfies Assumption 1 described in the main text.
 By Lemma \ref{lemma:mog_bc}, we know that this parameterization is a BC parameterization.
 Now, we will show that this parameterization also satisfies Assumption 3 in Appendix  \ref{app:BCN_CEF}.
 
 First note that $\vvarpar_\mix$ has only one block and it is the natural parameterization of exponential family distribution $q(\mix|\vvarpar_\mix)$, which implies that $\vvarpar_\mix$ is a BCN parameterization for  $q(\mix|\vvarpar_\mix)$.

 Note that given the rest blocks are known and conditioning on $\mix$, $q(\vlat|\mix,\vvarpar_\lat)$ can be re-expressed as follows in terms of block $\vmu_k$.
 \begin{align*}
&q( \vlat|\vmix,\vvarpar_\lat )
= \exp\Big( \sum_{c=1}^{K} \mathbb{I}(\mix=c)\sqr{ -\half  \vlat^T\vS_c\vlat + \vlat^T\vS_c\vmu_c} - A_\lat(\{\vmu_c, \vS_c\}_{c=1}^{K},\mix)\Big) \\
=&
\underbrace{\exp\Big( \sum_{c \neq k} \sqr{ \mathbb{I}(\mix=c)\sqr{ -\half  \vlat^T\vS_c\vlat + \vlat^T\vS_c\vmu_c} }  + \mathbb{I}(\mix=k)\sqr{ -\half  \vlat^T\vS_k\vlat } \Big)}_{  h_{\lat_{k_1}}(\mix,\bold{\lat},\bold{\varpar}_\lat^{[-k_1]}) } 
\exp \Big(  \myang{ \underbrace{ \mathbb{I}(\mix=k)\vS_k\vlat }_{\bold{\phi}_{\lat_{k_1}}(\mix,\bold{\lat},\bold{\varpar}_\lat^{[-k_1]})}, \underbrace{\vmu_k}_{\bold{\varpar}_\lat^{k_1}}}   - A_\lat(\{\vmu_c, \vS_c\}_{c=1}^{K},\mix) \Big) 
 \end{align*}
 
Similarly, for block $\vS_k$,  $q(\vlat|\vmix,\vvarpar_\lat)$ can be re-expressed as follows
\begin{align*}
&q( \vlat|\vmix,\vvarpar_\lat )\\
=&
\underbrace{\exp\Big( \sum_{c \neq k} \sqr{ \mathbb{I}(\mix=c)\sqr{ -\half  \vlat^T\vS_c\vlat + \vlat^T\vS_c\vmu_c} } \Big)}_{  h_{\lat_{k_2}}(\mix,\bold{\lat},\bold{\varpar}_\lat^{[-k_2]}) } 
\exp \Big(  \myang{ \underbrace{ \mathbb{I}(\mix=k) \sqr{ -\half  \vlat \vlat^T + \vmu_k\vlat^T  }    }_{\bold{\phi}_{\lat_{k_2}}(\mix,\bold{\lat},\bold{\varpar}_\lat^{[-k_2]})}, \underbrace{\vS_k}_{\bold{\varpar}_\lat^{k_2}}}   - A_\lat(\{\vmu_c, \vS_c\}_{c=1}^{K},\mix) \Big) 
\end{align*}

Since this parameterization satisfies Assumption 1 to 3,  this parameterization is a BCN parameterization.
 \end{proof}

We denote the Christoffel symbols of the first kind and the second kind for $\vmu_k$ as ${\Gamma_{a_{k_1},b_{k_1} c_{k_1}}}$ and ${\Gamma^{a_{k_1}}_{\ \ \ \ b_{k_1} c_{k_1}}}$  respectively.
\begin{lemma}
\label{lemma:mog_mean_chris}
For each component $k$, all entries of $\ {\Gamma^{a_{k_1}}_{\ \ \ \ b_{k_1} c_{k_1}}}$  for $\vmu_k$ are zeros.
\end{lemma}
\begin{proof}
The proof is very similar to the proof of Lemma \ref{lemma:gauss_gamma1}.
We will prove this by showing that  all entries of  ${\Gamma_{a_{k_1},b_{k_1}c_{k_1}}}$ are zeros.
For notation simplicity, we use $\Gamma_{a,bc}$ to denote ${\Gamma_{a_{k_1},b_{k_1}c_{k_1}}}$.
Let $\mu_k^a$ denote the $a$-th element of $\vmu_k$.

The following expression holds for any valid $a$, $b$, and $c$.
\begin{align*}
 \Gamma_{a,bc}
 &= \half \Unmyexpect{q(\lat,\mix|\varpar)}\sqr{  \gradop{\mu_k^b} \gradop{\mu_k^c} \gradop{\mu_k^a} A_\lat(\{\vmu_j, \vS_j\}_{j=1}^{K}, \mix) } \\
 &= \half \Unmyexpect{q(\lat,\mix|\varpar)}\sqr{ \mathbb{I}(\mix=k) \gradop{\mu_k^b} \gradop{\mu_k^c} \gradop{\mu_k^a} \mixCompA(\vmu_k, \vS_k) } \\
 &= \half  \Unmyexpect{q(\lat,\mix|\varpar)}\sqr{ \mathbb{I}(\mix=k)  \gradop{\mu_k^b} \gradop{\mu_k^c} \left(\ve_a^T \vS_k \vmu_k  \right) } \\
 &=  \half \Unmyexpect{q(\lat,\mix|\varpar)}\Big[ \mathbb{I}(\mix=k) \underbrace{\gradop{\mu_k^b}\left(\ve_a^T \vS_k \ve_c \right)}_{0} \Big] = 0
\end{align*}
 where in the last step we use the fact that 
$\vS_k$, $\ve_a$, and $\ve_c$ do not depend on $\vmu_k$.

\end{proof}

Similarly, we denote the Christoffel symbols  of the second kind for $\mathrm{vec}(\vS_k)$ as ${\Gamma^{a_{k_2}}_{\ \ \ \ b_{k_2} c_{k_2}}}$.

\begin{lemma}
\label{lemma:mog_cov_chris}
For each component $k$, the additional term for $\vS_k$ is
$- \vngrad_k^{[2]}   \vS_k^{-1}  \vngrad_k^{[2]}$
\end{lemma}
\begin{proof}
Recall that, in the Gaussian case $\gauss(\bar{\vmu}, \bar{\vS})$, the additional term for $\bar{\vS}$ is
$ \mathrm{Mat}( {\mbox{$\bar{\Gamma}^{a_2}_{\ \ b_2 c_2}$}} \ngrad^{b_2} \ngrad^{c_2}) =  \vngrad^{[2]}  \bar{\vS}^{-1}  \vngrad^{[2]}$, where
$\bar{\Gamma}^{a_2}_{\ \ b_2 c_2}$
 denotes the Christoffel symbols of the second kind for $\mathrm{vec}(\bar{\vS})$.

To prove the statement, we will show that the Christoffel symbols of the second kind for $\mathrm{vec}(\vS_k)$  is exactly the same as the Gaussian case, when $\bar{\vS}=\vS_k$.
In other words, when $\bar{\vS}=\vS_k$, we will show
 $\Gamma^{a_{k_2}}_{\ \ \ \ b_{k_2}c_{k_2}} = \bar{\Gamma}^{a_{2}}_{\ \ \ \ b_{2}c_{2}}$.

We denote the Christoffel symbols of the second kind for $\mathrm{vec}(\vS_k)$ using ${\Gamma^{a_{k_2}}_{\ \ \ \ b_{k_2} c_{k_2}}}$.
By definition, the Christoffel symbols of the second kind for $\mathrm{vec}(\vS_k)$ is defined as follows since $\vvarpar$ is a BC parameterization.
\begin{align*}
 \Gamma^{a_{k_2}}_{\ \ \ \ b_{k_2}c_{k_2}}= \fim^{a_{k_2}d_{k_2}}  \Gamma_{d_{k_2},b_{k_2}c_{k_2}}
\end{align*}
We will first show that $\Gamma_{d_{k_2},b_{k_2}c_{k_2}} = \pi_k \bar{\Gamma}_{d_2,b_2 c_2}$.

In the Gaussian case, by definition, we have
\begin{align*}
 \bar{\Gamma}_{d_2,b_2 c_2}
 = \half \Unmyexpect{q(\lat|\bar{\varpar})}\sqr{  \gradop{\bar{S}^b} \gradop{\bar{S}^c} \gradop{\bar{S}^d} A(\bar{\vmu},\bar{\vS}) } 
 = -\frac{1}{4}  \gradop{\bar{S}^b} \gradop{\bar{S}^c}  \gradop{\bar{S}^d}  \left( \log\left| \bar{\vS} \right| \right) 
\end{align*} where 
$A(\bar{\vmu},\bar{\vS})=\half \sqr{ \bar{\vmu}^T\bar{\vS}\bar{\vmu} - \log \left|\bar{\vS}/(2\pi)\right|}$ is the log partition function of the Gaussian distribution and 
 $\bar{S}^{d}$ denotes the $d$-th element of $\mathrm{vec}(\bar{\vS})$ in the Gaussian case.
 
Therefore, we have  the following result in the MOG case when $\vS_k=\bar{\vS}$.
\begin{align*}
 \Gamma_{d_{k_2},b_{k_2}c_{k_2}}
 &= \half \Unmyexpect{q(\lat,\mix|\varpar)}\sqr{  \gradop{S_k^b} \gradop{S_k^c} \gradop{S_k^d} A_\lat(\{\vmu_j, \vS_j\}_{j=1}^{K}, \mix) } \\
 &= \half \Unmyexpect{q(\lat,\mix|\varpar)}\sqr{ \mathbb{I}(\mix=k) \gradop{S_k^b} \gradop{S_k^c} \gradop{S_k^d} \mixCompA(\vmu_k, \vS_k) } \\
 &= \half  \Unmyexpect{q(\lat,\mix|\varpar)}\sqr{ \mathbb{I}(\mix=k)  \gradop{S_k^b} \gradop{S_k^c}  \gradop{S_k^d}  \left( -\half \log\left| \vS_k/(2\pi) \right| \right) } \\
 &= -\frac{\pi_k}{4}  \gradop{S_k^b} \gradop{S_k^c}  \gradop{S_k^d}  \left( \log\left| \vS_k \right| \right) \\
 &=  \pi_k   \bar{\Gamma}_{d_2,b_2 c_2}
\end{align*}
where 
 $S_k^a$ denotes the $a$-th element of $\mathrm{vec}(\vS_k)$ and $\Unmyexpect{q(\lat,\mix|\varpar)}\sqr{  \mathbb{I}(w=k)  }  =\pi_k$.

Let $\fim_{a_{k_2} d_{k_2}}$  denote 
the element at position $(a,d)$ of the sub-block matrix of the joint FIM for block  $\mathrm{vec}(\vS_k)$ in the MOG case.
Similarly, when $\vS_k=\bar{\vS}$,  we can show that $\fim_{a_{k_2} d_{k_2}} = \pi_k \bar{\fim}_{a_2 d_2}$, where
$\bar{\fim}_{a_2 d_2}$
denotes the element at position $(a,d)$ of the sub-block matrix of the FIM for block $\mathrm{vec}(\bar{\vS})$ in the Gaussian case.

Therefore, $\fim^{a_{k_2} d_{k_2}} = \pi_k^{-1} \bar{\fim}^{a_2 d_2}$ when $\bar{\vS}=\vS_k$.

Finally, when $\bar{\vS}=\vS_k$, we obtain the desired result since
\begin{align*}
 \Gamma^{a_{k_2}}_{\ \ \ \ b_{k_2}c_{k_2}} &= \fim^{a_{k_2} d_{k_2}}  \Gamma_{d_{k_2},b_{k_2}c_{k_2}} 
 = \left( \pi_k^{-1} \bar{\fim}^{a_2 d_2} \right) \left( \pi_k \bar{\Gamma}_{d_2,b_2 c_2}\right) 
 = \bar{\fim}^{a_2 d_2}  \bar{\Gamma}_{d_2,b_2 c_2} 
= \bar{\Gamma}^{a_2}_{\ \ b_2c_2}
\end{align*}
where $\bar{\Gamma}^{a_2}_{\ \ b_2 c_2}$
denotes the 
Christoffel symbols of the second kind for $\mathrm{vec}(\bar{\vS})$ in the Gaussian case.

\end{proof}

\subsection{Natural Gradients}
\label{app:mog_ng}
Recall that $\elbofinal(\vvarpar) =\Unmyexpect{q(\lat|\varpar)}\sqr{ \ell(\data,\vlat) -\log p(\vlat) + \log q(\vlat|\vvarpar) }$, where $q(\vlat|\vvarpar)= \int  q(\vlat,\mix|\vvarpar) d\mix$.

\citet{lin2019fast} propose to use
the  importance sampling technique so that the number of Monte Carlo gradient evaluations is independent of the number of mixing components $K$.

Note that $\vvarpar_\mix$ is the natural parameter of exponential family distribution $q(\mix|\vvarpar_\mix)$, we can obtain the natural gradient by computing the gradient w.r.t. the mean parameter as shown by 
\citet{lin2019fast}.
\begin{align*}
\ngrad_{\mix}  = \gradop{ \pi }{\elbofinal}.
\end{align*} where $\pi_c:=\Unmyexpect{q(\mix)}\sqr{ \mathbb{I}(\mix=c) }$, $\gradop{ \pi_c }{\elbofinal}$ denotes the $c$-th element of $\gradop{ \pi }{\elbofinal}$, and
the gradient $\gradop{ \pi_c }{\elbofinal}$ can be computed as below as suggested by \citet{lin2019fast}.
\begin{align*}
 \gradop{ \pi_c }{\elbofinal} = \Unmyexpect{q(\lat)}{\sqr{ (\delta_c-\delta_K) b(\vlat) } }
\end{align*} where $b(\vlat):=\ell(\data,\vlat) -\log p(\vlat) + \log q(\vlat|\varpar)$,
and 
$\delta_c:=\gauss(\vlat|\vmu_c,\vS_c)/ \sum_{k=1}^{K}\pi_k \gauss(\vlat|\vmu_k,\vS_k)$.

Recall that $\vvarpar_\mix$ is unconstrained in this case, there is no need to compute the addition term for $\vvarpar_\mix$.

Now, we discuss how to compute the natural gradients $\{ \vngrad_{c}^{[1]}, \vngrad_{c}^{[2]}\}_{c=1}^{K}$.
Since $\{ \vmu_c, \vS_c\}_{c=1}^{K}$ are BCN parameters, we can obtain the natural gradients by computing gradients w.r.t. its BC expectation parameter due to Theorem  \ref{thm:CEF_chris}.

Given the rest of blocks are known, 
the BC expectation parameter for block $\vmu_k$ is
\begin{align*}
\vvarmean_{k_1} & = 
\Unmyexpect{q(\mix,\lat)}\sqr{ \mathbb{I}(\mix=k)\rnd{ \vS_k\vlat }} 
 = \pi_k \vS_k\vmu_k 
\end{align*}

In this case, we know that $\gradop{ \mu_k}{\elbofinal}= \pi_k\vS_k\gradop{\varmean_{k_1}}{\elbofinal}$.
Therefore, the natural gradient w.r.t. $\vmu_k$ is $\vngrad_{k}^{[1]}= \gradop{ \varmean_{k_1} }{\elbofinal}= \pi_k^{-1}\vS_k^{-1}\gradop{\mu_k}{\elbofinal}= \pi_k^{-1}\vSigma_k \gradop{\mu_k}{\elbofinal}$, where 
the gradient $\gradop{ \mu_k }{\elbofinal}$ can be computed as belows as suggested by \citet{lin2019fast}.
\begin{align*}
 \gradop{ \mu_k }{\elbofinal}= \Unmyexpect{q(\lat)}{\sqr{ \pi_k \delta_k \nabla_\lat b(\vlat) } }
\end{align*}

Likewise, given the rest of blocks are known, 
the BC expectation parameter for block $\vS_k$ is
\begin{align*}
\vvarmean_{k_2}  = 
 \Unmyexpect{q(\mix,\lat)}\sqr{ \mathbb{I}(\mix=k)\rnd{ -\half \vlat \vlat^T + \vmu_k\vlat^T } } 
 = \frac{\pi_k}{2} \left( \vmu_k\vmu_k^T - \vS_k^{-1} \right)
\end{align*}

Therefore,  the natural gradient w.r.t. $\vS_{k}$ is $\vngrad_k^{[2]}=\gradop{\varmean_{k_2}}{\elbofinal}=-\frac{2}{\pi_k}\gradop{S_k^{-1} }{f}=-\frac{2}{\pi_k} \gradop{\Sigma_k}{f}$, where
where the gradient $\gradop{ \Sigma_k }{f}$ can be computed as belows as suggested by \citet{lin2019fast}.
\begin{align*}
 \gradop{ \Sigma_k }{\elbofinal}= \half \Unmyexpect{q(\lat)}{\sqr{ \pi_k \delta_k \nabla_\lat^2 b(\vlat) } }
\end{align*}
Alternatively, we can use the re-parametrization trick to compute the gradient as below.
\begin{align*}
 \gradop{ \Sigma_k }{\elbofinal}= \half \Unmyexpect{q(\lat)}{\sqr{ \pi_k \delta_k \vS_k (\vlat-\vmu_k) \nabla_\lat^T b(\vlat) } }
\end{align*}

By Lemma 
\ref{lemma:mog_mean_chris} and
\ref{lemma:mog_cov_chris},
the proposed update induced by our rule is

\begin{align}
 \log(\pi_c/\pi_K) & \leftarrow \log(\pi_c/\pi_K) - \stepsize \Unmyexpect{q(\lat)}{\sqr{ (\delta_c-\delta_K) b(\vlat) } }\nonumber \\
\vmu_c & \leftarrow  \vmu_c - \stepsize \vS_c^{-1} \Unmyexpect{q(\lat)}{\sqr{ \delta_c \nabla_\lat b(\vlat) } } \color{red}{+ \mathbf{0} }\nonumber \\
\vS_c & \leftarrow  \vS_c - \stepsize \hat{\vG}_c \color{red}{+ \frac{\stepsize^2}{2}\hat{\vG}_c \left( \vS_c \right)^{-1} \hat{\vG}_c } \label{eq:mog_iblr}
\end{align} where we do not compute the additional term for $\vvarpar_\mix$ since $\vvarpar_\mix$ is unconstrained,
$\delta_c:=\gauss(\vlat|\vmu_c,\vS_c)/ \sum_{k=1}^{K}\pi_k \gauss(\vlat|\vmu_k,\vS_k)$,
$b(\vlat):=\ell(\data,\vlat) -\log p(\vlat) + \log q(\vlat|\vvarpar)$ and
$\hat{\vG}_c$ can be computed as below.

Note that $b(\vlat)$ can be the logarithm of an unnormalized target function as such $b(\vlat)=\bar{\ell}(\vlat) + \mathrm{Constant} + \log q(\vlat|\vvarpar) $. Recall that $\ell(\data,\vlat) -\log p(\vlat) = \bar{\ell}(\vlat) + \mathrm{Constant}$.
\citet{lin2019fast} suggest  using the Hessian trick to compute $\hat{\vG}_c$ as shown in \eqref{eq:mog_cov_hess}.
We can also use the re-parameterization trick to compute $\hat{\vG}_c$ as shown in \eqref{eq:mog_cov_rep}.
\begin{align}
 \hat{\vG}_c&=-\Unmyexpect{q(\lat)}{\sqr{ \delta_c \vS_c (\vlat-\vmu_c) \nabla_\lat^T b(\vlat) } } = -\Unmyexpect{q(\lat)}{\sqr{ \delta_c \vS_c (\vlat-\vmu_c) \nabla_\lat^T \bar{\ell}(\vlat) } } - \Unmyexpect{q(\lat)}{\sqr{ \delta_c \nabla_\lat^2 \log q(\vlat|\vvarpar) } }  \label{eq:mog_cov_rep} \\
 &=-\Unmyexpect{q(\lat)}{\sqr{ \delta_c \nabla_\lat^2 b(\vlat) } } = -\Unmyexpect{q(\lat)}{\sqr{ \delta_c \nabla_\lat^2 \bar{\ell}(\vlat) } } - \Unmyexpect{q(\lat)}{\sqr{ \delta_c \nabla_\lat^2 \log q(\vlat|\vvarpar) } } . \label{eq:mog_cov_hess}
\end{align}
We use the MC approximation to compute $\hat{\vG}_c$ as below.
\begin{align*}
 \hat{\vG}_c& \approx  - \delta_c \Big( \frac{\bar{\vS}_c +\bar{\vS}_c^T}{2} + \nabla_\lat^2  \log q(\vlat|\vvarpar) \Big)  & \text{ referred to as ``-rep'' } \\
 \hat{\vG}_c&\approx -\delta_c \Big( \nabla_\lat^2 \bar{\ell}(\vlat)  + \nabla_\lat^2  \log q(\vlat|\vvarpar) \Big)   &\text{ referred to as ``-hess'' }
\end{align*}  where
$\vlat  \sim q(\vlat|\vvarpar)$, 
$\bar{\vS}_c:= \vS_c (\vlat-\vmu_c) \nabla_\lat^T \bar{\ell}(\vlat) $
and $\nabla_\lat^2  \log q(\vlat|\vvarpar)$ can be manually coded or computed by Auto-Diff.

Recall that when $q(\vlat|\vvarpar)$ is Gaussian, $-\Unmyexpect{q(\lat)}\sqr{\nabla_\lat^2  \log q(\vlat|\vvarpar)} = \vSigma^{-1}$, which is positive definite.
VOGN is proposed to approximate $\Unmyexpect{q(\lat)}{\sqr{ \nabla_\lat^2 \bar{\ell}(\vlat) } } $ by a positive definite matrix when $q(\vlat|\vvarpar)$ is Gaussian.
In MOG cases, $-\Unmyexpect{q(\lat)}\sqr{\nabla_\lat^2  \log q(\vlat|\vvarpar)} $ is no longer a positive definite matrix.
VOGN does not guarantee that the update for $\vS_c$ stays in the constraint set.
Furthermore, directly approximating
$-\hat{\vG}_c$ by naively extending the idea of VOGN does not give a good posterior approximation. Unlike VOGN, our update satisfies the constraint without the loss of the approximation accuracy for both Gaussian and MOG cases.

\section{Example: Skew Gaussian Approximation}
\label{app:skewg_case}
We consider the skew Gaussian approximation proposed by \citet{lin2019fast}. The joint distribution is given below.
\begin{align*}
q(\vlat,\mix|\valpha,\vmu,\vSigma) & =  q(\vlat|\mix,\valpha,\vmu,\vSigma) \gauss(\mix|0,1) \\
q(\vlat|\mix,\valpha,\vmu,\vSigma) &= \gauss(\vlat|\vmu +  |\mix|\valpha, \vSigma) \\
&=  \exp(  \crl{ \mathrm{Tr} \left( -\half\vSigma^{-1} \vlat\vlat^T \right) + |\mix| \valpha^T \vSigma^{-1} \vlat + \vmu^T \vSigma^{-1} \vlat- \half(  (\vmu+|\mix|\valpha)^T \vSigma^{-1}(\vmu+|\mix|\valpha)  +   \log \left|2\pi\vSigma\right|  } )
\end{align*}

We consider the parameterization $\vvarpar=\{ \begin{bmatrix} \vmu \\ \valpha \end{bmatrix}, \vS\}$, where $\vS=\vSigma^{-1}$,
$\vvarpar^{[1]}=\begin{bmatrix} \vmu \\ \valpha \end{bmatrix}$, and $\vvarpar^{[2]}=\vS$. The open-set constraint is $\vvarpar\in \mathbb{R}^{2d} \times \mathbb{S}^{d \times d}_{++}$.
Under this parameterization, the distribution $q(\vlat|\mix)$ can be re-expressed as below.
\begin{align*}
q(\vlat|\mix, \vvarpar)
&=  \exp\crl{ \mathrm{Tr} \left( -\half\vS \vlat\vlat^T \right) + \vlat^T  \vS  \left(\vQ(\mix)\right)^T \vvarpar^{[1]} -  A_\lat(\vvarpar,\mix) } 
\end{align*} where $\vQ(\mix):=\begin{bmatrix} \vI_d \\ \left|\mix\right|\vI_d \end{bmatrix}$ is a $2d$-by-$d$ matrix and $A_\lat(\vvarpar,\mix) = \half\sqr{ \begin{bmatrix} \vmu^T & \valpha^T \end{bmatrix}  \vQ(\mix)\vS \left(\vQ(\mix)\right)^T  \begin{bmatrix} \vmu \\ \valpha \end{bmatrix} -   \log \left|\vS/(2\pi) \right| }$.

\begin{lemma}
\label{lemma:skewg_bc}
 The joint FIM is block diagonal with two blocks under this parameterization.
 \begin{align*}
 \vfim = \begin{bmatrix}
 \vfim^{[1]} & {\color{red} \mathbf{0}} \\
 {\color{red} \mathbf{0}} & \vfim^{[2]}
         \end{bmatrix}
 \end{align*}

 Therefore, this parameterization is a BC parameterization.
\end{lemma}
\begin{proof}
We will prove this lemma by showing that all cross terms shown in red are zeros.

Let's denote $\varpar^{a_1}$ be the $a$-th element of $\vvarpar^{[1]}$ and $ S^{b c}$ be the element  of $\vS$ at position $(b,c)$.
Furthermore, $\ve_a$ denotes an one-hot vector where all entries are zeros except the $a$-th entry with value 1, and $\vI_{bc}$ denotes an one-hot matrix where all entries are zeros except the entry at position $(b,c)$ with value 1.

By definition, the cross term is defined as belows.
\begin{align*}
& - \Unmyexpect{q(\lat,\mix|\varpar)}\sqr{\crossop{\varpar^{a_1}}{S^{bc}}{\log q(\vlat,\mix|\vvarpar)} } \\
=& - \Unmyexpect{q(\lat,\mix|\varpar)}\sqr{ \vlat^T  \vI_{bc}  \left(\vQ(\mix)\right)^T \ve_a - \left( \vvarpar^{[1]} \right)^T \vQ(\mix)\vI_{bc} \left(\vQ(\mix)\right)^T  \ve_a } \\
=& - \Unmyexpect{q(\mix)}\sqr{ \Unmyexpect{q(\lat|\mix,\varpar)}\sqr{ \vlat^T  \vI_{bc}  \left(\vQ(\mix)\right)^T \ve_a - \left( \vvarpar^{[1]} \right)^T \vQ(\mix)\vI_{bc} \left(\vQ(\mix)\right)^T  \ve_a } } \\
=& - \Unmyexpect{q(\mix)}\sqr{ \Unmyexpect{q(\lat|\mix,\varpar)}\sqr{ \vlat^T  \vI_{bc}  \left(\vQ(\mix)\right)^T \ve_a} - \left( \vvarpar^{[1]} \right)^T \vQ(\mix)\vI_{bc} \left(\vQ(\mix)\right)^T  \ve_a  } \\
=& - \Unmyexpect{q(\mix)}\sqr{   \left( \vvarpar^{[1]}\right)^T \vQ(\mix) \vI_{bc}  \left(\vQ(\mix)\right)^T \ve_a - \left( \vvarpar^{[1]} \right)^T \vQ(\mix)\vI_{bc} \left(\vQ(\mix)\right)^T  \ve_a  } = 0 
\end{align*} where  we use the following expression in the last step.
\begin{align*}
\Unmyexpect{q(\lat|\mix,\varpar)} \sqr{ \vlat } =\left| \mix \right| \valpha + \vmu = \left(\vQ(\mix)\right)^T \vvarpar^{[1]}
\end{align*}
\end{proof}

Note that another parameterization $\{\vmu,\valpha,\vS\}$ is { \color{red} not} a BC parameterization since the joint FIM is not block-diagonal under this parameterization.

\begin{lemma}
 Parameterization $\vvarpar$ is a BCN parameterization.
\end{lemma}
\begin{proof}
 Clearly, this parameterization satisfies Assumption 1 described in the main text.
 By Lemma \ref{lemma:skewg_bc}, we know that this parameterization is a BC parameterization.
 Now, we will show that this parameterization also satisfies Assumption 3 in Appendix  \ref{app:BCN_CEF}.
 
 Note that given the rest blocks are known and conditioning on $\mix$, $q(\vlat|\mix,\vvarpar)$ can be re-expressed as follows in terms of block $\vvarpar^{[1]}$.
 \begin{align*}
q(\vlat|\mix, \vvarpar)
&=  \exp\crl{ \mathrm{Tr} \left( -\half\vS \vlat\vlat^T \right) + \vlat^T  \vS  \left(\vQ(\mix)\right)^T \vvarpar^{[1]} -  A_\lat(\vvarpar,\mix) } \\
&= \underbrace{ \exp\crl{ \mathrm{Tr} \left( -\half\vS \vlat\vlat^T \right)} }_{  h_{1}(\mix,\bold{\lat},\bold{\varpar}^{[-1]}) }
 \exp \Big[   \myang{  \underbrace{    \vQ(\mix) \vS \vlat }_{ \bold{\phi}_{1}(\mix,\bold{\lat},\bold{\varpar}^{[-1]}) } , \vvarpar^{[1]} } -  A_\lat(\vvarpar,\mix) \Big] 
 \end{align*}

Similarly, for block $\vS$,  $q(\vlat|\vmix,\vvarpar)$ can be re-expressed as follows
\begin{align*}
q(\vlat|\mix, \vvarpar) 
&= \underbrace{ 1 }_{  h_{2}(\mix,\bold{\lat},\bold{\varpar}^{[-2]}) }
 \exp \Big[   \myang{  \underbrace{ -\half \vlat\vlat^T + \vlat \left( \vvarpar^{[1]}\right)^T \vQ(\mix)  }_{ \bold{\phi}_{2}(\mix,\bold{\lat},\bold{\varpar}^{[-2]}) } , \vS  } -  A_\lat(\vvarpar,\mix) \Big] 
\end{align*}

Since this parameterization satisfies Assumption 1 to 3,  this parameterization is a BCN parameterization.
\end{proof}

We denote the Christoffel symbols of the first kind and the second kind for $\vvarpar^{[1]}$ as ${\Gamma_{a_1,b_1 c_1}}$ and ${\Gamma^{a_1}_{\ \ \ \ b_1 c_1}}$  respectively.
\begin{lemma}
All entries of $\ {\Gamma^{a_1}_{\ \ \ \ b_1 c_1}}$  for $\vvarpar^{[1]}$ are zeros.
\end{lemma}
\begin{proof}
We will prove this by showing that  all entries of $\ {\Gamma^{a_1}_{\ \ b_1 c_1}}$  are zeros.
Let $\varpar^{a_1}$ denote the $a$-th element of $\vvarpar^{[1]}$.

The following expression holds for any valid $a$, $b$, and $c$.
\begin{align*}
 \Gamma_{a_1,b_1 c_1}
 &= \half \Unmyexpect{q(\lat,\mix|\varpar)}\sqr{  \gradop{\varpar^{b_1}} \gradop{\varpar^{c_1}} \gradop{\varpar^{a_1}} A_\lat(\vvarpar, \mix) } \\
 &= \half  \Unmyexpect{q(\lat,\mix|\varpar)}\sqr{  \gradop{\varpar^{b_1}} \gradop{\varpar^{c_1}} \left(\left(\ve_a\right)^T  \vQ(\mix)\vS \left(\vQ(\mix)\right)^T \vvarpar^{[1]}  \right) } \\
 &=  \half \Unmyexpect{q(\lat,\mix|\varpar)}\sqr{\gradop{\varpar^{b_1}} \left(\ve_a^T  \vQ(\mix)\vS \left(\vQ(\mix)\right)^T \ve_c \right) } = 0
\end{align*}
 where in the last step we use the fact that 
$\vS$ , $\vQ(\mix)$, $\ve_a$, and $\ve_c$ do not depend on $\vvarpar^{[1]}$.

\end{proof}

We denote the Christoffel symbols of the second kind for $\mathrm{vec}(\vS)$ as  ${\Gamma^{a_2}_{\ \ \ \ b_2 c_2}}$.

\begin{lemma}
The additional  term for $\vS$ is
$- \vngrad^{[2]}   \vS^{-1}  \vngrad^{[2]}$
\end{lemma}
\begin{proof}
Recall that, in the Gaussian case $\gauss(\bar{\vmu}, \bar{\vS})$, the additional term for $\bar{\vS}$ is
$ \mathrm{Mat}( {\mbox{$\bar{\Gamma}^{a_2}_{\ \ b_2 c_2}$}} \ngrad^{b_2} \ngrad^{c_2} ) =  \vngrad^{[2]} \bar{\vS}^{-1}  \vngrad^{[2]}$, where
${\mbox{$\bar{\Gamma}^{a_2}_{\ \ b_2 c_2}$}}$
 denotes the Christoffel symbols of the second kind for $\mathrm{vec}(\bar{\vS})$.

To prove the statement, we will show that the Christoffel symbols of the second kind for $\mathrm{vec}(\vS)$  is exactly the same as the Gaussian case, when $\bar{\vS}=\vS$.

We denote the Christoffel symbols of the second kind for $\mathrm{vec}(\vS)$ as ${\Gamma^{a_2}_{\ \ b_2 c_2}}$.
By definition, the Christoffel symbols of the second kind for $\mathrm{vec}(\vS)$ is defined as follows.
\begin{align*}
 \Gamma^{a_2}_{\ \ b_2 c_2}= \fim^{a_2 d_2}  \Gamma_{d_2,b_2 c_2}
\end{align*}

We will show that $\Gamma_{a_2,b_2 c_2} =  \bar{\Gamma}_{a_2,b_2 c_2}$.

In the Gaussian case, we have
\begin{align*}
 \bar{\Gamma}_{d_2,b_2 c_2}
 &= -\frac{1}{4}  \gradop{\bar{S}^{b}} \gradop{\bar{S}^{c}}  \gradop{\bar{S}^{d} }  \left( \log\left|\bar{\vS} \right| \right) 
\end{align*} where 
$A(\bar{\vmu},\bar{\vS})=\half \sqr{\bar{\vmu}^T\bar{\vS}\bar{\vmu} - \log \left|\bar{\vS}/(2\pi)\right|}$ is the log partition function of the Gaussian distribution and 
 $\bar{S}^{a}$ is the $a$-th element of $\mathrm{vec}(\bar{\vS})$ in the Gaussian case.

Therefore, we have  the following result  when $\bar{\vS}=\vS$.
\begin{align*}
 \Gamma_{d_2,b_2 c_2}
 = \half \Unmyexpect{q(\lat,\mix|\varpar)}\sqr{  \gradop{S^{b}} \gradop{S^{c}} \gradop{S^{d}} A_\lat(\vvarpar, \mix) } 
 =  -\frac{1}{4}  \gradop{S^{b}} \gradop{S^c} \gradop{S^d} \log\left| \vS \right|  
 = \bar{\Gamma}_{d_2,b_2 c_2}
\end{align*}
where 
 $S^{a}$ denotes the $a$-th element of $\mathrm{vec}(\vS)$.

Let $\fim_{a_2 d_2}$  denote
the element at position $(a,d)$ of the sub-block matrix of the joint FIM for $\mathrm{vec}(\vS)$.
Similarly, we can show that $\fim_{a_2 d_2} = \bar{\fim}_{a_2 d_2}$, where
$\bar{\fim}_{a_2 d_2}$
denotes the element at position $(a,d)$ of the FIM for  $\mathrm{vec}(\bar{\vS})$ in the Gaussian case.
Therefore, $\fim^{a_2 d_2} =  \bar{\fim}^{a_2 d_2}$.

Finally,
 when $\bar{\vS}=\vS$,
  we obtain the desired result since
\begin{align*}
 \Gamma^{a_2}_{\ \ b_2 c_2} &= \fim^{a_2 d_2}  \Gamma_{d_2,b_2 c_2} 
= \bar{\fim}^{a_2 d_2}  \bar{\Gamma}_{d_2,b_2 c_2}
= \bar{\Gamma}^{a_2}_{\ \ b_2 c_2}
\end{align*}
where $\bar{\Gamma}^{a_2}_{\ \ b_2 c_2}$
denotes  the 
Christoffel symbols of the second kind for $\mathrm{vec}(\bar{\vS})$ in the Gaussian case.
\end{proof}

Using these lemmas, the proposed update induced by our rule is
\begin{align*}
\begin{bmatrix} \vmu \\ \valpha \end{bmatrix} & \leftarrow  \begin{bmatrix} \vmu \\ \valpha \end{bmatrix} - \stepsize \vngrad^{[1]} \color{red}{ + \mathbf{0} }\\
\vS & \leftarrow  \vS - \stepsize \vngrad^{[2]} \color{red}{+ \frac{\stepsize^2}{2}\vngrad^{[2]} \vS^{-1} \vngrad^{[2]} } 
\end{align*} where $\vngrad^{[1]}$ and $\vngrad^{[2]}$ are natural gradients.

Similarly, it can be shown that the above update satisfies the underlying constraints.

\subsection{Natural Gradients}
Now, we discuss how to compute the natural gradients.
Since the parameterization is a BCN parameterization,
gradients w.r.t. BC expectation parameters are natural gradients for BCN parameters due to Theorem \ref{thm:CEF_chris}.

Recall that $\vvarpar^{[1]}=\begin{bmatrix} \vmu \\ \valpha \end{bmatrix}$.
Let $\vvarmean_{[1]}=\begin{bmatrix} \vvarmean_\mu \\ \vvarmean_\alpha \end{bmatrix}$ denote the BC expectation  parameter for $\vvarpar^{[1]}$.
Given 
$\vS$ is known, the BC expectation parameter is
\begin{align*}
\begin{bmatrix} \vvarmean_\mu \\ \vvarmean_\alpha \end{bmatrix} &= 
\Unmyexpect{q(\mix,\lat)}\sqr{ \vQ(\mix) \vS \vlat }  \\
&= \Unmyexpect{q(\mix)}\sqr{ \vQ(\mix) \vS \left( \vQ(\mix) \right)^{T} \vvarpar^{[1]} }  \\
&= \Unmyexpect{q(\mix)}\sqr{ \begin{bmatrix} \vS & \left|\mix\right| \vS \\ \left|\mix\right| \vS & \mix^2 \vS\end{bmatrix} \begin{bmatrix} \vmu \\ \valpha \end{bmatrix} }  \\
&=\begin{bmatrix} \vS & c \vS \\ c \vS &  \vS \end{bmatrix} \begin{bmatrix} \vmu \\ \valpha \end{bmatrix} \\
&= \begin{bmatrix} \vS \vmu + c \vS \valpha  \\  c \vS \vmu +  \vS \valpha  \end{bmatrix}
\end{align*} where $c=\Unmyexpect{q(\mix)}\sqr{ \left| \mix\right| }=\sqrt{\frac{2}{\pi}}$.

Since $\vS=\vSigma^{-1}$, we have the following expressions.
\begin{align*}
\vmu = \frac{1}{1-c^2} \vSigma\rnd{ \vvarmean_\mu - c \vvarmean_\alpha }, \quad 
\valpha = \frac{1}{1-c^2} \vSigma\rnd{ \vvarmean_\alpha - c \vvarmean_\mu } 
\end{align*}
By the chain rule, we have
\begin{align*}
 \gradop{ \varmean_\mu }{\elbofinal} = \vSigma\rnd{ \frac{1}{1-c^2}\gradop{\mu}{\elbofinal} - \frac{c}{1-c^2} \gradop{\alpha}{\elbofinal} }, \quad
 \gradop{ \varmean_\alpha }{\elbofinal} = \vSigma\rnd{ \frac{1}{1-c^2}\gradop{\alpha}{\elbofinal} - \frac{c}{1-c^2} \gradop{\mu}{\elbofinal} } 
\end{align*}
Therefore, the natural gradient w.r.t. $\vvarpar^{[1]}=\begin{bmatrix} \vmu \\ \valpha \end{bmatrix}$ is
$\vngrad^{[1]}=
\begin{bmatrix}
 \gradop{ \varmean_\mu }{\elbofinal} \\
 \gradop{ \varmean_\alpha }{\elbofinal}
\end{bmatrix} $
where 
the gradient $\gradop{ \mu }{\elbofinal}$ and $\gradop{ \alpha }{\elbofinal}$ can be computed as suggested by \citet{lin2019fast}.

Likewise, 
 the BC expectation  parameter for block $\vS$ is
\begin{align*}
\vvarmean_{[2]}  = 
\Unmyexpect{q(\mix,\lat)} \sqr{ -\half \vlat \vlat^T + \vlat \left(\vvarpar^{[1]}\right)^T  \vQ(\mix)  }
=  -\half \vS^{-1} + \Unmyexpect{q(\mix)} \sqr{  \half \left( \vQ(\mix) \right)^T \vvarpar^{[1]} \left( \vvarpar^{[1]}\right)^T \vQ(\mix)    }
\end{align*}
Since
 $\vvarpar^{[1]}$ is known, $\Unmyexpect{q(\mix)} \sqr{  \half \left( \vQ(\mix) \right)^T \vvarpar^{[1]} \left( \vvarpar^{[1]}\right)^T \vQ(\mix) }$ does not depend on $\vS$.
Therefore,  the natural gradient w.r.t. $\vS$ is $\vngrad^{[2]}=\gradop{\varmean_{[2]}}{\elbofinal}=-2\gradop{ S^{-1} }{\elbofinal}=-2 \gradop{\Sigma}{\elbofinal}$,
where we compute  the gradient $\gradop{ \Sigma }{\elbofinal}$ as suggested by \citet{lin2019fast}.

\section{More Results}
\label{app:more_plots}

\vspace{-0.25cm}
\begin{figure}[H]
	\centering
	\hspace*{-1.5cm}
\subfigure[]{
\label{fig:ma}
  \includegraphics[width=0.34\linewidth]{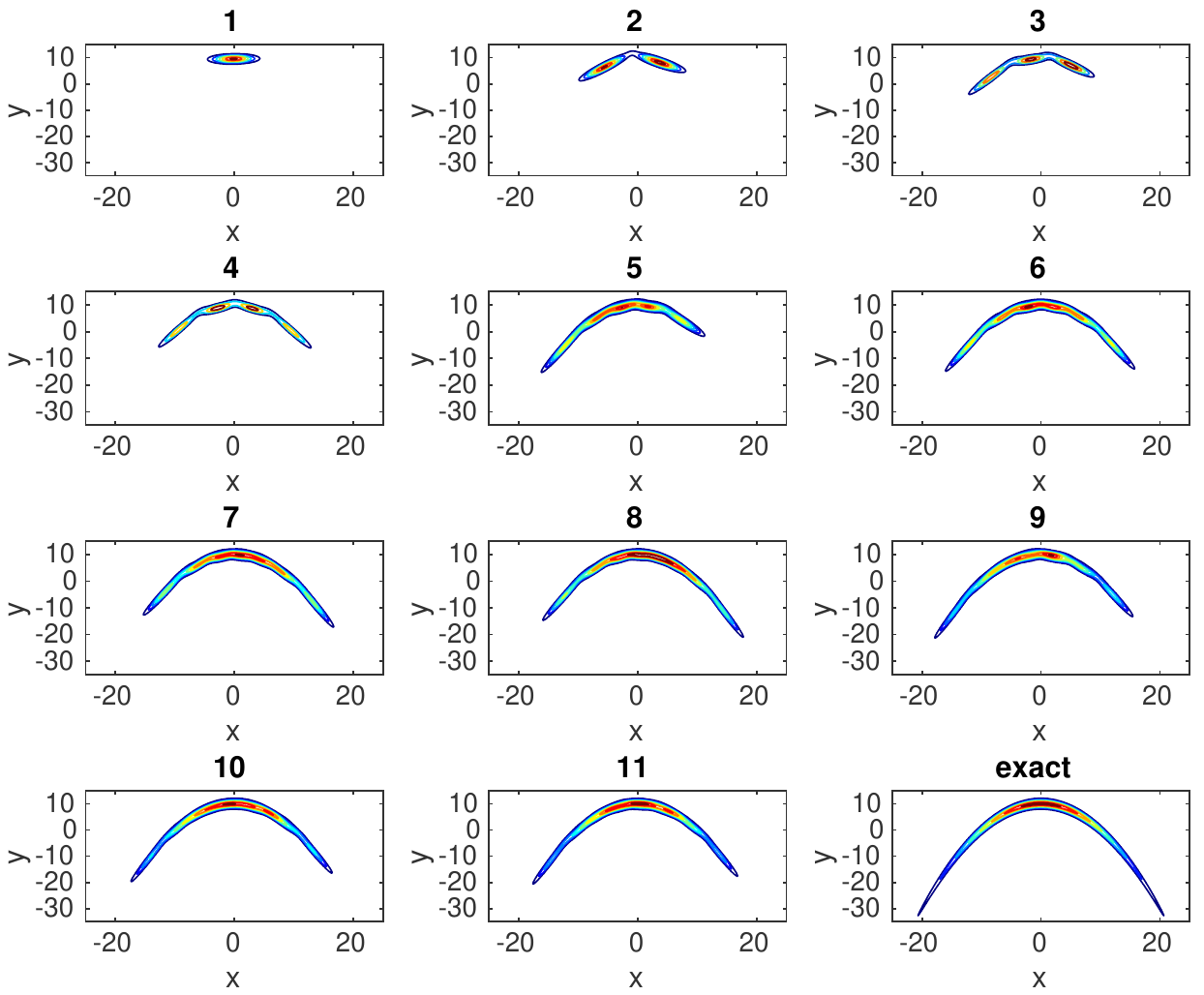}
  }
	\hspace*{-0.2cm}
\subfigure[]{
\label{fig:mb}
  \includegraphics[width=0.34\linewidth]{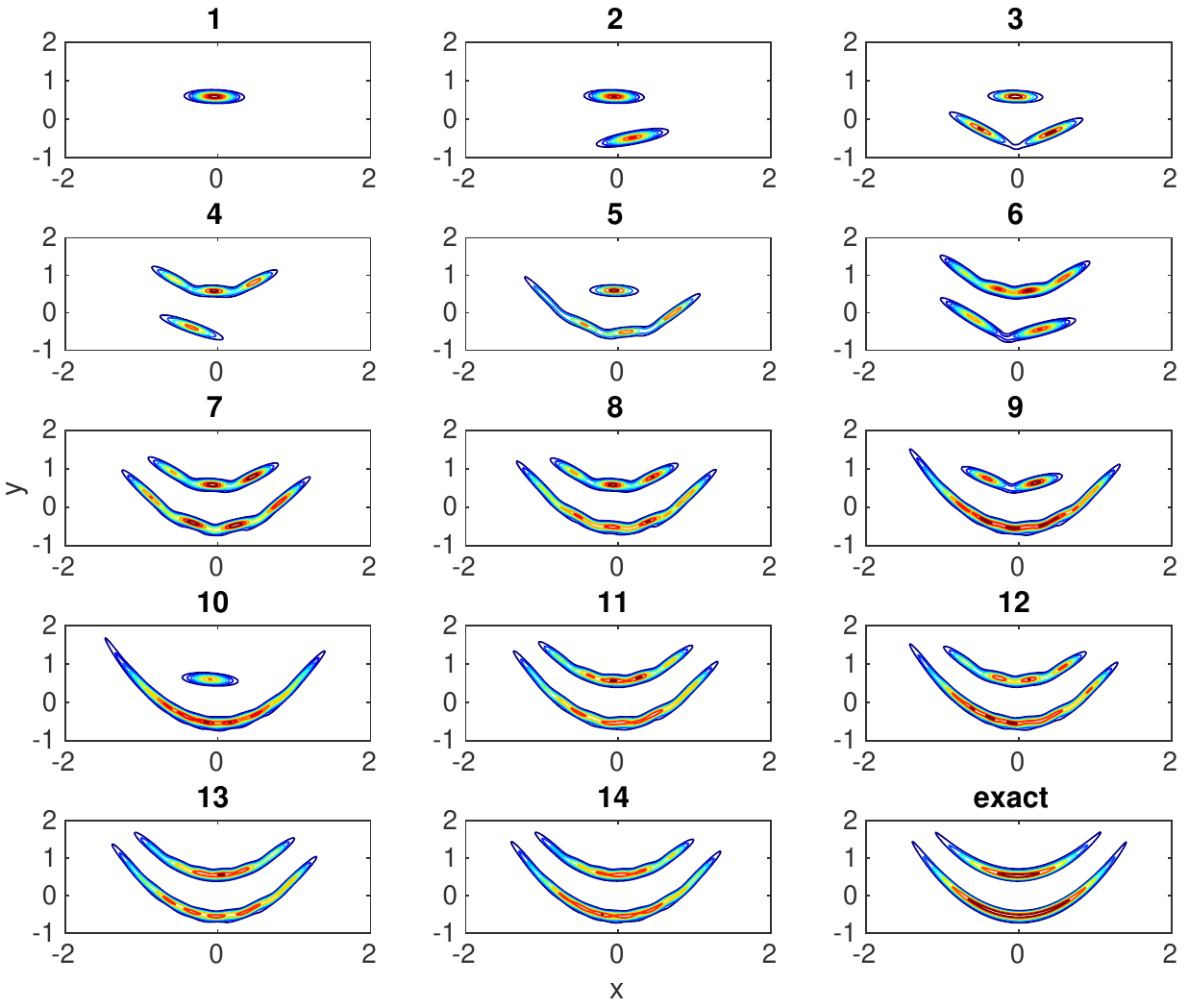}
  }
	\hspace*{-0.2cm}
\subfigure[]{
\label{fig:mc}
  \includegraphics[width=0.34\linewidth]{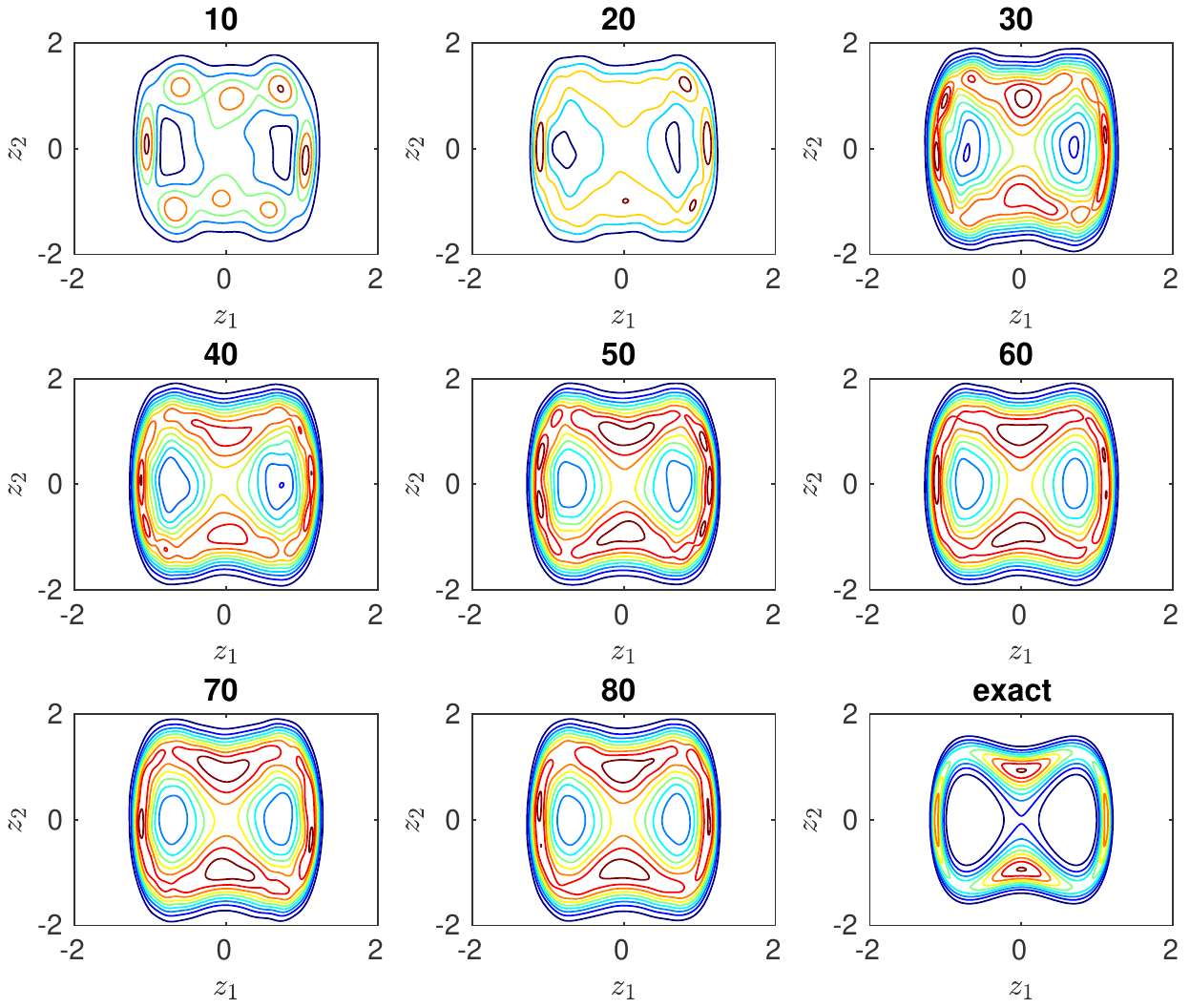}
  }
	\hspace*{-1.5cm}
\vspace{-0.5cm}
   \caption{ 
   The leftmost figure is MOG approximations for the banana distribution mentioned at Section \ref{sec:syn_examples}, where the number indicates the number of components used in the approximations.
  The middle figure is a complete version of MOG approximations for the double banana distribution (the rightmost plot in Figure \ref{figure:toy_examples}), where the number indicates the number of components used in the approximations.
   The rightmost figure is MOG approximations for the  posterior $p(\vlat|y=1)$ of a BNN with a Gaussian prior $p(\vlat)=\gauss(\vlat|\mathbf{0},\vI)$ and a NN likelihood $p(y|\vlat)=\gauss(y| 3\lat_1^2 (\lat_1^2-1) +\lat_2^2 , 0.5^2)$, where the number indicates the number of components used in the approximations.
   }
	\label{figure:banana_plots}
\end{figure}
\vspace{-0.85cm}

\begin{figure}[H]
	\centering
	\hspace*{-1.5cm}
\subfigure[]{
\label{fig:mma}
  \includegraphics[width=0.34\linewidth]{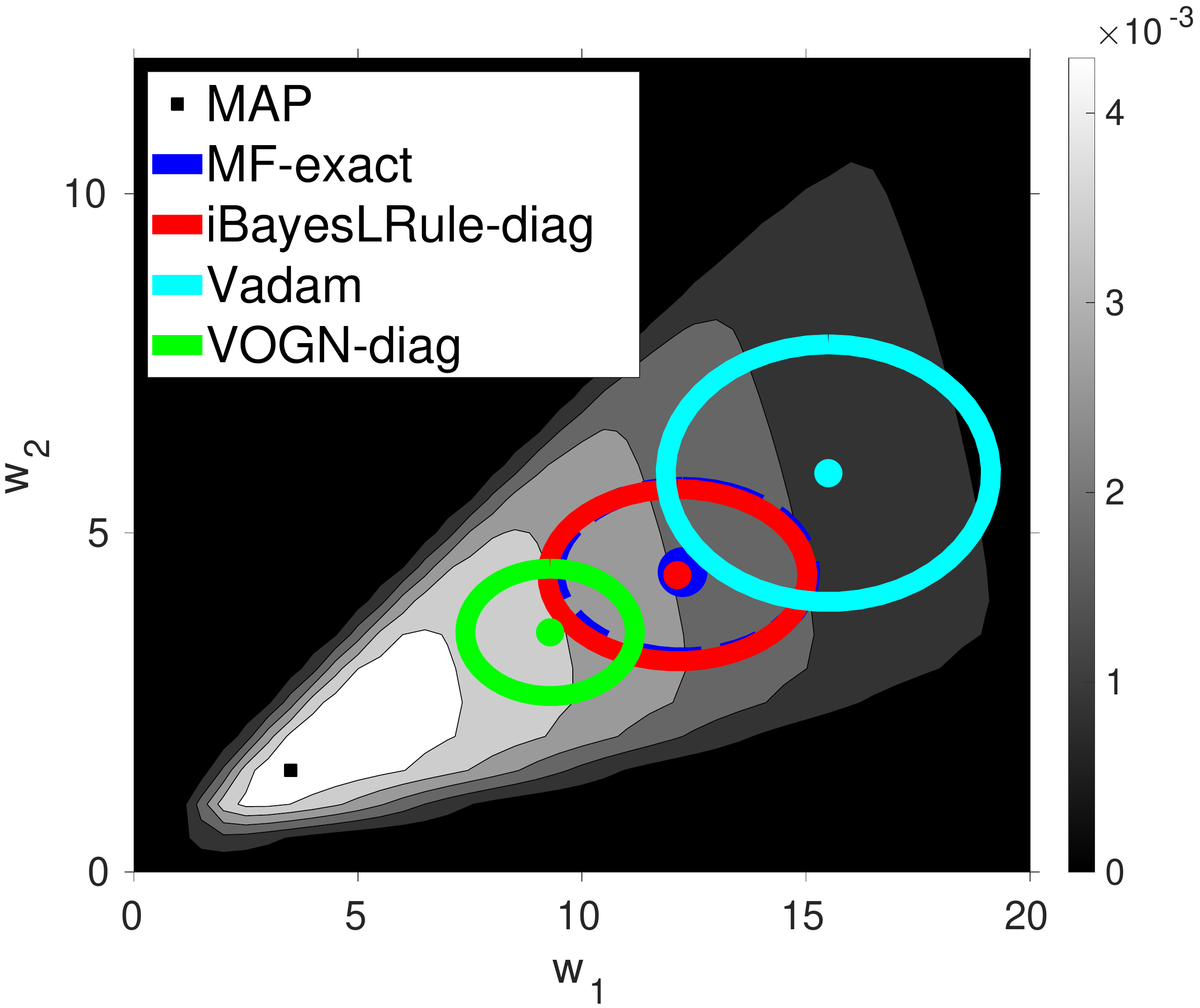}
  }
	\hspace*{0.5cm}
\subfigure[]{
\label{fig:mmb}
  \includegraphics[width=0.34\linewidth]{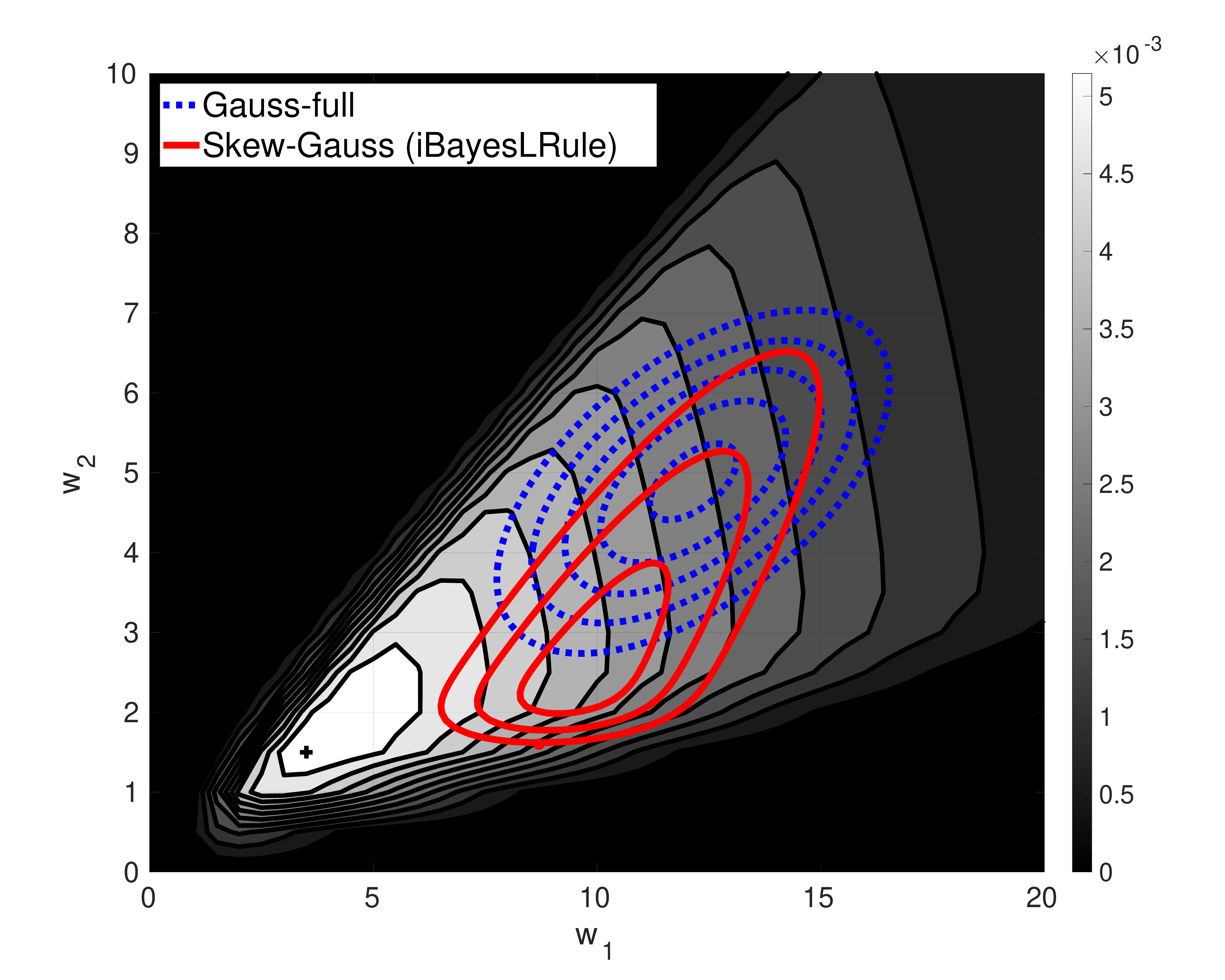}
  }
	\hspace*{-1.5cm}
\vspace{-0.5cm}
   \caption{
   The leftmost plot is mean-field Gaussian approximations for the toy Bayesian logistic regression example considered at Section \ref{sec:syn_examples}, where Vadam is proposed by \citet{khan18a}.
   The rightmost plot is a skew-Gaussian approximation with full covariance structure for the same example.
   }
	\label{figure:blr_plots}
\end{figure}
\vspace{-0.95cm}

\begin{figure}[H]
	\centering
	\hspace*{-1.5cm}
\subfigure[]{
\label{fig:unma}
  \includegraphics[width=0.34\linewidth]{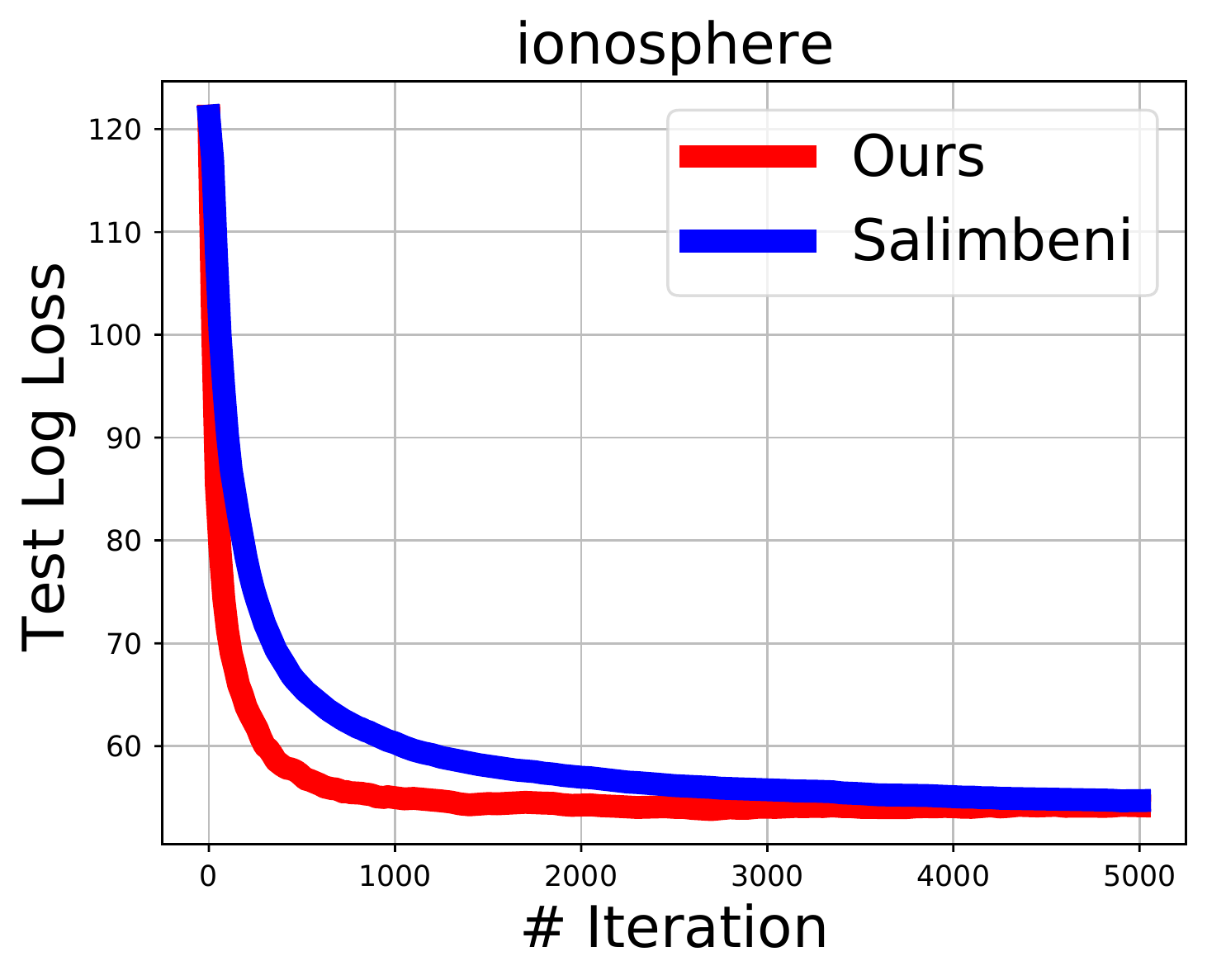}
  }
	\hspace*{0.5cm}
\subfigure[]{
\label{fig:unmb}
  \includegraphics[width=0.34\linewidth]{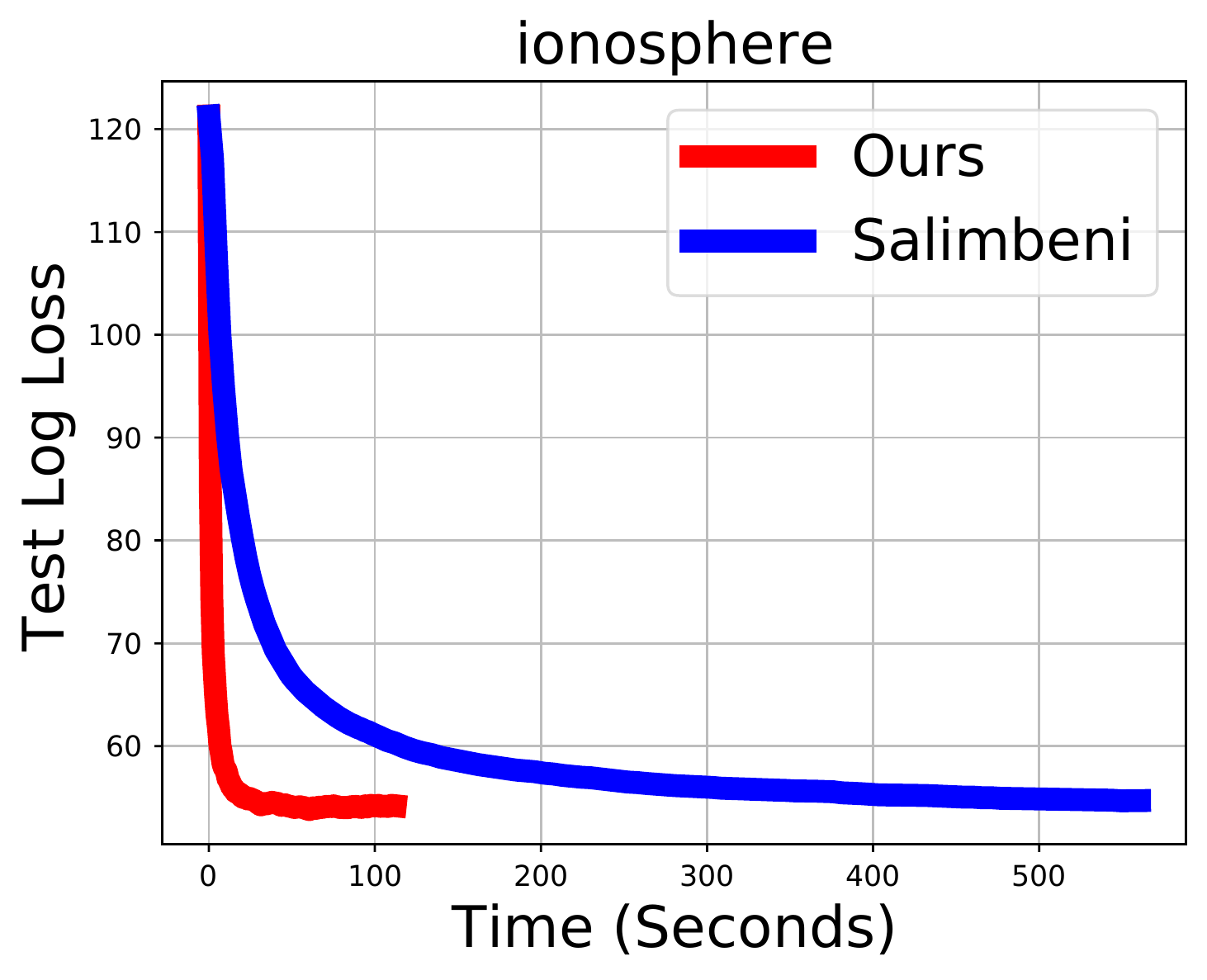}
  }
	\hspace*{-1.5cm}
\vspace{-0.5cm}
   \caption{
   We compare our method to the unconstrained transformation method proposed by \citet{salimbeni2018natural} in the Bayesian logistic model  with a full Gaussian approximation, where we use the re-parametrization trick to avoid computing the Hessian ($\nabla_{\lat}^2 \bar{\ell}(\vlat)$).  Both methods are implemented and tested in the same  environment.
   We tune the step size for each method by grid search.
   Our method can use a bigger step-size than theirs while still optimizing the objective function and maintaining numerical stability.
   Moreover, our method requires less memory than theirs.
   The leftmost plot shows the performance of both methods in terms of the number of iterations.
   The rightmost plot shows the performance of both methods from another perspective in terms of time, where both methods iterate $5,000$ iterations. From this plot, we can clearly see that our method has a lower iteration cost than theirs, where
    the standard gradient  computation time also is considered. If the standard gradient ($\nabla_{\lat} \bar{\ell}(\vlat)$) is computed beforehand, our update is at least  $6 \sim 10$ times faster than theirs.
   }
	\label{figure:unblr_plots}
\end{figure}
\vspace{-0.5cm}

\begin{figure*}[t]
	\centering
	\hspace*{-0.5cm}
  \includegraphics[width=0.7\linewidth]{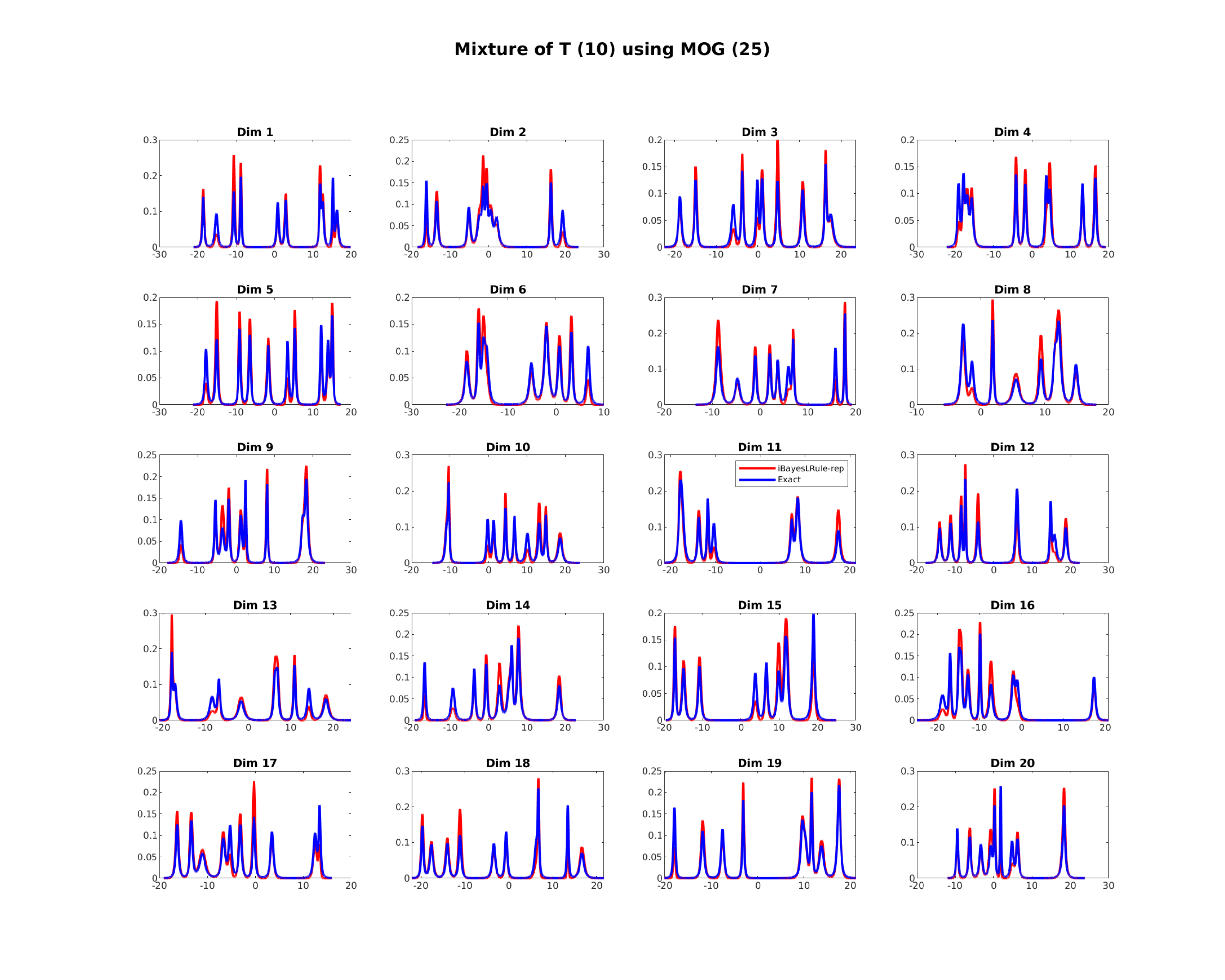}
  \includegraphics[width=0.7\linewidth]{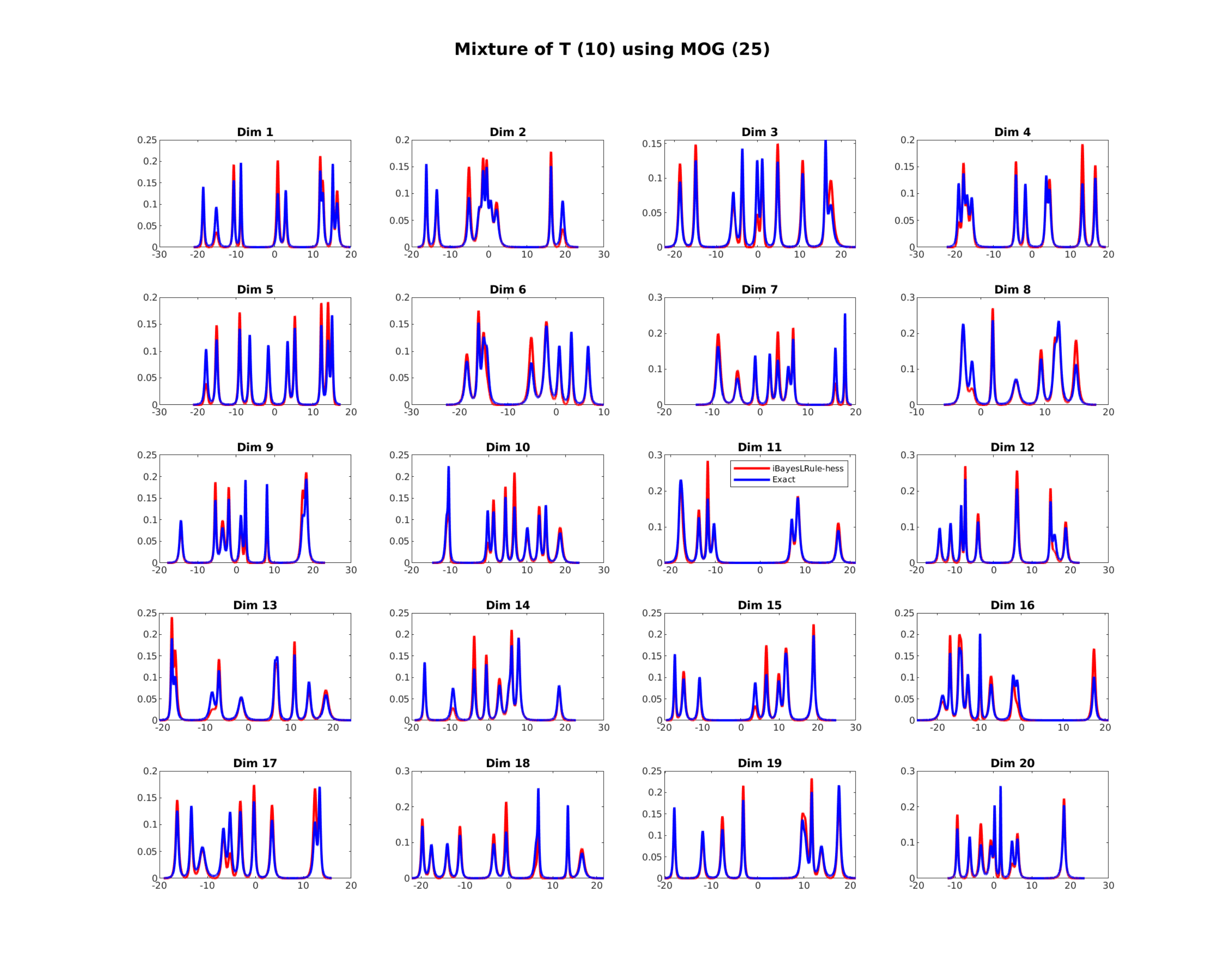}
   \caption{ This is a complete version of the leftmost figure in Figure \ref{figure:toy_examples}. The figure shows MOG approximation (with $K=25$) to fit an MOG model with 10 components in a 20 dimensional problem.}
	\label{figure:mixT20_plots}
\end{figure*}

\begin{figure*}[t]
	\centering
	\hspace*{-2.5cm}
	\begin{minipage}{\textwidth}
  \vspace{-2.5cm}
  \includegraphics[width=1.1\linewidth]{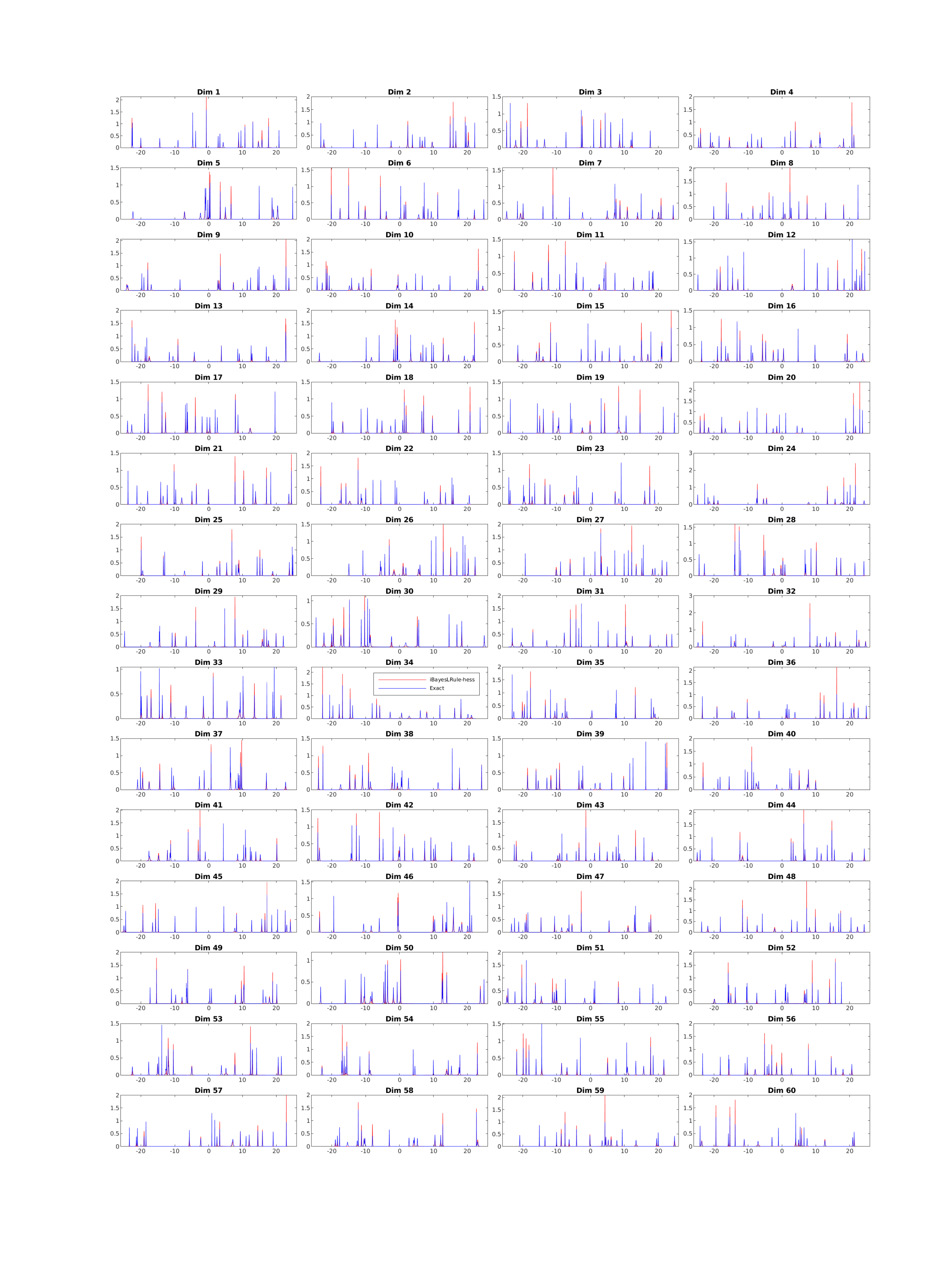}
	\end{minipage}
  \vspace{-2.5cm}
   \caption{
This is the first 60 marginal distributions obtained from a MOG approximation with $K=60$ for a 300-dimensional mixture of Student's T distributions with 20 components. 
 We describe the problem at Section \ref{sec:syn_examples},
where the approximation is obtained by our method at the 50,000-th iteration.
   }
	\label{figure:mixT300_p1_plots}
\end{figure*}

\begin{figure*}[t]
	\centering
	\hspace*{-2.5cm}
	\begin{minipage}{\textwidth}
  \vspace{-2.5cm}
  \includegraphics[width=1.1\linewidth]{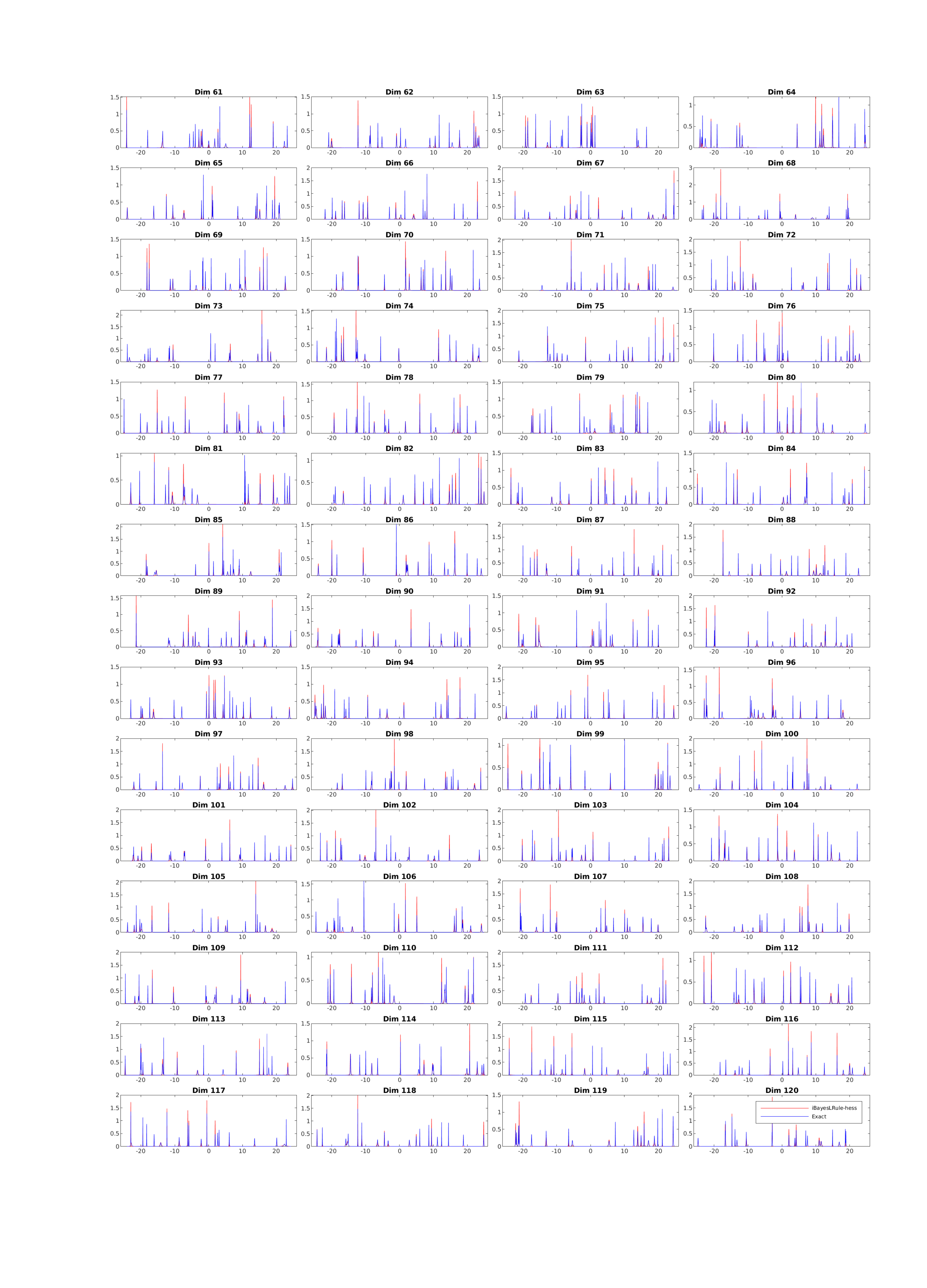}
  \end{minipage}
  \vspace{-2.5cm}
   \caption{
This is the second 60 marginal distributions obtained from a MOG approximation with $K=60$ for a 300-dimensional mixture of Student's T distributions with 20 components. 
 We describe the problem at Section \ref{sec:syn_examples},
where the approximation is obtained by our method at the 50,000-th iteration.
   }
	\label{figure:mixT300_p2_plots}
\end{figure*}

\begin{figure*}[t]
	\centering
	\hspace*{-2.5cm}
	\begin{minipage}{\textwidth}
  \vspace{-2.5cm}
  \includegraphics[width=1.1\linewidth]{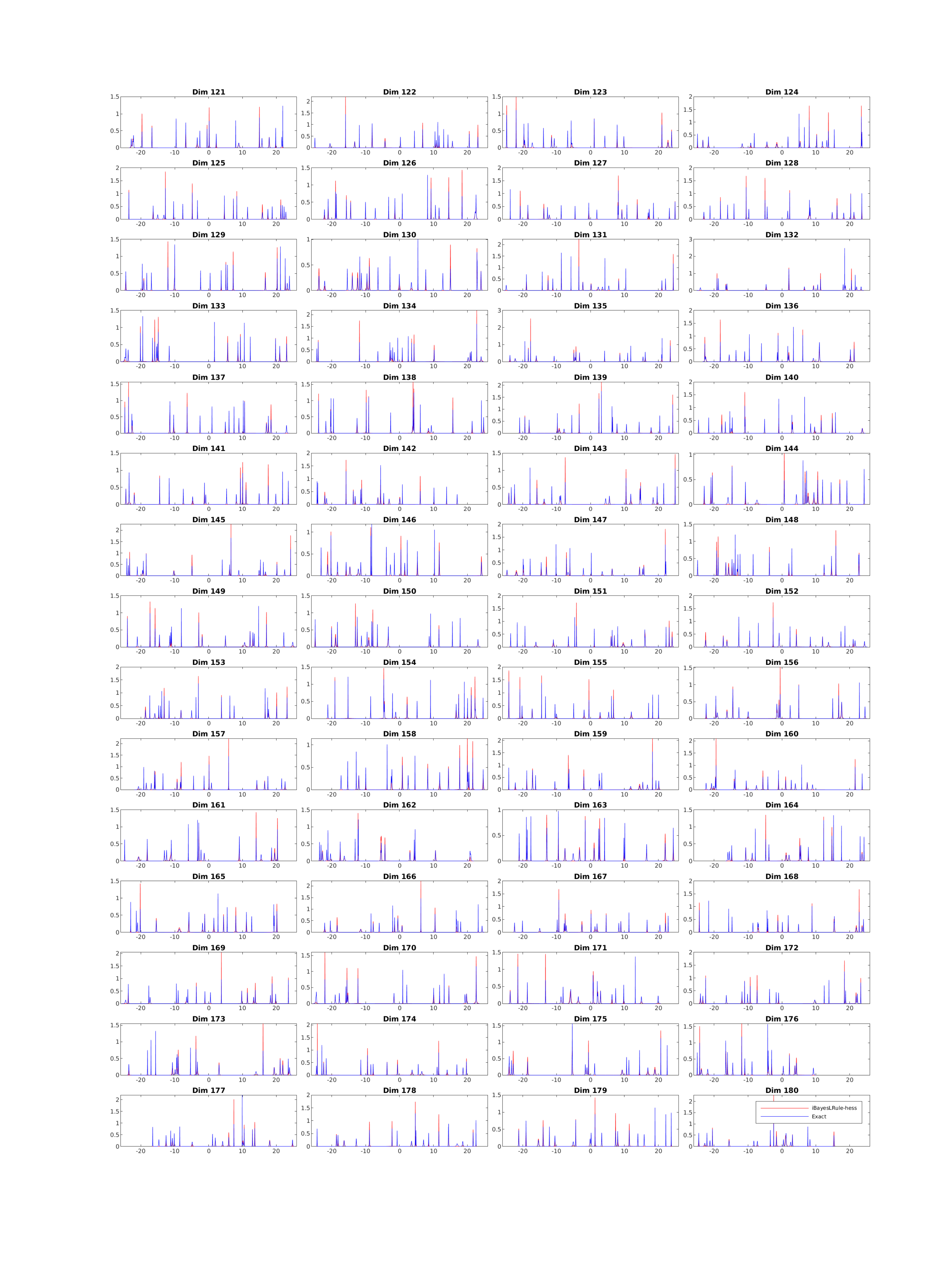}
  \end{minipage}
  \vspace{-2.5cm}
   \caption{
This is the third 60 marginal distributions obtained from a MOG approximation with $K=60$ for a 300-dimensional mixture of Student's T distributions with 20 components. 
 We describe the problem at Section \ref{sec:syn_examples},
where the approximation is obtained by our method at the 50,000-th iteration.
   }
	\label{figure:mixT300_p3_plots}
\end{figure*}

\begin{figure*}[t]
	\centering
	\hspace*{-2.5cm}
	\begin{minipage}{\textwidth}
  \vspace{-2.5cm}
  \includegraphics[width=1.1\linewidth]{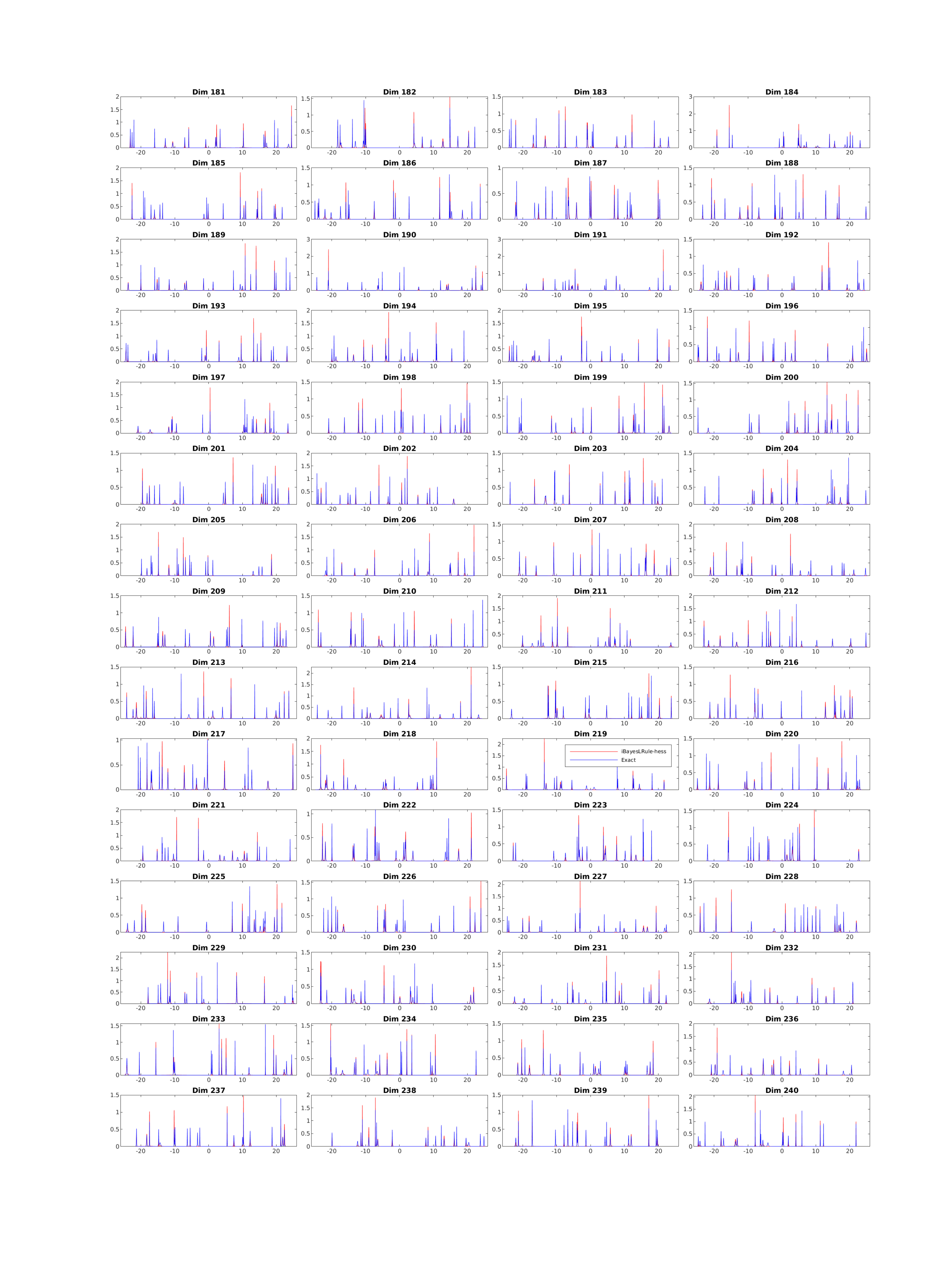}
  \end{minipage}
  \vspace{-2.5cm}
   \caption{
This is the fourth 60 marginal distributions obtained from a MOG approximation with $K=60$ for a 300-dimensional mixture of Student's T distributions with 20 components. 
 We describe the problem at Section \ref{sec:syn_examples},
where the approximation is obtained by our method at the 50,000-th iteration.
   }
	\label{figure:mixT300_p4_plots}
\end{figure*}

\begin{figure*}[t]
	\centering
	\hspace*{-2.5cm}
	\begin{minipage}{\textwidth}
  \vspace{-2.5cm}
  \includegraphics[width=1.1\linewidth]{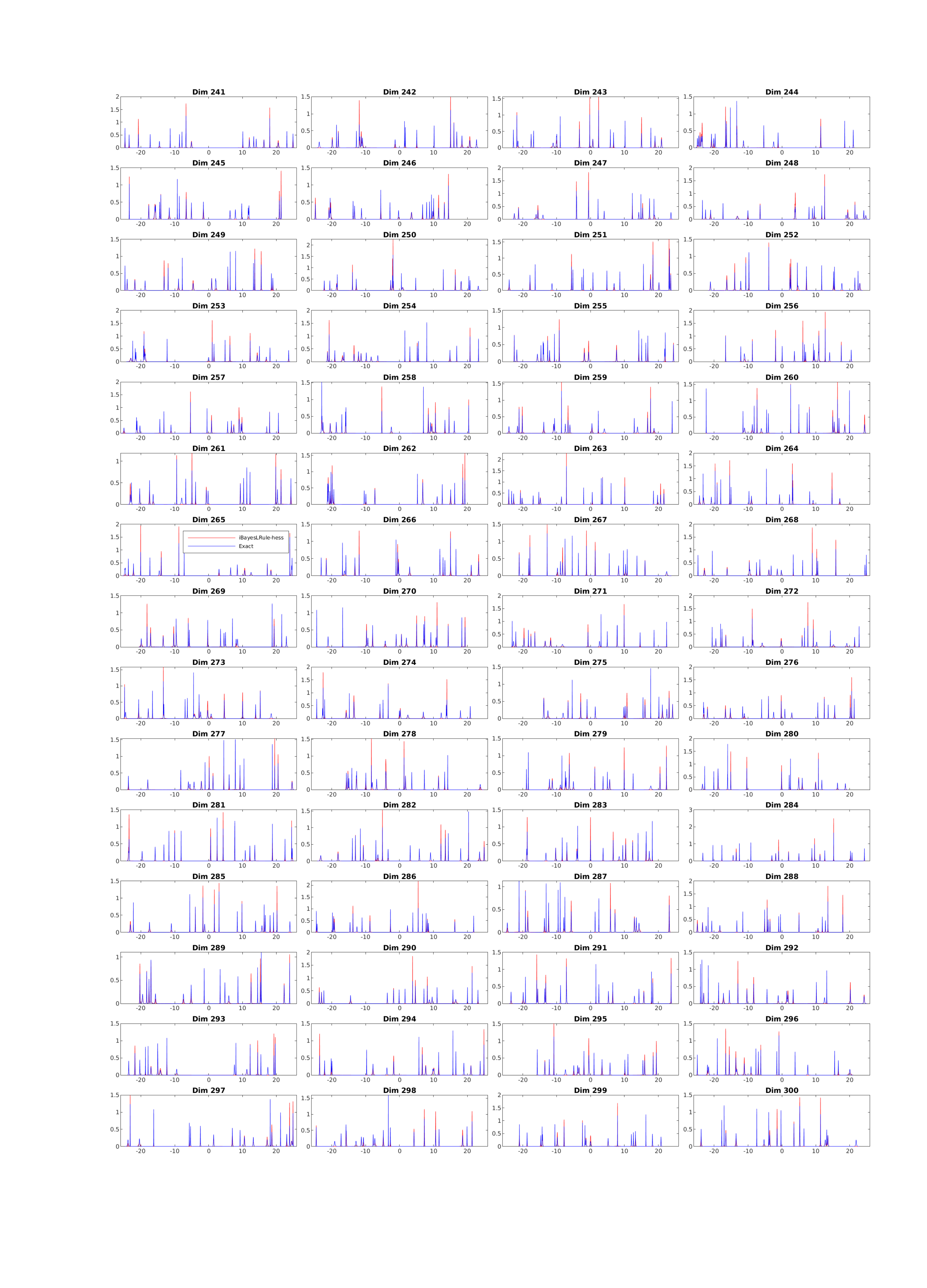}
  \end{minipage}
  \vspace{-2.5cm}
   \caption{
This is the last 60 marginal distributions obtained from a MOG approximation with $K=60$ for a 300-dimensional mixture of Student's T distributions with 20 components. 
 We describe the problem at Section \ref{sec:syn_examples},
where the approximation is obtained by our method at the 50,000-th iteration.
   }
	\label{figure:mixT300_p5_plots}
\end{figure*}

\end{appendices}

\end{document}